\documentclass[11pt,a4paper]{article}

\usepackage{booktabs}
\usepackage{tikz,pgfplots}
\pgfplotsset{compat=1.18}
\usepackage[dvipsnames,table]{xcolor}
\usepackage{hyperref}
\hypersetup{colorlinks=true,citecolor=blue,
filecolor=black,linkcolor=red,urlcolor=ForestGreen,	
linktocpage,linktoc=page}
\counterwithin{equation}{section}
\usepackage[left=3cm,right=3cm, top=3cm, bottom=3cm]{geometry}
\usepackage{natbib}
\usepackage{lmodern}

\usepackage{bm}
\usepackage{amsmath,amssymb,amsthm,amsfonts,amstext,amsbsy,amscd}
\usepackage{subcaption}
\usepackage[labelfont=bf]{caption}
\DeclareMathOperator*{\argmax}{arg\,max}
\DeclareMathOperator*{\argmin}{arg\,min}

\DeclareMathOperator*{\essinf}{ess\,inf}
\DeclareMathOperator{\dif}{d \!}

\DeclareMathOperator{\FI}{FI}
\DeclareMathOperator{\divergence}{div}

\theoremstyle{plain}
\newtheorem{theorem}{Theorem}[section]
\newtheorem{lemma}[theorem]{Lemma}
\newtheorem{corollary}[theorem]{Corollary}
\newtheorem{proposition}[theorem]{Proposition}
\theoremstyle{definition}

\newtheorem{assumption}[theorem]{Assumption}
\newtheorem{remark}[theorem]{Remark}
\newtheorem{definition}[theorem]{Definition}

\usepackage{xpatch}
\makeatletter
\xpatchcmd{\@thm}{.}{}{}{}
\makeatother

\newcommand{\modeldensity}{\ensuremath{f_\theta}}

\newcommand{\samplespace}{\ensuremath{\mathcal{X}}}
\newcommand{\sigmaalgebra}{\ensuremath{\mathcal{F}}}

\newcommand{\E}{\ensuremath{\mathbb{E}}}
\newcommand{\Prob}{\ensuremath{\mathbb{P}}}
\newcommand{\Probtrue}{\ensuremath{\mathbb{P}_0}}
\newcommand{\Probouter}{\mathrm{P}^*}

\newcommand{\Var}{\ensuremath{\text{Var}}}

\newcommand{\R}{\ensuremath{\mathbb{R}}}
\newcommand{\N}{\ensuremath{\mathbb{N}}}
\newcommand{\CIP}{\ensuremath{\mathtt{C}_{\normalfont\text{IP}}}}

\newcommand{\ML}{\ensuremath{\normalfont\text{ML}}}
\newcommand{\SM}{\ensuremath{\normalfont\text{SM}}}
\newcommand{\DDSM}{\ensuremath{\normalfont\text{DDSM}}}

\begin{document}

\title{Diffusion-based Denoising Beats Vanilla Score Matching in Parameter Estimation: A Theoretical Explanation}

\author{
Benedikt Lütke Schwienhorst \thanks{\scriptsize Department of Mathematics, University of Hamburg, \url{benedikt.luetke.schwienhorst@uni-hamburg.de}} \and
Nadja Klein \thanks{\scriptsize Scientific Computing Center, Karlsruhe Institute of Technology, \url{nadja.klein@kit.edu}} \and
Johannes Lederer \thanks{\scriptsize Department of Mathematics, University of Hamburg, \url{johannes.lederer@uni-hamburg.de}}
    }

\maketitle

\begin{abstract}
Score matching is an alternative to maximum likelihood estimation when the normalizing constant is unknown or too costly to evaluate. However, vanilla score matching has shown to be inefficient relative to maximum likelihood estimation for multimodal distributions with well-separated modes, which are commonly encountered in practical applications. We compare a novel diffusion-based denoising score matching estimator (DDSME) to the vanilla score matching estimator (SME) in this scenario. In particular, we prove statistical guarantees for both estimators, showing that the error bound for the vanilla SME worsens when the separation between the modes increases, which can be avoided in case of the DDSME with suitable hyperparameter tuning. This provides a novel theoretical explanation for the superior behavior of diffusion-based score matching over the vanilla version. We support our theoretical findings by numerical experiments.\smallskip
\end{abstract}

\section{Introduction}
\label{sec:introduction}

Score matching (SM) was suggested by~\citet{Hyv2005} as a method to estimate the parameters of a statistical model by minimizing the Fisher divergence (FI) between the model and the true distribution. The FI is a Euclidean distance between score functions, which are defined as the gradient of a log-density and therefore by construction independent of normalizing constants. As a consequence, SM lends itself to the estimation of models with intractable normalizing constants. However, two probability distributions $\Prob$ and $\mathbb{Q}$ defined on a common measurable space $(\Omega,\mathcal{A})$ can have virtually identical score functions on a measurable subset $A\subseteq\Omega$ with $\Prob(A)\approx1$ and $\mathbb{Q}(A)\approx1$, albeit being substantially different from one another. Then, it holds that $\FI(\Prob,\mathbb{Q})\approx0$ and $\FI(\mathbb{Q},\Prob)\approx0$ even though $\Prob$ and $\mathbb{Q}$ are not similiar in a meaningful way. In other words, a small FI between $\Prob$ and $\mathbb{Q}$ does not imply that the two distributions are close to each other with respect to (wrt) another distance measure, for example, the Kullback-Leibler divergence (KL), the total variation distance (TV), the Wasserstein distance of order $p$ ($\text{W}_p$) or some metric on a common parameter space.

An important and practically relevant scenario in which this problem arises are multimodal distributions. In fact, multimodality---or more abstract but similar characterizations thereof in terms of isoperimetry---poses a challenge for a number of (score-based) methods: vanilla SM~\citep{Wen2021,KoeHecRis2022,ZhaKeyHay2022}, Langevin Monte Carlo sampling~\citep{CheBalErd2022}, Markov chain Monte Carlo methods~\citep[MCMC;][]{LatMooStu2025}, generalized Bayesian inference based on Stein discrepancies~\citep{MatKnoBri2022,AfzMutLiq2025}, moment control with kernel Stein discrepancies~\citep{KanBarGre2025} and score-based generative modeling~\citep{SonErm2019}, to name just a few. The relevance of multimodal distributions lies in the fact that they are ubiquitous in statistics and machine learning, that is, many datasets used in practice exhibit multimodality~\citep{BreMij2025}. As a consequence, fitting a model requires families of distributions that are multimodal by construction, either explicitly as in the case of mixtures or implicitly~\citep{CobKopChe1983}.

In the case of vanilla SM, the consequences of multimodality, in particular when the separation between modes is large, can be severe. Roughly speaking,~\citet{KoeHecRis2022} show that the score matching estimator (SME) performs similarly to the maximum likelihood estimator (MLE) if the postulated model family fulfills a Log-Sobolev or Poincaré inequality, that is, if the distributions exhibit some level of concentration~\citep[see Theorem 1 and 2 in][]{KoeHecRis2022}. However, for an exponential family (EF) denoted by $(\Prob_\theta)_{\theta\in\Theta}$ with a $k$-dimensional parameter space $\Theta$, the authors show that there exists a direction $w\in\R^k$ such that
\begin{equation}
	\label{eq:koehler_avar_theorem}
	\frac{\langle w,\text{AVar}[\hat{\theta}_{\SM}]w\rangle}{\langle w,\text{AVar}[\hat{\theta}_{\ML}]w\rangle}\gtrsim\mathtt{C}_{\text{IP}}\left((\Prob_\theta)_{\theta\in\Theta}\right)^{-1},
\end{equation}
where $\mathtt{C}_{\text{IP}}((\Prob_\theta)_{\theta\in\Theta})$ is the isoperimetric constant of the family $(\Prob_\theta)_{\theta\in\Theta}$ and $\text{AVar}[\hat{\theta}_{\SM}]$ and $\text{AVar}[\hat{\theta}_{\ML}]$ denote the asymptotic covariance matrices of the SME and MLE, respectively. Since it is well known that multimodal distributions with well-separated modes have small isoperimetric constants, the result in~\eqref{eq:koehler_avar_theorem} suggests that the SME is asymptotically inefficient relative to the MLE in the case of multimodality. \cite{Dia2023} obtained an analogous result for a bimodal Gaussian mixture on the real line.

A score-based method which has gained a lot of attention recently are score-based generative models~\citep{SonErm2019}, in particular denoising diffusion probabilistic models~\citep[DDPMs;][]{SohWeiMah2015,HoJaiAbb2020,SonSohKin2021}. Interestingly, DDPMs have shown remarkable abilities in the approximation of complex and high-dimensional distributions, including multimodal distributions as in the case of images. By now, there exists a plethora of theoretical research studying the approximation and generalization error of DDPMs~\citep{OkoAkiSuz2023,AzaDelRou2025,YakPuc2025} and more generally, the convergence of diffusion models under different data assumptions and wrt different distance measures, for example, wrt the KL~\citep{Debortoli2023,ConDurGen2025}, TV~\citep{LeeLuTan2022,LeeLuTan2023,OkoAkiSuz2023,CheCheLi2023,ZhaYinLiu2024,JiaZhoLi2025}, $\text{W}_1$~\citep{OkoAkiSuz2023} or $\text{W}_2$~\citep{LeeLuTan2023,GaoNguZhu2025}. Given that DDPMs are based on denoising score matching (DSM), a variant of vanilla SM introduced by~\cite{Vin2011}, their ability to sample from multimodal distributions provokes the question, whether one can build on DDPMs and construct an estimator that overcomes the weaknesses of vanilla SM in this scenario.

An estimator for parametric statistical models based on the DDPM loss was first studied by~\cite{ShaCheKli2023} and more recently by~\cite{CheKalMeh2025}, however, not with the intention to resolve the challenges surrounding multimodality. While~\cite{ShaCheKli2023} essentially provided a proof of concept study, showing that the DDPM objective can be used to efficiently estimate the location parameters of an unknown Gaussian mixture,~\cite{CheKalMeh2025} studied the DDPM loss with a more general research question in mind: how does DDPM score estimation relate to more classical forms of distribution learning, such as parameter estimation and density estimation?

Under essentially the same regularity assumptions that are necessary to show the asymptotic normality of the MLE,~\cite{CheKalMeh2025} arrive at the conclusion that the estimator based on the DDPM loss, which we will call the diffusion-based denoising score matching estimator (DDSME) from now on (see Definition~\ref{def:dsme} in Section~\ref{subsec:estimators}), is asymptotically efficient, that is, they show that
\begin{equation}
	\label{eq:ddpm_asymptotic_normality}
	\sqrt{n}\left(\hat{\theta}_{\DDSM}-\theta_0\right)\xrightarrow{d}\mathcal{N}\left(0,\mathcal{I}(\theta_0)^{-1}\right),	
\end{equation}
where $\mathcal{I}(\theta_0)$ is the Fisher information matrix at the true parameter value $\theta_0$. In other words, $\hat{\theta}_{\DDSM}$ achieves the Cramér-Rao lower bound~\citep[CRLB; see Theorem 3.3 in][]{Sha2003}, which is the smallest asymptotic variance an unbiased estimator can have. Recall that the MLE has the same limiting distribution as given in~\eqref{eq:ddpm_asymptotic_normality}, that is, it achieves the CRLB as well~\citep[see Theorem 5.39 in][]{Vaa1998}. Thus,~\eqref{eq:ddpm_asymptotic_normality} can be understood as a characterization of the asymptotic optimality of the DDSME.

Although the result in~\eqref{eq:ddpm_asymptotic_normality} is undoubtedly powerful, it still leaves a lot to be desired. For one, it is asymptotic in nature and thus not meaningful for contemporary applications, characterized by high-dimensionality and sample sizes that never surpass the data dimension. In many of these cases, classical ``large $n$, ﬁxed $d$'' theory simply fails to provide useful predictions~\citep[see Chapter 1 of][]{Led2021}. Second and more importantly, it is formulated for general parametric statistical models. As a consequence, it does not allow any insight into the mechanism that allows the DDSME to overcome the inefficiency of the vanilla SME in a multimodal scenario. The work of~\citet{ShaCheKli2023} is also not insightful in this regard, as the authors consider Gaussian mixtures with known mixture weights and only treat the location parameters as unknown. But it is now well known that the mixture weights are precisely the parameters that cause problems for SM~\citep{Wen2021}. Furthermore, our main interest does not lie in the sample complexity but instead in the way that multimodality affects the performance of the chosen estimator. In this regard, our perspective is different from~\cite{YakMarPuc2025}, where the focus lies on the comparison between the sample complexity of vanilla SM and DSM (for a fixed noise level), leaving dependencies on the possibly multimodal structure of the distributions unexplored. These considerations lead us to the following two important yet open research questions:\\

\noindent
\textbf{Question 1:} Does the DDSME recover the unknown parameter of a multimodal distribution with high probability, and how does it compare to the vanilla SME?\\

\noindent
\textbf{Question 2:} If so, how is the DDSME able to outperform the vanilla SME when the true distribution is multimodal?\\

\begin{figure}[!h]
	\centering
    \begin{subfigure}{0.71\textwidth}
	\includegraphics[width=\textwidth]{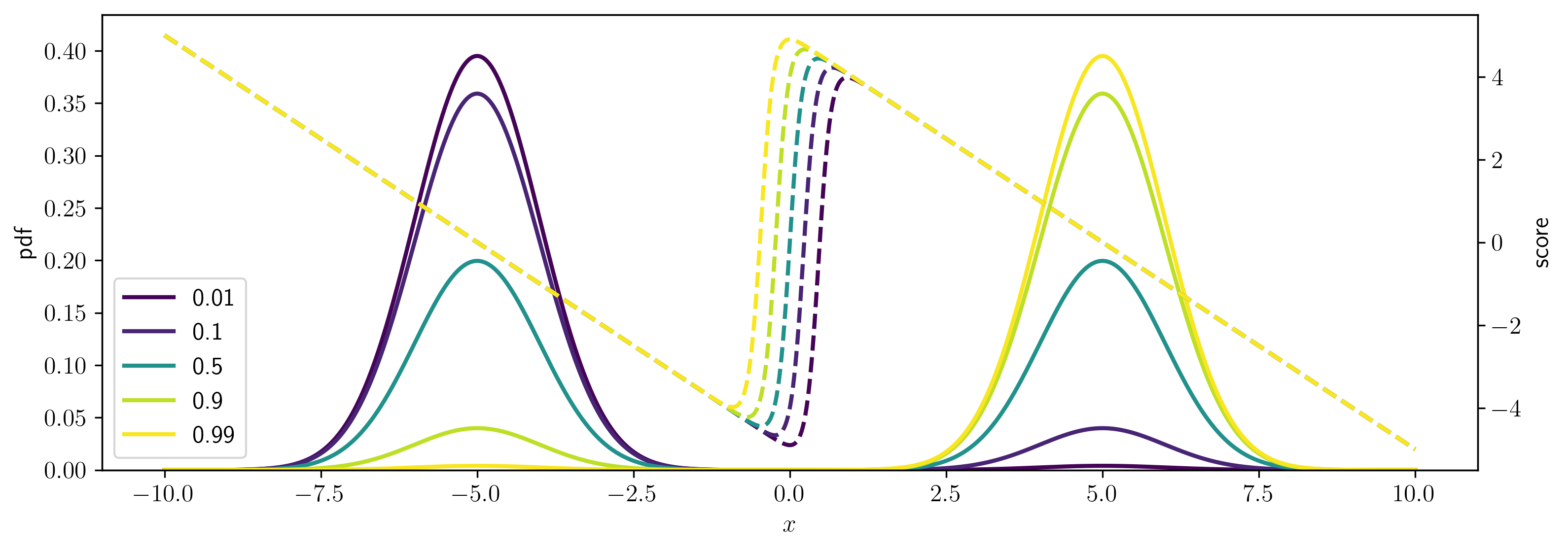}
	\caption{Densities and scores}
	\vspace{2mm}
	\label{fig:pdfs_scores}
	\end{subfigure}
    \begin{subfigure}{0.27\textwidth}
	\includegraphics[width=\textwidth]{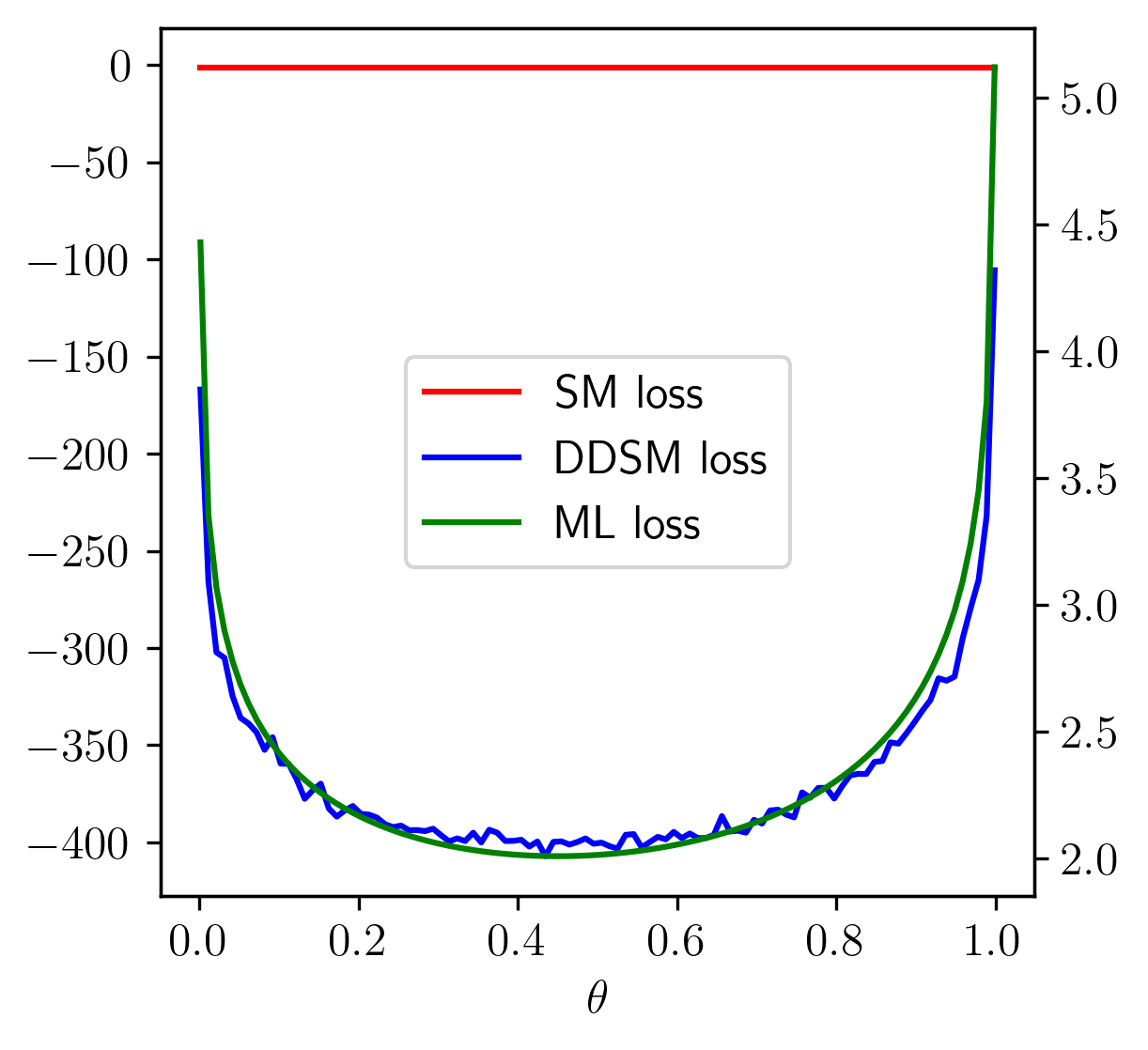}
	\caption{Loss functions}
	\vspace{2mm}
	\label{fig:losses}
	\end{subfigure}
	\caption{\textbf{(a)} Densities (solid lines) and score functions (dashed lines) of~\ref{eq:gm} (see Section~\ref{subsec:statistical_model}) with location parameter $\mu=5$. The densities and scores are given for $\theta\in\{0.01,0.1,0.5,0.9,0.99\}$. Computing the FI between any of the depicted distributions will result in a small value, given that the score functions only differ on a set of very small probability mass. \textbf{(b)} SM, DDSM and ML losses as functions of $\theta$, based on a sample drawn from $\Prob_{\theta_0}$ with $\theta_0=0.5$. The DDSM loss is more similar to the convex and therefore well-behaved ML loss than the flat SM loss, which explains why the DDSME can outperform the vanilla SME.}
	\label{fig:scores}
\end{figure}

In this work, we give answers to both questions. In order to do so, we restrict our analysis to a mixture of two isotropic Gaussian distributions on $\R^d$ as defined in Section~\ref{subsec:statistical_model}. The reason for this is twofold: for one, learning a distribution via diffusion models requires learning its score under Gaussian convolutions, which is intractable for most distributions. Gaussian mixtures are a rare exception, making them prime candidates for further analysis~\citep[e.g.][]{ShaCheKli2023,CheKonShah2024,GatKelLee2025}. Second, even though Gaussian mixtures are explicit in terms of the distributional structure, they are nevertheless very general approximators of distributions. In fact, Gaussian mixtures on $\R^d$ are dense in the set of probability distributions on $\R^d$ wrt the weak topology~\citep{Lo1972,Wie1932}. As a consequence, any density can be approximated arbitrarily well by a sufficiently rich  mixture of Gaussian densities. Hence, we argue that our setting is by no means restrictive, since often only a few mixture components are required.

Figure~\ref{fig:scores} illustrates the peculiar behavior of the FI described at the beginning and how it culminates into a flat SM loss function, contrasting the more or less convex DDSM loss. More precisely, the score functions (dashed lines in Figure \ref{fig:pdfs_scores}) differ only on a set that has almost zero mass under any of the distributions (with densities shown in solid lines). Then, Figure~\ref{fig:losses} shows the SM, DDSM and ML losses as functions of $\theta$ based on a sample drawn from $\Prob_{\theta_0}$ with $\theta_0=0.5$. One can see that the DDSM loss is more similar to the convex ML loss than the flat SM loss, which explains why the DDSME can outperform the vanilla SME in the model~\ref{eq:gm} (see Section~\ref{subsec:statistical_model}).

\paragraph{Our contributions.}

In this paper, motivated by the aforementioned research questions, we investigate the statistical properties of the DDSME in comparison to the vanilla SME, when estimating the parameter of a multimodal distribution. Our main contributions are as follows.
\begin{enumerate}
	\item Building on the isoperimetric perspective initiated by~\citet{KoeHecRis2022} we show that the family~\ref{eq:gm} (see Section~\ref{subsec:statistical_model}) has a small isoperimetric constant, which in turn characterizes both the flatness of the SM loss and the asymptotic inefficiency of the vanilla SME in this model.
	\item Having established the weakness of SM in terms of the loss landscape and the asymptotic variance of the resulting estimator in 1., we derive finite-samples guarantees that show that the DDSME does not suffer from this unfavorable property. The dependence on the (known) distance between the mean parameters is made explicit in the bounds. More precisely, the error bound of the SME weakens as the separation between the modes increases, while choosing a large time horizon in case of the DDSME eliminates this dependence on the separation.
	\item We confirm our theoretical findings with numerical experiments. In particular, we verify that these findings extend beyond the Gaussian setting.
\end{enumerate}

\paragraph{Further related work.}

After the seminal paper of~\cite{Hyv2005}, subsequent research on SM had focused on extension to other types of distributions, which do not fulfill the assumptions in~\cite{Hyv2005}, most notably the differentiability of the involved densities and score functions as well as boundary conditions relating the tail behavior of the true and model distribution. This includes but is not limited to discrete distributions~\citep{Hyv2007,Lyu2012,MenChoSon2022}, nonnegative distributions~\citep{Hyv2007,YuDrtSho2019}, ordinal distributions~\citep{XuSceWoo2025}, distributions with truncated support~\citep{LiuKanWil2022}, distributions with general domains~\citep{YuDrtSho2021}, distributions on manifolds~\citep{MarKenLah2016} and infinite dimensional models~\citep{SriFukGre2017}, to name just a few. Other works addressed the complexity surrounding the computation of the second derivative by comparing the average difference between projections of the score functions onto random directions, a solution inspired by the Hutchinson trace estimator~\citep{SonGarShi2020}.

While early research has established basic statistical properties of the SME, for example, consistency and asymptotic normality~\citep{Hyv2005,ForLau2015,Bar2022}, more recently, the focus shifted towards a more refined analysis of the statistical benefits and drawbacks of SM. While~\cite{Wen2021} and~\cite{ZhaKeyHay2022} discussed the \textit{blindness} of SM wrt mixture components in terms of the FI,~\cite{KoeHecRis2022} provided a first analysis that rigorously showed the inefficiency of the SME relative to the MLE in an EF that can be multimodal. \cite{PabRohSev2023} investigate the tradeoffs between computational and statistical aspects, by showing that there exists a natural EF for which the SME is more or less as efficient as the MLE, but the MLE is not efficiently computable.

The original motivation of~\cite{Hyv2005} for SM were unnormalized models, such as energy-based models, or statistical models with intractable normalizing constants. Naturally, SM has been used to build estimators in scenarios were normalization is difficult, costly or simply hopeless, for example, high-dimensional graphical models~\citep{YuDrtSho2019} or high-dimensional extreme-value distributions modeled with the Hüsler-Reiss model~\citep{LedOes2023}. However, SM has been applied in quite unexpected ways as well, most notably in the construction of an estimator for the coefficients of a linear regression model, which has minimal asymptotic variance~\citep{FenKaoSam2026}. Another unorthodox application appears in the recent work by~\cite{YouShaSam2026}, where the authors leverage score estimation to make local polynomial regression adaptive to the error distribution.

Generative models are designed to sample from an unknown distribution $\Prob$ when this distribution has only been observed via samples $X_1,\ldots,X_n$ with $X_i\sim\Prob$. Classical sampling algorithms such as MCMC~\citep{RobCas2004} pursue the same goal of sampling from $\Prob$, but instead of access to data they are generally based on the assumption of query access to $\nabla\log f$ with $f$ being the density of $\Prob$. With this important conceptual difference in mind, it is nonetheless natural to relate DDPMs to sampling algorithms that are based on diffusion processes such as Langevin Monte Carlo (LMC) sampling. Interestingly, while DDPMs are naturally powerful in the approximation of complex distributions, the vanilla Langevin diffusion used in LMC is known to perform poorly in this scenario. In fact, most convergence guarantees for LMC are restricted to the log-concave case or they require some type of functional inequality, such as a Log-Sobolev inequality to hold~\citep{CheBalErd2022}. Extensions of the vanilla Langevin diffusion are plentiful and, among other things, have been based on annealing schemes~\citep{QinRis2024} or the use of more general driving noise processes, such as Lévy-type processes, with an accordingly adjusted drift~\citep{BehLue2025}. A particularly interesting extension was studied by~\cite{KoeVuo2024}: the authors used early-stopping in combination with a data-based initialization, that is, they assumed not only query access to $\nabla\log f$ but also access to a number of samples from $\Prob$. In a follow-up work,~\cite{KoehLeeVuo2024} generalized this approach to Markov chains that fulfill a spectral gap condition.

\paragraph{Organization of the paper.}

The remainder of the paper is organized as follows. In Section~\ref{sec:preliminaries}, we introduce the statistical model and the estimators that we will study in this work. In Section~\ref{sec:main_results}, we present our main results: First, we give a characterization of the poor performance of the SME when estimating a bimodal Gaussian mixture. Second, we provide finite-sample guarantees for both the SME and the DDSME, which show that the DDSME achieves better finite-sample performance than the SME. The results are then discussed in Section~\ref{sec:discussion}.

\section{Preliminaries}\label{sec:preliminaries}

In this section, we introduce the notation, the statistical model, and the estimators that we will study and compare in Section~\ref{sec:main_results}.

\subsection{Notation}\label{subsec:notation}

Throughout, we use the shorthand $x\wedge y:=\min\{x,y\}$, $\max\{x,y\}:=x\vee y$ and $[k]:=\{1,\ldots,k\}$. The density of the standard Gaussian distribution on $\R$ is denoted by $\varphi(x):=(2\pi)^{-1/2}\exp(-x^2/2)$ and its cumulative distribution function by $\Phi(x):=\int_{-\infty}^x \varphi(t)\dif t$. On $\R^d$ we write $\phi(x):=(2\pi)^{-d/2}\exp(-\|x\|_2^2/2)$ for the density of the standard Gaussian distribution. For a function $g:\R^d\to\R$, we denote by $\nabla g$ its gradient and by $\partial_{x_j}$ its partial derivative wrt $x_j$ for $j=1,\ldots,d$, and $x=(x_1,\ldots,x_d)^\top\in\R^d$. For a vector field $v:\R^d\to\R^d$, we denote by $\divergence v:=\sum_{j=1}^d \partial_{x_j} v_j$ its divergence. The Euclidean norm on $\R^d$ is denoted by $\|\cdot\|_2$. By $f*g$ we denote the convolution of $f$ and $g$. For two matrices $A,B\in\R^{d\times d}$, we write $A\preceq B$ if $B-A$ is positive semi-definite. We will frequently suppress constants that are not relevant. More precisely, we will write $f\lesssim g$ or $f\gtrsim g$ when there exists an absolute constant $\mathtt{C}>0$ such that $f\leqslant \mathtt{C}g$ or $\mathtt{C}f\geqslant g$ respectively. The dependence of omitted constants will be indicated by a subscript, for example, $f\lesssim_{\eta} g$ means that the constant $\mathtt{C}$ depends on $\eta$.

\subsection{Statistical model}\label{subsec:statistical_model}

In order to study the problem that we put forth in Section~\ref{sec:introduction}, we postulate a parametric statistical model with a particular family of parameterized probability distributions. More precisely, let $(\samplespace,\sigmaalgebra,(\Prob_\theta)_{\theta\in\Theta})$ denote a parametric statistical model dominated by some measure $\gamma$, such that each $\Prob_\theta$ has a density $\modeldensity$ wrt $\gamma$. In general, we assume $\gamma$ to be the Lebesgue measure on $\R^d$. We call $\samplespace\subseteq\R^d$ the sample space and $\Theta\subseteq\R^k$ the parameter space. For the $\mathcal{X}$-valued random variable $X$ we write $X\sim\Prob$ to indicate that $X$ has distribution $\Prob$ and by $\E[X]:=\int_\samplespace x\dif \Prob(x)$ we denote its expected value. The data $\{X_i\}_{i=1}^n$ are assumed to be generated from a true probability distribution $\Prob_{0}$ with density $f_{0}$ wrt $\gamma$. Later in Section~\ref{subsec:assumptions}, we will assume that the model is well-specified in the sense that there exists a parameter value $\theta_0\in\Theta$ such that $\Probtrue=\Prob_{\theta_0}$ and we will make the assumption that $\Theta$ is a compact subset. In order to avoid ambiguity, we will sometimes write $\E_{\Prob}[X]$, $\E_0[X]$ or $\E_{\theta}[X]$ to emphasize the dependence on the distribution $\Prob$, $\Probtrue$ or $\Prob_\theta$ respectively.

In addition, we assume that the family $(\Prob_\theta)_{\theta\in\Theta}$ is a mixture of two isotropic Gaussians in $\R^d$ (denoted by~\ref{eq:gm}):
\begin{equation*}
	\label{eq:gm}
		(\Prob_\theta)_{\theta\in\Theta}:=\left\{\Prob_{\theta}\,|\,\Prob_{\theta}=\theta\mathcal{N}(\mu_1,\mathbb{I}_d)+(1-\theta)\mathcal{N}(\mu_2,\mathbb{I}_d)\text{ with }\theta\in\Theta\subseteq(0,1)\right\}\tag{$\mathtt{GM}$}
\end{equation*}
In the above, we treat the mean parameters $\mu_1,\mu_2\in\R^d$ with $\mu_1\not=\mu_2$ as known nuisance parameters and we use $\Delta:=\|\mu_1-\mu_2\|_2/2$ to denote half the distance between the means. The parameter of interest is the mixture weight $\theta\in\Theta$, which we want to estimate based on the data. Fortunately, we can restrict our analysis to a family with mean parameters that are point symmetric wrt the origin and non-zero only in the first coordinate, as shown in the following lemma. This will be useful in the subsequent analysis, as it allows us to trace back the $d$-dimensional problem to a single dimension. The proof of the following lemma is given in Appendix~\ref{app:lemma_embedded_one_dim_sym_mixture}.

\begin{lemma}[\ref{eq:gm} as embedded one-dimensional symmetric Gaussian mixtures]
\label{lemma:embedded_one_dim_sym_mixture}
Let $(\Prob_\theta)_{\theta\in\Theta}$ be the family in~\eqref{eq:gm} with separation $2\Delta=\|\mu_1-\mu_2\|_2$. Then, the following holds true:
\begin{enumerate}
	\item There exists an isometry $T:\R^d\to\R^d$ such that for all $\theta\in\Theta$ the pushforward $\mathbb{Q}_\theta:=\Prob_\theta\circ T^{-1}$ has density
	\begin{equation*}
		q_\theta(x)=\theta \phi(x-\Delta e_1)+(1-\theta)\phi(x+\Delta e_1),
\end{equation*}
where $\Delta e_1=T(\mu_1)$, $e_1=(1,0,\ldots,0)^\top\in\R^d$ with $\|T(\mu_1)-T(\mu_2)\|_2=2\Delta$.
\item We have $\mathbb{Q}_\theta=\mathbb{Q}_{\theta,1}\otimes\mathcal{N}(0,\mathbb{I}_{d-1})$, where $\mathbb{Q}_{\theta,1}$ is the marginal distribution of $\mathbb{Q}_\theta$ in the first coordinate with density
\begin{equation*}
	q_{\theta,1}(x_1)=\theta \varphi(x_1-\Delta)+(1-\theta)\varphi(x_1+\Delta).
\end{equation*}
In particular, this implies $q_\theta(x)=q_{\theta,1}(x_1)\phi(x_{2:d})$ with $x_{2:d}:=(x_2,\ldots,x_d)^\top\in\R^{d-1}$ for all $x=(x_1,x_{2:d})^\top\in\R^d$.
\end{enumerate}
\end{lemma}

\begin{remark}[Implications of Lemma~\ref{lemma:embedded_one_dim_sym_mixture}]
	Since we assume knowledge of the mean parameters $\mu_1$ and $\mu_2$ throughout this work, Lemma~\ref{lemma:embedded_one_dim_sym_mixture} tells us, that we can estimate the mixture weight $\theta$ from transformed data $T(X)\sim \Prob_\theta\circ T^{-1}$ instead of the original data $X\sim\Prob_\theta$. In fact, we only require the first coordinate of the transformed data. Any possible dependence on the real-valued mean parameters $\Delta$ and $-\Delta$ could be traced back to the original mean parameters $\mu_1$ and $\mu_2$ in $\R^d$ via the isometry $T$, which is bijective and therefore invertible. However, we will see in Lemma~\ref{lemma:reduction_constants} in Section~\ref{sec:main_results} that by studying the marginal distribution $\mathbb{Q}_{\theta,1}$ we will obtain results that are valid for $\Prob_\theta$ as well and that the performance of the SME and the DDSME only depends on $\Delta$, confirming the intuition that the estimation of the mixture weight should only depend on the separation between the mixture components, rather than their specific locations.
\end{remark}

\begin{remark}[Multimodality of~\ref{eq:gm}]
	The parameterization in~\eqref{eq:gm} does not ensure that $\Prob_\theta$ is bimodal for all $\theta\in\Theta$. In fact, for $d=1$ \citet{Eis1964} showed that $(\mu_2-\mu_1)^2<\sqrt{27/8}$ is a sufficient condition for all distributions in the family~\ref{eq:gm} to be unimodal, while $(\mu_1-\mu_2)^2>4$ implies the existence of values of $\theta\in[0,1]$ for which $\Prob_\theta$ is bimodal. This analysis was further extended by~\citet{Beh1970}, who showed that a mixture of two Gaussians on the real line can have at most two modes. For general $d$, the multimodality of the mixture distribution is more complex~\citep{AmeEngHaa2019}. Even though, in this work, we are primarily interested in showing the superiority of the DDSME in the setting of well-separated modes, our results do not require general assumptions on the minimum separation between the mixture components as in~\citet{ShaCheKli2023}, such that they hold for any $\Delta>0$.
\end{remark}

\subsection{Estimators}\label{subsec:estimators}

We briefly introduce the estimators that we will study in this work: the SME and the DDSME. Both are M-estimators as we will explain in the sequel. This fact is crucial for our later analysis. But first, let us recall that a classical estimator for a parametric statistical model as described in Section~\ref{subsec:statistical_model} is the MLE~\citep[see Chapter 5.5 in][]{Vaa1998}:
\begin{equation*}
   \hat{\theta}_{\text{ML}} \in \argmin_{\theta \in \Theta} \left\{-\frac{1}{n}\sum_{i=1}^n \log f_\theta(X_i) \right\}.
\end{equation*}
Its asymptotic optimality under fairly general conditions~\citep[see Theorem 5.39 in][]{Vaa1998} is one of the main reasons for its success. But since any density can be written as $f_\theta=\mathtt{C}(\theta)^{-1}\kappa_\theta$ with normalizing constant $\mathtt{C}(\theta)$ and kernel $\kappa_\theta$, computing the MLE becomes challenging whenever $\mathtt{C}(\theta)$ is intractable and cannot easily be evaluated pointwise. This motivates score matching.
  
\paragraph{Score matching.}

To avoid dealing with $\mathtt{C}(\theta)$,~\citet{Hyv2005} considered the Fisher divergence~\citep[or relative Fisher information;][]{Wib2025}
\begin{align}
    \label{eq:fisher_distance}
    \FI(\Prob,\mathbb{Q}) &= \E_{\Prob}\left[\left\|\nabla \log q(X) - \nabla \log p(X)\right\|_2^2\right]\\
	\label{eq:fisher_distance_2}
	&= \E_{\Prob}\left[\left\|\nabla \log q(X)\right\|_2^2 - 2\left\langle\nabla \log q(X),\nabla \log p(X)\right\rangle\right]+\mathtt{C}(\Prob),
\end{align}
between two measures $\Prob$  and $\mathbb{Q}$ with densities $p$ and $q$,  respectively. In the above, $\mathtt{C}(\Prob)=\E_{\Prob}[\|\log q(X)\|_2^2]$ is a constant independent of $\theta$. Since~\eqref{eq:fisher_distance} and~\eqref{eq:fisher_distance_2} are themselves intractable due to the explicit dependence on the true but unknown density $p$, integration by parts under regularity assumptions~\citep[see Theorem 1 in][]{Hyv2005} yields
\begin{equation}
    \label{eq:fisher_distance_3}
    \FI(\Prob,\mathbb{Q})=\E_{\Prob}\left[\left\|\nabla \log q(X)\right\|^2_2+2\divergence\nabla\log q(X)\right]+\mathtt{C}(\Prob).
\end{equation}
The latter distance measure is suited to empirical risk minimization and since the score of the measure $\Prob_\theta$ satisfies $\nabla\log f_\theta=\nabla\log \kappa_\theta$, and thus
the (population) SM loss $\theta\mapsto\FI(\Prob_0,\Prob_\theta)$ can be minimized without knowing $\mathtt{C}(\theta)$. Specifically, replacing the true unknown $\Prob_0$ with the empirical measure $\Prob_n$ produces an empirical loss, which justifies the following definition. Note that the objective in~\eqref{eq:fisher_distance_3} first appeared in~\cite{Cox1985}, however in order to build a nonparametric penalized estimator for the score function itself.

\begin{definition}[SME]
	\label{def:sme}
Let $f_0,f_\theta$ and $\nabla f_\theta$ be continuously differentiable wrt $x\in\samplespace$ and assume that $\E_0[\|\nabla\log f_\theta(X)\|_2^2]$ and $\E_0[\divergence\nabla\log f_\theta(X)]$ are finite for every $\theta\in\Theta$. Further, assume that $f_0(x)\nabla\log f_\theta(x)\rightarrow0$ for $x\rightarrow\partial\samplespace$. Then, given an i.i.d.~sample $\{X_i\}_{i=1}^n$ with $X_i\sim\Probtrue$, the SME is defined as
    \begin{equation}
	\label{eq:definition_sme}
	\hat{\theta}_{\SM}\in\argmin_{\theta\in\Theta}\left\{\frac{1}{n}\sum_{i=1}^n \left(\left\|\nabla\log f_\theta(X_i)\right\|_2^2+2\divergence\nabla\log f_\theta(X_i)\right)\right\}.
\end{equation}
\end{definition}

\paragraph{Diffusion-based denoising score matching.}

DSM as introduced by~\citet{Vin2011} is motivated by the desire to avoid computing the second derivatives needed to evaluate~\eqref{eq:definition_sme}. To that end, the objective is turned into a denoising problem by replacing the ground truth $\Prob_0$ in $\theta\mapsto\FI(\Prob_0,\Prob_\theta)$ with the measure $\Prob_{\sigma}$, which is the distribution of the noisy observation $\widetilde{X}=X+\varepsilon$ obtained by adding independent mean-zero noise $\varepsilon\sim\mathbb{Q}_\sigma$ with $\Var[\varepsilon]$ being a function of $\sigma$. The hyperparameter $\sigma>0$ controls the level of noise and we expect $\Prob_0\approx\Prob_\sigma$ for small values of $\sigma$. By construction, the density $f_\sigma$ of $\Prob_{\sigma}$ is given by $f_\sigma = f_0 * q_\sigma$, where $q_\sigma$ denotes the density of the noise distribution $\mathbb{Q}_\sigma$. In theory, one is then interested in the loss $\theta\mapsto \FI(\Prob_\sigma,\Prob_\theta)$, but $\Prob_\sigma$ and $f_\sigma$ are typically intractable themselves since convolutions of densities cannot be computed in closed-form for arbitrary distributions. However, \citet{Vin2011} showed the identity
\begin{equation}
    \label{eq:dsm_loss_1}
    \FI(\Prob_\sigma,\Prob_\theta)=\E_{\Probtrue}\left[\E_{\Prob_{\widetilde{X}\mid X}}\left[\left\|\nabla \log f_\theta(\widetilde{X})\right\|_2^2-2\left\langle\nabla \log f_\theta(\widetilde{X}), \nabla\log f_{\widetilde{X}\mid X}(\widetilde{X}
	)\right\rangle\right]\right]+\mathtt{C}(\Prob_\sigma),
\end{equation}
where $\mathtt{C}(\Prob_\sigma)=\E_{\Prob_\sigma}[\|\nabla\log f_\sigma(X)\|_2^2]$ is independent of $\theta$. In~\eqref{eq:dsm_loss_1}, as indicated by the subscripts, the inner expectation is taken wrt the conditional distribution of $\widetilde{X}$ given $X$, denoted by $\Prob_{\widetilde{X}\mid X}$. The corresponding density is $f_{\widetilde{X}\mid X}$. If one resorts to the classical choice $\mathbb{Q}_\sigma=\mathcal{N}(0,\sigma^2\mathbb{I}_d)$, as was done by~\citet{Vin2011}, then $\widetilde{X}\mid X\sim\mathcal{N}(X,\sigma^2\mathbb{I}_d)$ and the score $\nabla\log f_{\widetilde{X}\mid X}(\widetilde{X})$ can be computed in closed form:
\begin{equation*}
	\nabla\log f_{\widetilde{X}\mid X}(\widetilde{X})=\frac{X-\widetilde{X}}{\sigma^2}.
\end{equation*}
From~\eqref{eq:dsm_loss_1} it follows that the optimization of $\theta\mapsto \FI(\Prob_\sigma,\Prob_\theta)$ is equivalent to the optimization of $\theta\mapsto \E_{\Probtrue}[\FI(\Prob_{\widetilde{X}\mid X},\Prob_\theta)]$ and replacing $\Probtrue$ with the empirical measure $\Prob_n$ yields the corresponding empirical risk.

DDPMs are based on DSM. More precisely, in order to generate samples from some distribution $\Probtrue$, DDPMs require an estimate of the score function $(x,t)\mapsto\nabla\log f_t(x)$ for $(x,t)\in\mathcal{X}\times[0,T]$, where $f_t$ denotes the density of the marginal distribution $\Prob_t$ of an Ornstein-Uhlenbeck (OU) process $(X_t)_{t\in[0,T]}$ with initial condition $X_0\sim\Prob_0$ and some fixed time horizon $T>0$. This estimate is obtained by minimizing the DDPM loss, which is given by
\begin{equation}
	\label{eq:ddpm_loss}
	s_\vartheta\mapsto\E_{\Probtrue}\left[\int_0^T\E_{X_t\mid X_0}\left[\left\|s_\vartheta(X_t,t)\right\|^2_2+\frac{2}{\sqrt{1-\exp(-2t)}}\left\langle s_\vartheta(X_t,t),Z_t\right\rangle\right]\dif t\right],
\end{equation}
where $s_\vartheta$ is generally a neural network with parameters $\vartheta\in\R^k$~\citep[cf.][]{OkoAkiSuz2023,Debortoli2023}. Then, minimizing \eqref{eq:ddpm_loss} means finding a parameter estimate $\hat{\vartheta}$ such that $s_{\hat{\vartheta}}(x,t)\approx\nabla\log f_t(x)$ for all $(x,t)\in\mathcal{X}\times[0,T]$. The random variable $X_t:=\exp(-t)X_0+\sqrt{1-\exp(-2t)}Z_t$ appearing in~\eqref{eq:ddpm_loss} is a weighted sum of the observation $X_0\sim\Prob_0$ and the independent standard Gaussian random variable $Z_t\sim\mathcal{N}(0,\mathbb{I}_d)$. One easily checks that $X_t\mid X_0\sim\mathcal{N}(\exp(-t)X_0,(1-\exp(-2t))\mathbb{I}_d)$. More details on DDPMs are given in Appendix~\ref{app:diffusion_models_ddpm_objective}.

While generative modeling via DDPMs is agnostic wrt the exact structure of the data-generating distribution, likelihood- or score-based parameter estimation requires a distributional assumption. This can be introduced into the loss by replacing the score function $s_\vartheta$ with $s_{\theta,t}:=\nabla\log f_{\theta,t}$, where $f_{\theta,t}$ is the density of the distribution of $X_t$ assuming that $X_0\sim\Prob_\theta$ with $\Prob_\theta$ being a distribution from the postulated family. The resulting loss function can be used to estimate the parameter $\theta$ of the model and it is exactly the loss considered by~\citet{ShaCheKli2023} and~\citet{CheKalMeh2025}, which we give in the subsequent definition of the DDSME.

\begin{definition}[DDSME]
	\label{def:dsme}
	Fix a terminal time $T>0$ and set the drift and diffusion coefficients $m_t:=\exp(-t)$ and $\sigma_t:=\sqrt{1-\exp(-2t)}$ for every $t\in[0,T]$. Further, given i.i.d.~data $\{X_0^{(i)}\}_{i=1}^n$ with $X_0^{(i)}\sim\Probtrue$, for each $i\in[n]$ and $t\in[0,T]$ define the noised observation
	\begin{equation}
		\label{eq:noised_sample}
		X_t^{(i)}:=m_tX_0^{(i)}+\sigma_tZ_t^{(i)},
	\end{equation}
	where $Z_t^{(i)}\sim\mathcal{N}(0,\mathbb{I}_d)$ is independent of $X_0^{(i)}$. Then, the DDSME is defined as the following minimizer:
	\begin{equation*}
	\hat{\theta}_{\DDSM}\in\argmin_{\theta\in\Theta}\left\{\frac{1}{n}\sum_{i=1}^n \int_0^T \E\left[\left\|s_{\theta,t}(X_t^{(i)})\right\|^2_2+\frac{2}{\sigma_t}\left\langle s_{\theta,t}(X_t^{(i)}),Z_t^{(i)}\right\rangle\Bigm\vert X_0^{(i)}\right]\dif t\right\}
\end{equation*}
In the above, $s_{\theta,t}:=\nabla\log f_{\theta,t}$ denotes the score function of the marginal distribution $\Prob_{\theta,t}$ of $X_t^{(i)}$, but assuming that $X_0^{(i)}\sim\Prob_\theta$, that is, $f_{\theta,t}$ is the density of $\Prob_{\theta,t}$.
\end{definition}

\begin{remark}[Marginal density under the model]
	\label{remark:defintion_convolved_density}
	The density $f_{\theta,t}$ is the density of the marginal distribution of the OU process that solves the SDE in~\eqref{eq:forward_process}, but with the initial condition $X_0^{(i)}\sim\Prob_\theta$. It can be derived more explicitly by using the transformation theorem for densities~\citep[Theorem 1.101 in][]{Kle2020} together with the fact that the density of the sum of two independent random variables is given by the convolution of their densities, which yields
	\begin{equation*}
		f_{\theta,t}(x)=\int_{\R^d}m_t^{-d}f_\theta\left(m_t^{-1}\tau\right)\sigma_t^{-d}\phi\left(\sigma_t^{-1}(x-\tau)\right)\dif \tau
		=\int_{\R^d}f_\theta(\tau)\sigma_t^{-d}\phi\left(\sigma_t^{-1}(x-m_t\tau)\right)\dif \tau,
	\end{equation*}
	for every $(\theta,t)\in\Theta\times(0,T]$. The second equality follows from the change of variable $u=m_t^{-1}\tau$.
\end{remark}

\paragraph{M-estimation.}

The SME and DDSME are M-estimators, that is, they are defined as minimizers of an empirical risk function. In particular, an M-estimator is defined as follows.
\begin{definition}[M-estimator]
\label{definition:m_estimator}
    Let $m:\Theta\times\samplespace\to\R\cup\{\infty\}$ be a function such that for each $\theta\in\Theta$ the map $x\mapsto m(\theta,x)$ is measurable and integrable wrt $\Probtrue$. Then, we call $m$ contrast (function) and 
\begin{equation*}
	\hat{\theta}_n \in \argmin_{\theta\in\Theta} \frac{1}{n}\sum_{i=1}^n m(\theta,X_i),
\end{equation*}
the M-estimator for the value $\theta_0$, which is defined as
\begin{equation}
	\label{eq:true_parameter}
	\theta_0 := \argmin_{\theta\in\Theta} \E_{\Probtrue}\left[m(\theta,X)\right],
\end{equation}
that is, $\theta_0$ is the minimizer of the corresponding population risk $\theta\mapsto\E_{\Probtrue}[m(\theta,X)]$. 
\end{definition}
In the case of SM and DDSM, the contrast functions are given by
\begin{align}
	\label{eq:contrast_function_sm}
	m_{\text{SM}}(\theta,x)&:=\left\|\nabla\log f_\theta(x)\right\|_2^2+2\divergence\nabla\log f_\theta(x),\\
	\label{eq:contrast_function_dsm}
	m_{\text{DDSM}}(\theta,x)&:=\int_0^T \E\left[\left\|\nabla\log f_{\theta,t}(X_t)\right\|^2_2+\frac{2}{\sigma_t}\left\langle \nabla\log f_{\theta,t}(X_t),Z_t\right\rangle\Bigm\vert X_0=x \right]\dif t,
\end{align}
with $X_t$, $Z_t$ and $\sigma_t$ as in Definition~\ref{def:dsme}. We will sometimes write $m_{\text{SM}}^{\Prob}$ and $m_{\text{SM}}^{\mathbb{Q}}$ to avoid ambiguity and emphasize the dependence on distinct families $(\Prob_\theta)_{\theta\in\Theta}$ and $(\mathbb{Q}_\theta)_{\theta\in\Theta}$, respectively. In light of this perspective, we can use the rich theory on M-estimation to obtain our results. The main theorems from this theory that are relevant to us are stated in Appendix~\ref{app:m_estimation_theory}. 

\section{Main results}\label{sec:main_results}

In this section, we present the main results of our paper. First, we state the assumptions our results are based on. Then, to motivate the need for an alternative score-based estimator, in Section~\ref{subsec:weakness_sm} we illustrate the weakness of vanilla SM by characterizing its poor properties in the model~\ref{eq:gm} when the separation of the modes grows. In Section~\ref{subsec:benefit_dsme} we show that the DDSME fulfills a statistical guarantee that does not suffer when the mode separation grows, in contrast to the error bound of the SME.

\subsection{Assumptions}\label{subsec:assumptions}

We rely on two assumptions for our results. The first one is the standard assumption of a well-specified model, which is common in mathematical statistics. The second one is a technical assumption on the parameter space, which is needed to avoid boundary issues, that would not only complicate the analysis but also allow for the mixture property of the family~\ref{eq:gm} to break in the limit.

\begin{assumption}[Well-specified model]
	\label{assumption:probtrue}
	The parametric model is correctly specified, such that there exists a true parameter value $\theta_0\in\Theta$ with $\Probtrue=\Prob_{\theta_0}$.
\end{assumption}

\begin{assumption}[Compact parameter space]
	\label{assumption:compact_parameter_space}
	The parameter space $\Theta$ is a compact subset of $(0,1)$ and has the form $\Theta:=[\eta,1-\eta]$, where $\eta\in(0,1/2)$ is chosen small.
\end{assumption}

\begin{remark}[Discussion of the assumptions]
	Assumptions~\ref{assumption:probtrue} and~\ref{assumption:compact_parameter_space} are standard in the literature, but we provide a short discussion on the motivation and possible relaxations.
	\begin{itemize}
		\item Even though Assumption~\ref{assumption:probtrue} is common, M-estimation theory generally allows for misspecification, that is, the case where $\Probtrue$ does not belong to the postulated family and $\theta_0$ is merely defined as the minimizer of $\theta\mapsto\E_{\Probtrue}[m(\theta,X)]$~\citep[cf.~Example 5.25 in Chapter 5 of][]{Vaa1998}. Note that Definition~\ref{definition:m_estimator} is formulated exactly in this spirit. But in order to prove the bounds in Lemmata~\ref{lemma:lipschitz_curvature_bounds_sm} and~\ref{lemma:lipschitz_curvature_bounds_dsm} we relied on the fact that the convolution of $\Probtrue$ with a Gaussian will still be a Gaussian mixture, that is, the closedness of Gaussians under convolution. Whenever $\Probtrue$ is not an element of~\ref{eq:gm}, this does not hold. Nonetheless, we focus on the well-specified case for simplicity. We expect that our results can be extended to misspecified settings, which could be characterized as an explicit misspecification of the model structure~\citep{DwiHoKha2018}, or more generally in terms of the distance between the true distribution and the model~\citep[cf.~][]{Whi1982,KleVaa2012} or in the spirit of the $\varepsilon$-contamination model by~\citet{Hub1964}.
		\item Assumption~\ref{assumption:compact_parameter_space} is made mostly for technical convenience. Strategies for relaxations to non-compact parameter spaces exist~\cite[cf.~Chapter 5.2.1 in][]{Vaa1998}, however we do not pursue this direction here.
	\end{itemize}
\end{remark}

\subsection{The weakness of vanilla score matching}\label{subsec:weakness_sm}

In order to illustrate the weakness of the vanilla SME, we adopt the isoperimetry-based perspective initiated by~\citet{KoeHecRis2022}. More precisely, we show that the isoperimetric constant of the statistical model~\ref{eq:gm} can be used to characterize the flatness of the SM loss, as well as the exploding asymptotic variance of the SME. By Definition~\ref{definition:minkowski_content_isoperimetric_constant} in Appendix~\ref{app:isoperimetric_constant}, the isoperimetric constant $\CIP(\mathcal{F})$ of a family $\mathcal{F}$ of probability measures defined on the measurable space $(\Omega,\mathcal{A})$ is the largest constant that satisfies the isoperimetric inequality
\begin{equation}
	\label{eq:isoperimetric_inequality}
\left\{\Prob(A)\wedge\Prob(A^\complement)\right\}\CIP(\mathcal{F})\leqslant\Prob^+(A)
\end{equation}
for all measurable sets $A\in\mathcal{A}$ and all distributions $\Prob\in\mathcal{F}$, and where $\Prob^+(A)$ denotes the boundary measure of $A$ wrt $\Prob$ (see Definition~\ref{definition:minkowski_content_isoperimetric_constant}). Loosely speaking, the isoperimetric constant of a distribution will be small, if there exists a set $A$ that has non-trivial mass and is surrounded by a low-probability region, such that the boundary measure $\Prob^+(A)$ is small compared to $\Prob(A)$ or $\Prob(A^\complement)$. \citet{KoeHecRis2022} call such a set a \textit{sparse cut}, terminology borrowed from graph theory. Clearly, multimodal distributions with well-separated modes admit sparse cuts, while unimodal distributions do not.

Computing the isoperimetric constant for a given probability measure is mostly intractable. A well-known exception are Gaussian distributions in $\R^d$, for which the isoperimetric problem has been studied intensily~\citep[cf.][]{Led1996}. In this case,~\citet{Bor1975} and~\citet{SudTsi1978} showed independently that half-spaces of the form $H_{u,t}:=\{x\in\R^d\mid\langle u,x\rangle\leqslant t\}$ with $u\in\mathbb{S}^{d-1}:=\{x\in\R^d:\|x\|_2=1\}$ and $t\in\R$ are extremal sets, that is, they achieve equality in~\eqref{eq:isoperimetric_inequality}. This fact makes the task tractable and as a consequence, the isoperimetric constant of a Gaussian distribution in $\R^d$ can be computed in closed form.

\begin{proposition}[Isoperimetric constant of Gaussian distributions in $\R^d$]
	\label{proposition:gaussian_isoperimetric_constant}
	Let $\Prob=\mathcal{N}(\mu,\Sigma)$ with mean $\mu\in\R^d$ and positive definite covariance matrix $\Sigma\in\R^{d\times d}$. Then, the isoperimetric constant of $\Prob$ is given by
	\begin{equation*}
		\CIP(\Prob)=\frac{2\varphi(0)}{\sqrt{\lambda_{\max}(\Sigma)}},
	\end{equation*}
	where $\lambda_{\max}(\Sigma)$ denotes the largest eigenvalue of $\Sigma$.
\end{proposition}

Being an immediate consequence of~\citet{Bor1975} and~\citet{SudTsi1978}, the optimal constant for the Gaussian isoperimetric inequality given in Proposition~\ref{proposition:gaussian_isoperimetric_constant} is known in the literature~\citep[cf.~discussion after Lemma 1.6 in][]{BobHou1998}, but its proof is often omitted. Therefore, and because we rely on its monotonicity argument in the proof of Proposition~\ref{proposition:isoperimetric_constant}, we provide it in Appendix~\ref{app:proposition_gaussian_isoperimetric_constant}.

When it comes to mixtures of Gaussians, much less is known. \citet{CanMirVit2010} studied isoperimetric problems in Euclidean space endowed with a density. The authors conjectured that a finite mixture of Gaussians in $\R^d$ with means on the $x$-axis has half-spaces bounded by vertical hyperplanes as extremal sets. More recently,~\citet{BerDanLia2018} and~\citet{Lie2026} studied the isoperimetric problem for mixtures of Gaussians in more detail. However, even though~\citet{Lie2026} provides an explicit formula for $\CIP(\Prob_{1/2})$, the author does not answer the question whether this constant is optimal for the entire family~\ref{eq:gm}. In Proposition~\ref{proposition:isoperimetric_constant} we answer it in the affirmative, that is, we show that $\CIP((\Prob_\theta)_{\theta\in\Theta})=\CIP(\Prob_{1/2})$. In order to do so, we require the following two results. A proof of Lemma~\ref{lemma:properties_isoperimetric_constants} can be found in Appendix~\ref{app:lemma_properties_isoperimetric_constants}.

\begin{lemma}[Properties of isoperimetric constants]
	\label{lemma:properties_isoperimetric_constants}
	The following is true:
	\begin{enumerate}
		\item For a probability measure $\Prob$ on $\R^d$ and its pushforward $\mathbb{Q}:=\Prob\circ T^{-1}$ under the isometry $T:\R^d\to\R^d$, it holds that $\CIP(\Prob)=\CIP(\mathbb{Q})$.
		\item For a distribution $\Prob_\theta$ from the family~\ref{eq:gm} and $\mathbb{Q}_\theta=\mathbb{Q}_{\theta,1}\otimes\mathcal{N}(0,\mathbb{I}_{d-1})$ as constructed in Lemma~\ref{lemma:embedded_one_dim_sym_mixture}, it holds that $\CIP(\Prob_\theta)=\CIP(\mathbb{Q}_{\theta,1})$.
	\end{enumerate}
\end{lemma}

\begin{theorem}[Theorem 1.3 in~\citet{BobHou1997}]
	\label{theorem:isoperimetric_constant}
	Let $F(x) = \mu((-\infty, x])$ be the cumulative distribution function of a probability measure $\mu$ on the real line ($\mu$ is not the unit mass at a point) and let $f$ be the density of its absolutely continuous part. Then,
	\begin{equation}
		\label{eq:real_line_isoperimetric_constant}
		\CIP(\mu)=\essinf_{a\leqslant x\leqslant b}\frac{f(x)}{F(x)\wedge \{1-F(x)\}},
	\end{equation}
	where $a:=\inf\{x\in\R:F(x)>0\}$ and $b:=\sup\{x\in\R:F(x)<1\}$.
\end{theorem}

While 2. of Lemma~\ref{lemma:properties_isoperimetric_constants} allows us to reduce the problem of computing the isoperimetric constant of the family~\ref{eq:gm} to the symmetric one-dimensional case, Theorem~\ref{theorem:isoperimetric_constant} provides a closed-form expression for the isoperimetric constant of a probability measure on the real line. With these tools, we can obtain an analogous result to Proposition~\ref{proposition:gaussian_isoperimetric_constant}, but for the entire family~\ref{eq:gm}. The need to optimize wrt $\theta$ and $x$ simultaneously requires slightly more involved arguments. The proof of the following proposition is defferred to Appendix~\ref{appendix:proof_isoperimetric_constant}.

\begin{proposition}[Isoperimetric constant of the family~\ref{eq:gm}]
	\label{proposition:isoperimetric_constant}
	Let $(\Prob_\theta)_{\theta\in\Theta}$ be the family in~\eqref{eq:gm} with parameter set $\Theta=[0,1]$ and separation $2\Delta=\|\mu_1-\mu_2\|_2$. Then, it holds that
	\begin{equation*}
		\CIP((\Prob_\theta)_{\theta\in\Theta})=2\varphi(\Delta),
	\end{equation*}
	such that $\CIP((\Prob_\theta)_{\theta\in\Theta})\downarrow0$ as $\Delta\to\infty$.
\end{proposition}

Proposition~\ref{proposition:isoperimetric_constant} formalizes the intuition around sparse cuts for the family~\ref{eq:gm}: As $\Delta$ grows, the modes of the mixture become more and more separated, creating a low-probability region between them. This region allows for a sparse cut, implying that the isoperimetric constant of the family decays to zero as $\Delta\to\infty$. Before we state the consequences of Proposition~\ref{proposition:isoperimetric_constant} for the SM loss and the asymptotic variance of the SME, we gather some properties of the family~\ref{eq:gm}. These will be used not only in the proof of Proposition~\ref{proposition:score_bound}, but also in the proofs of Lemmata~\ref{lemma:reduction_constants},~\ref{lemma:lipschitz_curvature_bounds_sm} and~\ref{lemma:lipschitz_curvature_bounds_dsm} in Section~\ref{subsec:benefit_dsme}.

\begin{lemma}[Properties of the family~\ref{eq:gm}]
	\label{lemma:gm_properties}
	Let $(\Prob_\theta)_{\theta\in\Theta}$ be the family given in~\eqref{eq:gm} and $(\mathbb{Q}_\theta)_{\theta\in\Theta}$ the corresponding pushforwards with marginals $(\mathbb{Q}_{\theta,1})_{\theta\in\Theta}$ from Lemma~\ref{lemma:embedded_one_dim_sym_mixture}. Then, the following holds true:
	\begin{enumerate}
		\item For every $\theta\in(0,1)$ and every $x=(x_1,\ldots,x_d)\in\R^d$ it holds that
		\begin{align*}
			\nabla\log f_\theta(x)&=(2w_\theta(x)-1)\mathtt{d}+\mathtt{m}-x,\\
			\partial_{x_1}\log q_{\theta,1}(x_1)&=(2w_{\theta,1}(x_1)-1)\Delta-x_1,
		\end{align*}
		with weight functions $w_\theta(x):=\theta\phi(x-\mu_1)/f_\theta(x)$ and $w_{\theta,1}(x_1):=\theta\varphi(x_1-\Delta)/q_{\theta,1}(x_1)$, midpoint $\mathtt{m}:=(\mu_1+\mu_2)/2$, direction of separation $\mathtt{d}:=(\mu_1-\mu_2)/2$ and $\Delta=\|\mathtt{d}\|_2$.
		\item The weight functions $w_\theta$ and $w_{\theta,1}$ are generalized logistic functions:
		\begin{align*}
			w_\theta(x)&=\sigma\left(\langle x,\mu_1-\mu_2\rangle+\frac{1}{2}\left(\|\mu_2\|_2^2-\|\mu_1\|_2^2\right)+\log\left(\frac{\theta}{1-\theta}\right)\right),\\
			w_{\theta,1}(x_1)&=\sigma\left(2\Delta \left(x_1-\frac{\log\left(\theta/(1-\theta)\right)}{2\Delta}\right)\right),
		\end{align*}
		where $\sigma(z)=(1+\exp(-z))^{-1}$ is the standard logistic function.
		\item For all $\theta,\theta^*\in(0,1)$ and every $x=(x_1,\ldots,x_d)\in\R^d$ it holds that
		\begin{enumerate}
			\item $\displaystyle\|\nabla\log f_\theta(x)-\nabla\log f_{\theta^*}(x)\|_2=2\Delta|\theta-\theta^*|\frac{\phi(x-\mu_1)\phi(x-\mu_2)}{f_\theta(x)f_{\theta^*}(x)}$,
		\item $\|\nabla\log q_\theta(x)-\nabla\log q_{\theta^*}(x)\|_2=|\partial_{x_1}\log q_{\theta,1}(x_1)-\partial_{x_1}\log q_{\theta^*,1}(x_1)|$,
		\item $\FI(\Prob_\theta,\Prob_{\theta^*})=\FI(\mathbb{Q}_{\theta,1},\mathbb{Q}_{\theta^*,1})$.
	\end{enumerate}
	\item For all $\theta,\theta^*\in(0,1)$ and every $x_1\in\R$ it holds that
	\begin{equation*}
		\exp(-2\Delta|x_1|)\leqslant\frac{\varphi(x_1-\Delta)\varphi(x_1+\Delta)}{q_{\theta,1}(x_1)q_{\theta^*,1}(x_1)}\leqslant \frac{\exp(-2\Delta|x_1|)}{\theta\theta^*\wedge(1-\theta)(1-\theta^*)}.
	\end{equation*}
		\item For every $\theta\in(0,1)$ and every $t\geqslant \Delta$ it holds that
		\begin{equation*}
			\E_{\mathbb{Q}_{\theta,1}}\left[\exp(-t |X|)\right]\leqslant\left\{\sqrt{\frac{\pi}{2}}\wedge\left(\frac{1}{t-\Delta}+\frac{1}{t+\Delta}\right)\right\}\varphi(\Delta).
		\end{equation*}
		\item For every $\theta\in(0,1)$ and every $x=(x_1,\ldots,x_d)\in\R^d$ it holds that
		\begin{align*}
			m_{\SM}^\mathbb{Q}(\theta,x)=m_{\SM}^{\mathbb{Q}_1}(\theta,x_1)+\mathtt{r}(x_{2:d}),
		\end{align*}
		where $m_{\SM}^\mathbb{Q}$ and $m_{\SM}^{\mathbb{Q}_1}$ denote the contrasts based on the families $(\mathbb{Q}_\theta)_{\theta\in\Theta}$ and $(\mathbb{Q}_{\theta,1})_{\theta\in\Theta}$, respectively, and $\mathtt{r}$ is a function independent of $\theta$.
	\end{enumerate}
\end{lemma}

A proof of Lemma~\ref{lemma:gm_properties} is provided in Appendix~\ref{appendix:proof_lemma_gm_properties}. Based on Lemma~\ref{lemma:gm_properties} and Proposition~\ref{proposition:isoperimetric_constant}, we are now in a position to state and prove the main result of Section~\ref{subsec:weakness_sm}. It shows that the isoperimetry of the family~\ref{eq:gm} as a whole, that is, the smallest constant across the entire family (see Definition~\ref{definition:minkowski_content_isoperimetric_constant}), determines the flatness of the loss and the large asymptotic variance for \textit{all} distributions in the family. Simply put, estimating a ground truth parameter $\theta_0\in\{\eta,1-\eta\}$ for some small $\eta>0$ via SM is still difficult, even though the corresponding distributions $\Prob_{\theta_0}$ are almost Gaussians and do not have \textit{bad} isoperimetry, in the sense of a small isoperimetric constant.

\begin{proposition}[Bounds for SM loss and AVar of SME]
	\label{proposition:score_bound}
	Let $(\Prob_\theta)_{\theta\in\Theta}$ be the family given in~\eqref{eq:gm} with separation $2\Delta=\|\mu_1-\mu_2\|_2$. Then, the following holds true:
	\begin{enumerate}
		\item Suppose that Assumption~\ref{assumption:compact_parameter_space} is satisfied. Then, it holds that
        \begin{equation*}
            \FI(\Prob_\theta,\Prob_{\theta^*})\lesssim_{\eta}\Delta\CIP((\Prob_\theta)_{\theta\in\Theta})
        \end{equation*}
        for every $\theta,\theta^*\in\Theta$.
		\item Suppose that Assumptions~\ref{assumption:probtrue} and~\ref{assumption:compact_parameter_space} are satisfied. Then, it holds that
		\begin{equation*}
			\sqrt{n}\left(\hat{\theta}_{\SM}-\theta_0\right)\xrightarrow{d}\mathcal{N}\left(0,\text{\normalfont AVar}[\hat{\theta}_{\SM}]\right),
		\end{equation*}
		where $\text{\normalfont AVar}[\hat{\theta}_{\SM}]\gtrsim_{\eta} \Delta\CIP((\Prob_\theta)_{\theta\in\Theta})^{-1}$.
	\end{enumerate}
\end{proposition}

\begin{proof}
	To show 1. we first use 3c. and 4. of Lemma~\ref{lemma:gm_properties} to arrive at
	\begin{equation*}
		\FI(\Prob_\theta,\Prob_{\theta^*})\leqslant4\Delta^2\left(\frac{|\theta-\theta^*|}{\theta\theta^*\wedge(1-\theta)(1-\theta^*)}\right)^2\int_\R\exp(-4\Delta|x_1|)q_{\theta,1}(x_1)\dif x_1,
	\end{equation*}
	for all $\theta,\,\theta^*\in\Theta$. Applying 5. of Lemma~\ref{lemma:gm_properties} to the above integral then gives
	\begin{equation*}
		\FI(\Prob_\theta,\Prob_{\theta^*})\leqslant4\Delta^2\left(\frac{1-\eta}{\eta}\right)^2\left\{\sqrt{\frac{\pi}{2}}\wedge\frac{8}{15\Delta}\right\}\varphi(\Delta)\leqslant\frac{32\Delta}{15}\left(\frac{1-\eta}{\eta}\right)^2\varphi(\Delta).
	\end{equation*}
	and with Proposition~\ref{proposition:isoperimetric_constant} the claim follows.

	Regarding 2., from Theorem 5.23 in~\citet{Vaa1998} it follows that
	\begin{equation*}
			\sqrt{n}\left(\hat{\theta}_{\SM}-\theta_0\right)\xrightarrow{d}\mathcal{N}\left(0,\text{\normalfont AVar}[\hat{\theta}_{\SM}]\right)\quad\text{with}\quad \text{\normalfont AVar}[\hat{\theta}_{\SM}]:=\frac{\E[(\partial_\theta m_{\SM}(\theta,X)|_{\theta=\theta_0})^2]}{\E[\partial_\theta^2 m_{\SM}(\theta,X)|_{\theta=\theta_0}]^2}.
		\end{equation*}
	In view of 6. of Lemma~\ref{lemma:gm_properties}, we can reduce the problem to the one-dimensional case, that is, with $X_1\sim\mathbb{Q}_{\theta,1}$ we have
	\begin{equation*}
		\text{\normalfont AVar}[\hat{\theta}_{\SM}]=\frac{\E[(\partial_\theta m_{\SM}^{q_{\theta,1}}(\theta,X_1)|_{\theta=\theta_0})^2]}{\E[\partial_\theta^2 m_{\SM}^{q_{\theta,1}}(\theta,X_1)|_{\theta=\theta_0}]^2}.
	\end{equation*}
	For the numerator we have the lower bound
	\begin{equation*}
		\E[(\partial_\theta m_{\SM}^{q_{\theta,1}}(\theta,X_1)|_{\theta=\theta_0})^2]\gtrsim_{\eta} \Delta^3\exp\left(-\frac{\Delta^2}{2}\right),
	\end{equation*}
	which follows with the same arguments as in the proof of 1. of Lemma~\ref{lemma:lipschitz_curvature_bounds_sm}. Note that~\citet{Dia2023} erroneously gives $\Delta^4$ as the leading polynomial term. The upper bound of the denominator shown by~\citet{Dia2023} is
	\begin{equation*}
		\E[\partial_\theta^2 m_{\SM}^{q_{\theta,1}}(\theta,X_1)|_{\theta=\theta_0}]^2\lesssim_{\eta}\Delta^2\exp(-\Delta^2),
	\end{equation*}
	and it holds for any $\theta_0\in\Theta$. Combining the two bounds gives
	\begin{equation*}
		\text{\normalfont AVar}[\hat{\theta}_{\SM}]\gtrsim_{\eta}\Delta\exp\left(\frac{\Delta^2}{2}\right),
	\end{equation*}
	such that the result follows with Proposition~\ref{proposition:isoperimetric_constant}.
\end{proof}

Proposition~\ref{proposition:score_bound} contains the main message of the subsection: the poor properties of the SME in the multimodal setting with well-separated modes can be characterized in terms of the vanishing curvature of the loss function or the exploding asymptotic variance of the estimator. And both characterizations can be parsed in terms of the isoperimetric constant of the postulated family or its reciprocal. In other words, the distribution with the \textit{worst} isoperimetry of the entire family affects the estimation via SM for any of the distributions in the family. To conclude the subsection, we briefly discuss how Proposition~\ref{proposition:score_bound} generalizes and relates to previous work.

\begin{remark}[Novelty and connection to previous work]
The bounds stated in Proposition~\ref{proposition:score_bound} are novel and extend previously known results as follows:
	\begin{itemize}
		\item The main novelty is to establish that the performance of the SME is governed by the isoperimetric properties of the model class itself. In particular, favorable (unfavorable) isoperimetric behavior of the family translates into favorable (unfavorable) estimation performance uniformly over distributions in the family.
		\item The upper bound given in the first item is a generalization of Proposition 1 in~\citet{CheBalErd2022}. First, it holds for any two mixtures from the family~\ref{eq:gm}, which is more general than the family considered in~\citet{CheBalErd2022}. Second, it provides a bound in terms of $\CIP((\Prob_\theta)_{\theta\in\Theta})$.
		\item The lower bound in Proposition 2 can be viewed as an extension of the lower-bound result in Theorem 3 of~\cite{KoeHecRis2022} from natural exponential families with a given sufficient statistic to the broader family~\ref{eq:gm}. Moreover, whereas the result of~\cite{KoeHecRis2022} does not appear to admit an explicit interpretation in terms of the inverse isoperimetric constant of the model family, the present formulation establishes such a geometric connection directly.
		\item Proposition~\ref{proposition:isoperimetric_constant} fits into a broader literature demonstrating that isoperimetric quantities govern the performance of statistical procedures. In estimation problems, $\CIP^{-1}$ often appears in error bounds, quantifying the adverse impact of multimodality on statistical guarantees~\citep[cf.][]{GroRat2019,GroRat2021,SauWel2019}. In contrast, clustering methods benefit from well-separated modes, leading to guarantees that scale with $\CIP$ rather than its inverse~\citep[cf.][]{GreBalTib2021,TriHofHos2021}.
	\end{itemize}
\end{remark}

\subsection{The benefit of diffusion-based denoising score matching}\label{subsec:benefit_dsme}

In this subsection, we present the main results of our paper in the form of error bounds for the SME and DDSME. To prove these error bounds, we first derive intermediate results. The latter include the crucial observation that the random variable $X_t$ appearing in the DDSM loss given in Definition~\ref{def:dsme} follows a mixture distribution from~\ref{eq:gm} for all $t\geqslant0$ if the distribution of $X_0$ also lies in the family~\ref{eq:gm}. A proof of the following lemma can be found in Appendix~\ref{appendix:proof_lemma_ddpm_density}.

\begin{lemma}[Marginal distribution of the OU process when $X_0\sim\Prob_\theta$]
	\label{lemma:ddpm_density_and_score}
	Let $\Prob_\theta$ be a distribution from the family of distributions given in~\eqref{eq:gm} and assume $X_0\sim\Prob_\theta$. Then, the distribution $\Prob_{\theta,t}$ of $X_t=\exp(-t)X_0+\sqrt{1-\exp(-2t)}Z_t$ with $Z_t\sim\mathcal{N}(0,\mathbb{I}_d)$ independent of $X_0$ is also a mixture of two Gaussians and its density is given by
	\begin{equation*}
		f_{\theta,t}(x)=\theta\phi(x-\mu_{1,t})+(1-\theta)\phi(x-\mu_{2,t}),
	\end{equation*}
	with time-dependent mean parameters $\mu_{1,t}=\exp(-t)\mu_1$ and $\mu_{2,t}=\exp(-t)\mu_2$.
\end{lemma}

Next, we state a non-asymptotic error bound for M-estimators with contrast functions that are Lipschitz continuous in its parameter. It is a consequence of the more general rate theorem for M-estimators given in Theorem~\ref{theorem:rate_theorem} in Appendix~\ref{app:m_estimation_theory}. Based on the particular form of the error bound in Proposition~\ref{proposition:m_estimator_error_bound}, we will motivate our strategy to obtain the error bounds for the SME and DDSME. First, we need the following definition.

\begin{definition}[Local curvature]
	\label{definition:local_curvature_condition}
    Let $m$ be a contrast according to Definition~\ref{definition:m_estimator}. Then, we say that $m$ fulfills a local curvature condition with $\alpha$ and constant $\mathtt{Curv}(m)$, if there exists $\alpha>0$ and $\mathtt{C}:=\mathtt{Curv}(m)$ such that
    \begin{equation*}
		\inf_{\|\theta-\theta_0\|_2 \geqslant \delta} \E_0[m(\theta,X) - m(\theta_0,X)]\geqslant \mathtt{C} \delta^{\alpha}
	\end{equation*}
    holds for sufficiently small $\delta>0$ and the constant $\mathtt{C}$ only depends on $m$.
\end{definition}

\begin{proposition}[Error bound for M-estimators with Lipschitz contrast function]
	\label{proposition:m_estimator_error_bound}
    Let $m: \Theta\times \samplespace\rightarrow\R\cup\{\infty\}$ be a contrast and $\hat {\theta}_n=\mathop{\textnormal{argmin}}_{\theta\in \Theta}\frac{1}{n}\sum_{i=1}^n m(\theta,X_i)$ its corresponding M-estimator as defined in Definition~\ref{definition:m_estimator} with $\Theta\subseteq\R^k$. Assume that the following is true:
	\begin{enumerate}
		\item\label{assumption:lipschitz} $\theta\mapsto m(\theta,x)$ is Lipschitz for every $x\in \samplespace$ with Lipschitz constant $\mathtt{L}(x)$.
		\item\label{assumption:curvature_condition} $m$ fulfills a local curvature condition with $\alpha>1$ and constant $\mathtt{C}_1$.
		\item\label{assumption:consistency} $\hat {\theta}_n$ converges in probability to $\theta_0$.
	\end{enumerate}
	Then, for any $0<\delta<\mathtt{C}K_{\alpha}\|\mathtt{L}\|_{L^2(\Probtrue)}\sqrt{k}/\mathtt{C}_1$ there exists $n_\delta\in\N$ such that for every $n\geqslant n_\delta$ the M-estimator $\hat {\theta}_n$ satisfies
    \begin{equation*}
            \Probtrue\left(\|\hat\theta_n-\theta_0\|_2\leqslant \left(\frac{\mathtt{C}K_\alpha\|\mathtt{L}\|_{L^2(\Prob_0)}\sqrt{k}}{\delta\mathtt{C}_1\sqrt{n}}\right)^{\frac{1}{\alpha-1}}\right)\geqslant 1-2\delta.
    \end{equation*}
    with $\mathtt{C}>0$ and $K_{\alpha}:=2^{2\alpha}/(2^{\alpha-1}-1)$.
\end{proposition}

\begin{remark}[Eventual sample-size threshold $n_\delta$]
	The above result essentially coincides with Theorem 4 of~\cite{LiuKanWil2022} and Proposition 36 of~\cite{Cao2025}. However, in both references the authors overlook that upgrading the asymptotic statement given in Theorem 5.52 of~\cite{Vaa1998} to an explicit high-probability guarantee, which holds for a fixed sample size $n$, requires that $n\geqslant n_\delta$ eventually for any admissable $\delta$. For details see the proof of Proposition~\ref{proposition:m_estimator_error_bound} in Appendix~\ref{app:m_estimation_theory}.
\end{remark}

The proof of Proposition~\ref{proposition:m_estimator_error_bound} is given towards the end of Appendix~\ref{app:error_bounds_for_m_estimators}. In the above error bound, the norm $\|\mathtt{L}\|_{L^2(\Probtrue)}$ and the curvature constant $\mathtt{C}_1$ are of particular interest. The former can be interpreted as a measure of variability of the contrast function, while the latter quantifies the curvature of the population risk. Thus, both quantities capture the difference between the loss functions in Figure~\ref{fig:losses}. To derive our main contribution, the error bounds for SME and DDSME in Theorems~\ref{theorem:error_bound_sme} and~\ref{theorem:error_bound_ddsme}, we will bound both constants for the SME and DDSME in Lemmata~\ref{lemma:lipschitz_curvature_bounds_sm} and~\ref{lemma:lipschitz_curvature_bounds_dsm}. In order to do so, we need the following lemma, which is proven in Appendix~\ref{app:proof_lemma_reduction_constants}.

\begin{lemma}[Reduction of Lipschitz and curvature constants]
	\label{lemma:reduction_constants}
	Consider the contrasts $m_{\bullet}^\Prob$ and $m_{\bullet}^{\mathbb{Q}_{1}}$ with $\bullet\in\{\SM,\DDSM\}$. Further, let $\mathtt{Lip}(f)$ and $\mathtt{Curv}(f)$ denote the Lipschitz and curvature constant of a function $f$, respectively. Then, for $\bullet\in\{\SM,\DDSM\}$ and any $\theta_0\in\Theta$ the following holds true:
	\begin{enumerate}
		\item $\|\mathtt{Lip}(m_{\bullet}^{\Prob})\|_{L^2(\Prob_{\theta_0})}=\|\mathtt{Lip}(m_{\bullet}^{\mathbb{Q}_{1}})\|_{L^2(\mathbb{Q}_{\theta_0,1})}$
		\item $\mathtt{Curv}(m_{\bullet}^{\Prob})=\mathtt{Curv}(m_{\bullet}^{\mathbb{Q}_{1}})$
	\end{enumerate}
\end{lemma}

\begin{remark}[Notation]
	Both constants are to be understood wrt $\theta$, but we leave the dependence implicit in the statement to not overload the notation. That is, $\mathtt{Lip}(m_{\bullet}^{\Prob})$ refers to the Lipschitz constant of the map $\theta\mapsto m_{\bullet}^{\Prob}(\theta,x)$ for fixed $x\in\R^d$ and $\mathtt{Curv}(m_{\bullet}^{\Prob})$ corresponds to the curvature constant as in Definition~\ref{definition:local_curvature_condition}.
\end{remark}

\begin{lemma}[Lipschitz and curvature bounds for SM]
	\label{lemma:lipschitz_curvature_bounds_sm}
	Let $(\Prob_\theta)_{\theta\in\Theta}$ be the family in~\eqref{eq:gm} with separation $2\Delta=\|\mu_1-\mu_2\|_2$. Grant Assumptions~\ref{assumption:probtrue} and~\ref{assumption:compact_parameter_space}, in particular, fix $\eta\in(0,1/2)$ and set $\Theta:=[\eta,1-\eta]$. Then, the contrast function $m_{\SM}$ given in~\eqref{eq:contrast_function_sm} has the following properties:
	\begin{enumerate}
		\item It is Lipschitz continuous in $\theta$ with Lipschitz constant $\mathtt{L}_{\SM}$ and it holds that
	\begin{equation}
		\label{eq:lipschitz_constant_norm_sm}
		\|\mathtt{L}_{\SM}(X)\|_{L^2(\Probtrue)}\gtrsim_{\eta} \Delta^{3/2}\exp\left(-\frac{\Delta^2}{4}\right)
	\end{equation}
    \item It fulfills a curvature condition with $\alpha=2$ and constant $\mathtt{C}_{\SM}$, which satisfies
    \begin{equation*}
        \mathtt{C}_{\SM}\lesssim_{\eta}  \Delta\exp\left(-\frac{\Delta^2}{2}\right).
    \end{equation*}
	\end{enumerate}
\end{lemma}

\begin{remark}[Proper scoring rules]
	Note that the upper bound in 2.~of Lemma~\ref{lemma:lipschitz_curvature_bounds_sm} and the upper bound in 1.~of Proposition~\ref{proposition:score_bound} coincide up to multiplicative constants. This is due to the fact that $m_{\SM}$ is a proper scoring rule and every proper scoring rule induces a divergence~\citep{ParDawLau2012,ForLau2015}. This implies that
	\begin{equation}
		\label{eq:scoring_rule_identity}
		\FI(\Prob_{\theta_0},\Prob_\theta)=\E_{\theta_0}\left[m_{\SM}(\theta,X)-m_{\SM}(\theta_0,X)\right],
	\end{equation}
	such that the curvatures of $\theta\mapsto\FI(\Prob_{\theta_0},\Prob_\theta)$ and $\theta\mapsto\E_{\theta_0}[m_{\SM}(\theta,X)]$ are identical. See the proof of 2.~of Lemma~\ref{lemma:lipschitz_curvature_bounds_sm} for details on the derivation of~\eqref{eq:scoring_rule_identity}.
\end{remark}

\begin{lemma}[Lipschitz and curvature bounds for DDSM]
  \label{lemma:lipschitz_curvature_bounds_dsm}
  Let $(\Prob_\theta)_{\theta\in\Theta}$ be the family in~\eqref{eq:gm} with separation $2\Delta=\|\mu_1-\mu_2\|_2$. Grant Assumptions~\ref{assumption:probtrue} and~\ref{assumption:compact_parameter_space}, in particular, fix $\eta\in(0,1/2)$ and set $\Theta:=[\eta,1-\eta]$. Then, the contrast function $m_{\DDSM}$ given in~\eqref{eq:contrast_function_dsm} has the following properties:
 \begin{enumerate}
	\item It is Lipschitz continuous in $\theta$ with Lipschitz constant $\mathtt{L}_{\DDSM}$ and for every $T>0$ it holds that
  \begin{equation}
	\label{eq:lipschitz_constant_norm_dsm}
	\|\mathtt{L}_{\DDSM}(X_0)\|_{L^2(\Probtrue)}\lesssim_{\eta} \psi(\Delta,T):=  T\E\left[g(\Delta_U)\exp\left(-\frac{\Delta_U^2}{8}\right)\right],
  \end{equation}
  where $g(x)=x^{7/4}+x^{3/4}(3+6x^2+x^4)^{1/4}+x^{3/2}$, $\Delta_t:=\exp(-t)\Delta$ and the expectation on the right-hand side is taken wrt $U\sim\text{\normalfont Uniform}([0,T])$. Further, if $T>0\vee\log\Delta$, then
  \begin{equation*}
	\|\mathtt{L}_{\DDSM}(X_0)\|_{L^2(\Probtrue)}\leqslant \mathtt{C}(\eta),
  \end{equation*}
  for an explicit constant $\mathtt{C}(\eta)$ independent of $\Delta$.
  	\item It fulfills a curvature condition with $\alpha=2$ and constant $\mathtt{C}_{\DDSM}$, which satisfies for every $T>0$
    \begin{align*}
	\mathtt{C}_{\DDSM}\gtrsim\xi(\Delta,T):=
	\begin{cases}
		\Delta^2(1-\exp(-2T)),&\text{if }\Delta\leqslant1,\\
		 1-\Delta^2\exp(-2T),&\text{if }\Delta>1.
	\end{cases}
\end{align*}
	In particular, if $\Delta>1$ and $T\geqslant\log(\Delta^m)$ for $m>1$, then $\mathtt{C}_{\DDSM}\gtrsim 1-\Delta^{-2(m-1)}$.
 \end{enumerate} 
\end{lemma}

Lemmata~\ref{lemma:lipschitz_curvature_bounds_sm} and~\ref{lemma:lipschitz_curvature_bounds_dsm} are proven in Appendices~\ref{app:proof_lemma_lipschitz_curvature_bounds_sm} and~\ref{app:proof_lemma_lipschitz_curvature_bounds_dsm}, respectively, and they confirm an important fact: while $\|\mathtt{L}_{\SM}\|_{L^2(\Probtrue)}$ and $\mathtt{C}_{\SM}$ decay exponentially in $\Delta^2$, the respective DDSM quantities do not exhibit this type of behaviour, provided that the hyperparameter $T$ is chosen sufficiently large. Comparing the bound for $\|\mathtt{L}_{\SM}\|_{L^2(\Probtrue)}$ in~\eqref{eq:lipschitz_constant_norm_sm} with the bound for $\|\mathtt{L}_{\DDSM}\|_{L^2(\Probtrue)}$ in~\eqref{eq:lipschitz_constant_norm_dsm} reveals why $T$ needs to be large. To see this, one may rewrite the right-hand side of~\eqref{eq:lipschitz_constant_norm_dsm} into 
\begin{equation*}
	T\E\left[g(\Delta_U)\exp\left(-\frac{\Delta_U^2}{8}\right)\right]=(\Delta-\exp(-T)\Delta)\E\left[g(\widetilde{U})\exp\left(-\frac{\widetilde{U}^2}{8}\right)\right],
\end{equation*}
where $\widetilde{U}\sim\text{\normalfont Uniform}([\exp(-T)\Delta,\Delta])$. The above expectation on the right-hand side is a uniform average of the decaying function $h(x):=g(x)\exp(-x^2/8)$ over the interval $[\exp(-T)\Delta,\Delta]$. The lower bound for $\|\mathtt{L}_{\SM}\|_{L^2(\Probtrue)}$ is given by evaluating $k(x)=x\exp(-x^2/2)$ at $x=\Delta$. Clearly, $k\in o(h)$, that is, for all constants $\mathtt{C}>0$ there exists $x_0\in\R$ such that $k(x) \leqslant \mathtt{C} h(x)$ for all $x\geqslant x_0$. But the expectation in the above display will only take into account realizations smaller than $\Delta$ and thus the expectation will be larger than $k(\Delta)$ for any $T>0$. And choosing $T$ large shifts more weight towards the origin, which in turn increases the expectation even more.

In other words, the norm of $\mathtt{L}_{\SM}$ can become arbitrarily small as $\Delta$ grows, while for DDSM one can avoid this by choosing $T$ sufficiently large. Similar statements can be made about the curvature constants. Given that both quantities encode the local change and curvature of the loss function, respectively, this confirms the initial observation from Figure~\ref{fig:pdfs_scores}: when the separation between the modes is large, the SM loss is flat, but the DSM loss is not. Let us now state the main results.

\begin{theorem}[Error bound for SME]\label{theorem:error_bound_sme}
    Let $(\Prob_\theta)_{\theta\in\Theta}$ be the family in~\eqref{eq:gm} with separation $2\Delta=\|\mu_1-\mu_2\|_2$. Grant Assumptions~\ref{assumption:probtrue} and~\ref{assumption:compact_parameter_space}, in particular, fix $\eta\in(0,1/2)$ and set $\Theta:=[\eta,1-\eta]$. Then, for any $0<\delta<\mathtt{C}\|\mathtt{L}_{\SM}\|_{L^2(\Probtrue)}/\mathtt{C}_{\SM}$ there exists $n_\delta\in\N$ such that for all $n\geqslant n_\delta$ the SME as defined in Definition~\ref{def:sme} satisfies
       \begin{equation*}
            \Probtrue\left(|\hat {\theta}_{\SM} - \theta_0| \leqslant \widetilde{\mathtt{C}}(\eta)\frac{\Delta^{1/2}\exp(\Delta^2/4)+\rho(\Delta)}{\delta\sqrt{n}}\right)\geqslant 1-2\delta,
    \end{equation*}
    where $\mathtt{C},\widetilde{\mathtt{C}}(\eta)>0$ are constants and $\rho:\R\rightarrow[0,\infty)$.
\end{theorem}

\begin{proof}
	We want to apply Proposition~\ref{proposition:m_estimator_error_bound} and therefore need to verify its assumptions. First, from 1. of Lemma~\ref{lemma:lipschitz_curvature_bounds_sm} we know that $\theta\mapsto m_{\SM}(\theta,x)$ is Lipschitz for every $x\in\samplespace$ with Lipschitz constant $\mathtt{L}_{\SM}(x,\mu)=\sup_{\theta\in\Theta}|\partial_\theta m_{\SM}(\theta,x)|$. Second, the curvature condition follows from 2. of Lemma~\ref{lemma:lipschitz_curvature_bounds_sm} with $\alpha=2$. Lastly, let us check that $\hat {\theta}_{\SM}$ converges in probability to $\theta_0$. Using Theorem~\ref{theorem:consistency} in Appendix~\ref{appendix:consistency} we need to show that $\mathcal{F}_\Theta=\{x\mapsto m_{\SM}(\theta,x):\theta\in\Theta\}$ is Glivenko-Cantelli, that is,
	\begin{equation*}
		\sup_{\theta\in\Theta}\left|\frac{1}{n}\sum_{i=1}^n m_{\SM}(\theta,X_i)-\E_0[m_{\SM}(\theta,X)]\right|\xrightarrow{\Probtrue}0,
	\end{equation*}
	and that a well-separated population maximizer at $\theta_0$ exists, that is,
	\begin{equation*}
		\inf_{|\theta-\theta_0|>\varepsilon}\E_0[m_{\SM}(\theta,X)]>\E_0[m_{\SM}(\theta_0,X)]
	\end{equation*}
	for every $\varepsilon>0$.
	
	With the same logic as in the proof of Lemma~\ref{lemma:entropy_parametric_class} we have
	\begin{equation*}
		N_{[\,]}(2\varepsilon\|\mathtt{L}_{\SM}\|_{L^1(\Probtrue)},\mathcal{F}_\Theta,L^1(\Probtrue))\leqslant N(\varepsilon,\Theta,|\cdot|)\leqslant 1+\frac{2}{\varepsilon}<\infty,
	\end{equation*}
	for every $\varepsilon>0$, where $N_{[\,]}$ and $N$ are the bracketing and covering numbers, respectively. The first inequality follows from Theorem 2.7.17 in~\citet{VaaWel2023} and the second from a slight generalization of Corollary 4.2.11 from~\cite{Ver2026}. With analogous arguments as in the proof of 1. of Lemma~\ref{lemma:lipschitz_curvature_bounds_sm} one obtains $\|\mathtt{L}_{\SM}\|_{L^1(\Probtrue)}<\infty$. Then, Theorem 2.4.1 in~\citet{VaaWel2023} implies that $\mathcal{F}_\Theta$ is Glivenko-Cantelli. The existence of a well-separated population maximizer at $\theta_0$ follows from the fact that $m_{\SM}$ is a proper scoring rule, that is, it holds that
	\begin{equation*}
		\E_0[m_{\SM}(\theta,X)]-\E_0[m_{\SM}(\theta_0,X)]=\FI(\Prob_{\theta_0},\Prob_\theta)> 0,
	\end{equation*}
	if $\theta\neq\theta_0$. Thus, we have shown that $\hat {\theta}_{\SM}$ converges in probability to $\theta_0$. This concludes the verification of the assumptions of Proposition~\ref{proposition:m_estimator_error_bound}.

	What remains to be shown is the explicit form of the error bound given in the theorem. Combining the two bounds from Lemma~\ref{lemma:lipschitz_curvature_bounds_sm} gives the lower bound
	\begin{equation*}
		\frac{\|\mathtt{L}_{\SM}\|_{L^2(\Probtrue)}}{\mathtt{C}_{\SM}}\gtrsim_{\eta}\Delta^{1/2}\exp(\Delta^2/4).
	\end{equation*}
	Clearly, any upper bound of the left-hand side of the above display has to be at least as large as $\Delta^{1/2}\exp(\Delta^2/4)$, say $\Delta^{1/2}\exp(\Delta^2/4)+\rho(\Delta)$ for some map $\rho:\R\rightarrow[0,\infty)$ possibly depending on $\Delta$. With the inclusion
	\begin{equation}
		\label{eq:error_bound_inclusion}
		\left\{|\hat {\theta}_{\SM} - \theta_0| \leqslant \frac{\mathtt{C}\|\mathtt{L}_{\SM}\|_{L^2(\Probtrue)}}{\delta\mathtt{C}_{\SM}\sqrt{n}}\right\}\subseteq\left\{|\hat {\theta}_{\SM} - \theta_0| \leqslant \widetilde{\mathtt{C}}(\eta)\frac{\Delta^{1/2}\exp(\Delta^2/4)+\rho(\Delta)}{\delta\sqrt{n}}\right\},	
	\end{equation}
	for some constant $\widetilde{\mathtt{C}}(\eta)>0$ and the monotonicity of probability measures, the claim follows from Proposition~\ref{proposition:m_estimator_error_bound}.
\end{proof}

\begin{remark}[Relevance of $\rho$]
	The high-level goal of Lemmata~\ref{lemma:lipschitz_curvature_bounds_sm} and~\ref{lemma:lipschitz_curvature_bounds_dsm} is to show that $\|\mathtt{L}_{\DDSM}\|_{L^2(\Probtrue)}/\mathtt{C}_{\DDSM}\ll \|\mathtt{L}_{\SM}\|_{L^2(\Probtrue)}/\mathtt{C}_{\SM}$, by upper bounding the left-hand side and providing a lower bound for the right-hand side. On the other hand, we need to enlarge the event on the left-hand side of~\eqref{eq:error_bound_inclusion} in order for the lower bound of $1-2\delta$ on its probability to still hold. Therefore, the additional term $\rho(\Delta)$ is necessary, but it is not relevant to know its exact form.
\end{remark}

\begin{theorem}[Error bound for DDSME]
	\label{theorem:error_bound_ddsme}
	Let $(\Prob_\theta)_{\theta\in\Theta}$ be the family in~\eqref{eq:gm} with with separation $2\Delta=\|\mu_1-\mu_2\|_2$. Grant Assumptions~\ref{assumption:probtrue} and~\ref{assumption:compact_parameter_space}, in particular, fix $\eta\in(0,1/2)$ and set $\Theta:=[\eta,1-\eta]$. Then, for any $0<\delta<\mathtt{C}\|\mathtt{L}_{\DDSM}\|_{L^2(\Probtrue)}/\mathtt{C}_{\DDSM}$ there exists $n_\delta\in\N$ such that for all $n\geqslant n_\delta$ and any $T>0$ the DDSME as defined in Definition~\ref{def:dsme} satisfies
       \begin{equation*}
            \Probtrue\left(|\hat {\theta}_{\DDSM} - \theta_0| \leqslant \frac{\widetilde{\mathtt{C}}(\eta)}{\delta\sqrt{n}}\frac{\psi(\Delta,T)}{\xi(\Delta,T)}\right)\geqslant 1-2\delta,
    \end{equation*}
	where $\mathtt{C},\widetilde{\mathtt{C}}(\eta)>0$ are constants and $\psi$ and $\xi$ are as defined in Lemma~\ref{lemma:lipschitz_curvature_bounds_dsm}.
\end{theorem}

\begin{proof}
	Again, the statement follows from Proposition~\ref{proposition:m_estimator_error_bound}. Its assumptions can be verified with analogous arguments to the ones used in the proof of Theorem~\ref{theorem:error_bound_sme}, but considering the contrast $m_{\DDSM}$. Then, a combination of the bounds from Lemma~\ref{lemma:lipschitz_curvature_bounds_dsm} yields the upper bound
	\begin{equation*}
		\frac{\|\mathtt{L}_{\DDSM}\|_{L^2(\Probtrue)}}{\mathtt{C}_{\DDSM}}\lesssim_{\eta} \frac{\psi(\Delta,T)}{\xi(\Delta,T)},
	\end{equation*}
	which gives the inclusion
	\begin{equation}
		\label{eq:error_bound_inclusion_ddsm}
		\left\{|\hat {\theta}_{\DDSM} - \theta_0| \leqslant \frac{\mathtt{C}\|\mathtt{L}_{\DDSM}\|_{L^2(\Probtrue)}}{\delta\mathtt{C}_{\SM}\sqrt{n}}\right\}\subseteq\left\{|\hat {\theta}_{\SM} - \theta_0| \leqslant \frac{\widetilde{\mathtt{C}}(\eta)}{\delta\sqrt{n}}\frac{\psi(\Delta,T)}{\xi(\Delta,T)}\right\},
	\end{equation}
	for a constant $\widetilde{\mathtt{C}}(\eta)>0$. The claim follows with Proposition~\ref{proposition:m_estimator_error_bound} and the monotonicity of measures.
\end{proof}

For $\Delta=0$ the error bounds for both the SME and DDSME state that the true parameter $\theta_0$ can be recovered exactly with probability one. This simply reflects the fact that the mixing probability becomes non-identifiable if $\Delta\downarrow0$, since $\Prob_\theta=\mathcal{N}(\mu,\mathbb{I}_d)$ if $\Delta=0$. For distributions with a minimal separation between the means of the two components, the following corollary of Theorem~\ref{theorem:error_bound_ddsme} shows that the DDSME can still recover the true parameter with high precision and high probability. 

\begin{corollary}[Error bound for DDSME with lower bound on separation]
	Consider the setting of Theorem~\ref{theorem:error_bound_ddsme} and further assume that $\Delta\geqslant 1+\nu$ for some $\nu>0$ and $T\geqslant\log(\Delta^2)$. Then, it holds that
	\begin{equation*}
		\Probtrue\left(|\hat {\theta}_{\DDSM} - \theta_0| \leqslant \frac{\overline{\mathtt{C}}(\eta,\nu)}{\delta\sqrt{n}}\right)\geqslant 1-2\delta,
	\end{equation*}
	where $\overline{\mathtt{C}}(\eta,\nu)>0$ is a constant depending only on $\eta$ and $\nu$.
\end{corollary}

\begin{proof}
	If $\Delta\geqslant 1+\nu$ for $\nu>0$ and $T\geqslant\log(\Delta^2)$ holds, then $T\geqslant 0\vee\log\Delta$ is satisfied and from 1. of Lemma~\ref{lemma:lipschitz_curvature_bounds_dsm} we get $\psi(\Delta,T)\leqslant \mathtt{C}(\eta)$. Also, if $\Delta\geqslant 1+\nu$ and $T\geqslant\log(\Delta^2)$, then 2. of Lemma~\ref{lemma:lipschitz_curvature_bounds_dsm} gives
	\begin{equation*}
		\xi(\Delta,T)\gtrsim 1-\frac{1}{(1+\nu)^2}=:\mathtt{C}(\nu).
	\end{equation*}
	Combining these bounds gives
	\begin{equation*}
		\frac{\widetilde{\mathtt{C}}(\eta)\psi(\Delta,T)}{\xi(\Delta,T)}\lesssim \overline{\mathtt{C}}(\eta,\nu),
	\end{equation*}
	with $\overline{\mathtt{C}}(\eta,\nu):=(\widetilde{\mathtt{C}}(\eta)\mathtt{C}(\eta))/\mathtt{C}(\nu)$. With the same arguments as in the proofs of Theorems~\ref{theorem:error_bound_sme} and~\ref{theorem:error_bound_ddsme}, the claim follows.
\end{proof}

\begin{remark}[Assumptions on $T$]
	Conditions similar to $T\geqslant\log(\Delta^2)$ can be found in the literature on diffusion models. For instance,~\citet{Debortoli2023} obtains a guarantee wrt $\text{\normalfont W}_1$ by requiring $T \geqslant 2\beta(1 + \log(1 + \text{\normalfont diam}(\mathcal{M}))$, where $\mathcal{M}$ is a compact subset of $\R^d$ that contains the support of the data distribution $\Probtrue$. Similarly,~\citet{ConDurGen2025} shows a guarantee wrt $\text{\normalfont KL}$ under the assumption $T = (1/2) \log((\mathtt{M}_2 + d)/\varepsilon^2 )$, where $\mathtt{M}_2:=\|X\|_{L^2(\Probtrue)}$. Lastly,~\citet{BreMij2025} require $T=\log R$, where $\Probtrue(B_R(0))>1-\varepsilon$ for some small $\varepsilon>0$. These conditions show, that a more spread out distribution (i.e., a larger $\Delta$, $\text{\normalfont diam}(\mathcal{M})$ or $\mathtt{M}_2$) requires a longer mixing time $T$ in order to be able to learn the data distribution.
\end{remark}

The statistical guarantees derived in Theorem~\ref{theorem:error_bound_sme} and~\ref{theorem:error_bound_ddsme} behave very differently for large values of $\Delta$. More precisely, they imply
\begin{equation*}
	\frac{\|\mathtt{L}_{\DDSM}(X_0)\|_{L^2(\Probtrue)}}{\mathtt{C}_{\DDSM}}\leqslant \mathtt{C}(\eta,\nu)\ll\Delta^{1/2}\exp\left(\frac{\Delta^2}{4}\right)\leqslant\frac{\|\mathtt{L}_{\SM}(X_0)\|_{L^2(\Probtrue)}}{\mathtt{C}_{\SM}}
\end{equation*}
for large values of $\Delta$ and $T\geqslant \log(\Delta^2)$. In the above, $\mathtt{C}(\eta,\nu)$ is a constant proportional to $\overline{\mathtt{C}}(\eta,\nu)$ from Theorem~\ref{theorem:error_bound_ddsme}. In other words, the precision of the error bound for the DDSME is unaffected by well-separated modes, while the error bound for the SME becomes vacuous for large values of $\Delta$.

\section{Numerical results}

\subsection{A Gaussian mixture model}

We revisit the Gaussian mixture model~\ref{eq:gm} that we studied in Section~\ref{sec:main_results}. More precisely, we focus on the simplified case where the data distribution is given by
\begin{equation*}
    \Prob_\theta=\theta\mathcal{N}(\mu,1)+(1-\theta)\mathcal{N}(-\mu,1),
\end{equation*}
where $\theta\in(0,1)$ and the mean parameter $\mu>0$ is treated as known.

We compare three estimators: the MLE, SME and DDSME. For each parameter configuration, we generate $N=20$ independent datasets of size $n=100$. The true mixing probabilities are $\theta \in \{0.1,0.3,0.5\}$ and the mean parameters are $ \mu \in \{1,2,3,4,5,6\}$. The performance is evaluated across the $N=20$ replications using the bias, variance, and mean squared error (MSE) of the estimates. Table~\ref{tab:simulation-setup} summarizes the simulation design.

\begin{table}[htb]
\centering
\begin{tabular}{ll}
\toprule
Quantity & Value \\
\midrule
Mean parameters & $\mu\in\{1,2,3,4,5,6\}$ \\
Mixing weights & $\theta\in\{0.1,0.3,0.5\}$ \\
Replications per setting $(\mu,\theta)$ & \(N=20\) \\
Sample size per dataset & $n=100$ \\
Performance measures & Bias, variance, MSE \\
DDSME terminal time & $T=\log(\mu^2)\vee 1$ \\
DDSME Monte Carlo setup & 50 time points, 16 Gaussian samples \\
\bottomrule
\end{tabular}
\caption{Simulation design for the Gaussian mixture model~\ref{eq:gm}.}
\label{tab:simulation-setup}
\end{table}

Figure~\ref{fig:simulation_gm} shows boxplots of the estimates for all possible combinations of $(\mu,\theta)$ and for the three estimators. The boxplots also show the jittered estimates for a more detailed visualization of the estimates across replications. The boxplots clearly reveal that the SME deteriorates in performance as $\mu$ increases and as $\theta\rightarrow0.5$, which results in a smaller isoperimetric constant, as is known from Proposition~\ref{proposition:isoperimetric_constant}. The DDSME and MLE on the other hand remain stable. This is in line with the theoretical results from Section~\ref{sec:main_results}, which show that the error bound for the SME becomes vacuous for large values of $\mu$, while the error bound for the DDSME remains unaffected by well-separated modes.

\begin{figure}[htb]
	\centering
	\includegraphics[width=\textwidth]{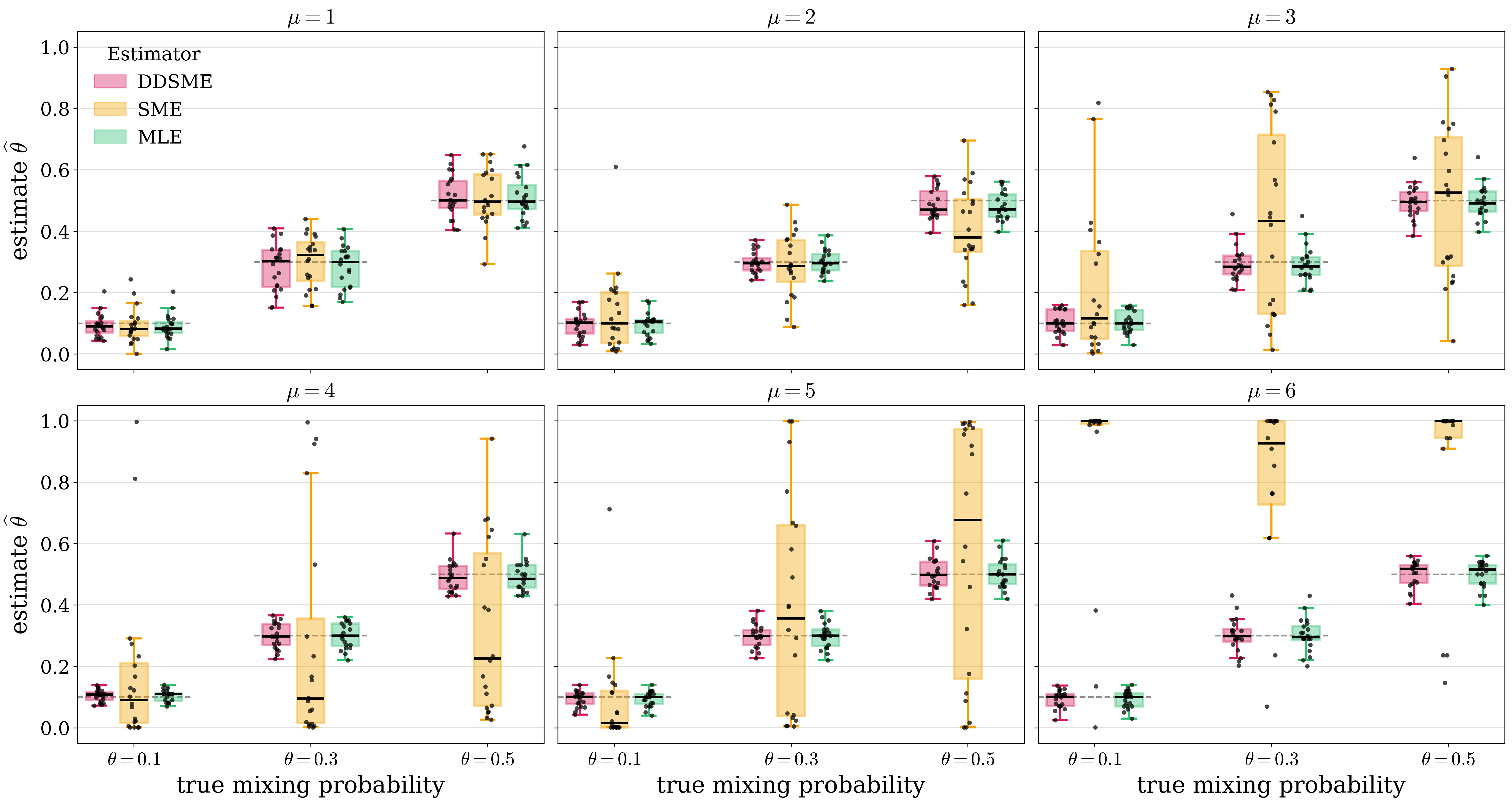}
	\caption{Boxplots of the estimates for the Gaussian mixture model. From the top left to the bottom right, the nuisance parameter $\mu$ increases from 1 to 6. Within each plot, the mixing weight $\theta$ increases from 0.1 to 0.5 from left to right. The jittered points show the estimates across the $N=20$ replications for each setting. Clearly, the SME deteriorates in performance as $\mu$ increases and as $\theta\rightarrow0.5$, while the DDSME and MLE remain stable. The bias of the SME for $\mu=6$ is due to numerical instability of the SM loss.}
	\label{fig:simulation_gm}
\end{figure}

For larger values of $\mu$, one might get the impression from the boxplots that the SME not only has a larger variance, but also a substantial bias. While this can certainly be the case for a finite sample size, it is important to note that the bias of the SME is expected to vanish as $n\to\infty$ for all values of $\mu$ and $\theta$, since the SME is consistent. In other words, the bias observed in the boxplots for large values of $\mu$ is a finite-sample phenomenon, which is expected to disappear as the sample size increases. Furthermore, the degeneracies that can be observed in the case $\mu=6$ are due to numerical instability of the SM loss.

Table~\ref{tab:performance-all-cases-highlighted} contains the bias, variance and mean squared error of the estimates for all combinations of $(\mu,\theta)$ and for all three estimators. The best and worst values across the three estimators are highlighted in green and red, respectively. Again, the table confirms that the SME performs worse than both the DDSME and MLE in almost all instances, but substantially worse for large values of $\mu$. The performance of the DDSME and MLE is comparable across all parameter configurations, with a slight advantage for the DDSME in terms of bias.

\newcommand{\best}[1]{\cellcolor{green!18}{#1}}
\newcommand{\worst}[1]{\cellcolor{red!18}{#1}}
\begin{table}[htb]
\centering
\scriptsize
\setlength{\tabcolsep}{3.2pt}
\begin{tabular}{cc rrr rrr rrr}
\toprule
& &
\multicolumn{3}{c}{DDSME} &
\multicolumn{3}{c}{SME} &
\multicolumn{3}{c}{MLE} \\
\cmidrule(lr){3-5}
\cmidrule(lr){6-8}
\cmidrule(lr){9-11}
$\theta$ & $\mu$
& Bias & Var. & MSE
& Bias & Var. & MSE
& Bias & Var. & MSE \\
\midrule
$0.1$ & 1
& \best{$-0.0061$} & \best{$0.0015$} & \best{$0.0015$}
& $-0.0081$ & \worst{$0.0033$} & \worst{$0.0033$}
& \worst{$-0.0097$} & $0.0016$ & $0.0017$ \\

$0.1$ & 2
& $-0.0038$ & $0.0016$ & $0.0016$
& \worst{$0.0361$} & \worst{$0.0189$} & \worst{$0.0202$}
& \best{$-0.0037$} & \best{$0.0014$} & \best{$0.0014$} \\

$0.1$ & 3
& \best{$0.0030$} & $0.0014$ & \best{$0.0014$}
& \worst{$0.1172$} & \worst{$0.0575$} & \worst{$0.0713$}
& $0.0043$ & \best{$0.0013$} & \best{$0.0014$} \\

$0.1$ & 4
& $0.0045$ & \best{$0.0004$} & \best{$0.0004$}
& \worst{$0.0777$} & \worst{$0.0714$} & \worst{$0.0775$}
& \best{$0.0041$} & \best{$0.0004$} & \best{$0.0004$} \\

$0.1$ & 5
& $-0.0063$ & $0.0007$ & \best{$0.0007$}
& \worst{$-0.0117$} & \worst{$0.0265$} & \worst{$0.0266$}
& \best{$-0.0060$} & \best{$0.0006$} & \best{$0.0007$} \\

$0.1$ & 6
& $-0.0071$ & \best{$0.0008$} & \best{$0.0009$}
& \worst{$0.7718$} & \worst{$0.0947$} & \worst{$0.6904$}
& \best{$-0.0065$} & \best{$0.0008$} & \best{$0.0009$} \\

\midrule
$0.3$ & 1
& \worst{$-0.0183$} & $0.0061$ & $0.0065$
& \best{$0.0018$} & \worst{$0.0074$} & \worst{$0.0074$}
& $-0.0155$ & \best{$0.0047$} & \best{$0.0049$} \\

$0.3$ & 2
& \best{$-0.0015$} & \best{$0.0013$} & \best{$0.0013$}
& \worst{$-0.0094$} & \worst{$0.0111$} & \worst{$0.0112$}
& $0.0018$ & $0.0015$ & $0.0015$ \\

$0.3$ & 3
& \best{$-0.0056$} & \best{$0.0035$} & \best{$0.0035$}
& \worst{$0.1238$} & \worst{$0.0909$} & \worst{$0.1062$}
& $-0.0086$ & $0.0037$ & $0.0038$ \\

$0.3$ & 4
& \best{$-0.0004$} & \best{$0.0017$} & \best{$0.0017$}
& \worst{$-0.0241$} & \worst{$0.1266$} & \worst{$0.1271$}
& $-0.0005$ & $0.0018$ & $0.0018$ \\

$0.3$ & 5
& \best{$-0.0005$} & \best{$0.0016$} & \best{$0.0016$}
& \worst{$0.0945$} & \worst{$0.1239$} & \worst{$0.1329$}
& $-0.0010$ & $0.0017$ & $0.0017$ \\

$0.3$ & 6
& \best{$0.0023$} & \best{$0.0031$} & \best{$0.0031$}
& \worst{$0.5071$} & \worst{$0.0707$} & \worst{$0.3278$}
& $0.0025$ & \best{$0.0031$} & \best{$0.0031$} \\

\midrule
$0.5$ & 1
& \worst{$0.0147$} & \best{$0.0049$} & \best{$0.0051$}
& \best{$0.0079$} & \worst{$0.0086$} & \worst{$0.0086$}
& $0.0124$ & $0.0052$ & $0.0053$ \\

$0.5$ & 2
& \best{$-0.0119$} & $0.0025$ & $0.0026$
& \worst{$-0.0958$} & \worst{$0.0217$} & \worst{$0.0309$}
& $-0.0178$ & \best{$0.0022$} & \best{$0.0025$} \\

$0.5$ & 3
& \best{$-0.0044$} & $0.0031$ & $0.0031$
& \worst{$-0.0079$} & \worst{$0.0654$} & \worst{$0.0654$}
& $-0.0046$ & \best{$0.0030$} & \best{$0.0030$} \\

$0.5$ & 4
& \best{$-0.0084$} & \best{$0.0025$} & \best{$0.0026$}
& \worst{$-0.1706$} & \worst{$0.0792$} & \worst{$0.1084$}
& $-0.0085$ & \best{$0.0025$} & \best{$0.0026$} \\

$0.5$ & 5
& \best{$0.0019$} & $0.0025$ & $0.0025$
& \worst{$0.0876$} & \worst{$0.1614$} & \worst{$0.1691$}
& $0.0020$ & \best{$0.0024$} & \best{$0.0024$} \\

$0.5$ & 6
& $0.0011$ & \best{$0.0018$} & \best{$0.0018$}
& \worst{$0.3694$} & \worst{$0.0825$} & \worst{$0.2190$}
& \best{$-0.0010$} & $0.0019$ & $0.0019$ \\
\bottomrule
\end{tabular}
\caption{
Empirical bias, variance, and MSE across the Monte Carlo replications for the DDSME, SME, and MLE for each parameter configuration $(\theta,\mu)$. Green cells mark the best value among the three estimators in a given row and metric; red cells mark the worst value. For bias, the comparison is based on absolute bias. The numerical results confirm the visual impression from Figure~\ref{fig:simulation_gm} that the SME deteriorates in performance as $\mu$ increases and as $\theta\rightarrow0.5$, while the DDSME and MLE remain stable.}
\label{tab:performance-all-cases-highlighted}
\end{table}

\subsection{A tilted double-well energy-based model in $\R^2$}

We next consider a tilted double-well energy-based model. The so-called double-well potential is a common example of a non-convex energy landscape with multiple modes, which is often used in physics and machine learning to illustrate the challenges of sampling and optimization in multimodal distributions. The one-dimensional double-well potential is given by
\begin{equation*}
U_{\omega,\mu,\vartheta}(x)=\frac{\omega^2}{8\mu^2}(x^2-\mu^2)^2+\vartheta x,
\end{equation*}
and its graph is shown in Figure~\ref{fig:double_well}.

Even in this simple symmetric case, computing the normalizing constant
\begin{equation*}
\mathtt{C}(\omega,\mu,\vartheta)=\int_{\R}\exp\{-U_{\omega,\mu,\vartheta}(x)\}\dif x
\end{equation*}
is nontrivial and it can only be expressed in terms of modified Bessel functions~\citep[cf.~formula 3.469.1 in][]{GraRyz2007}. Extending the double-well potential to two or more dimensions, adding a tilt parameter and interaction term, makes the normalizing constant generally intractable, that is, no useful elementary closed-form expression is available in general. Therefore, the tilted double-well model is a suitable candidate for comparing the performance of the MLE, SME, and DDSME in a nontrivial setting with an intractable normalizing constant.

\begin{figure}[htb]
	\centering
\begin{subfigure}{0.45\textwidth}{\begin{tikzpicture}[baseline=1ex, scale=1.5]
\draw[->] (-1.8,0) -- (1.8,0) node[right] {\small $x$};
\draw[->] (0,-0.2) -- (0,1.4);
\draw[thick, blue, domain=-1.5:1.5, samples=80] plot (\x, {0.125*((\x)^2-1)^2*8});
\fill[red] (-1,0) circle (1pt);
\fill[red] (1,0) circle (1pt);
\fill[orange] (0,1) circle (1pt);
\node[above right, scale=0.7] at (0,1) {$\frac{1}{8}\mu^2\omega^2$};
\node[below, scale=0.7] at (-1,0) {$-\mu$};
\node[below, scale=0.7] at (1,0) {$\mu$};
\end{tikzpicture}}
\caption{$\vartheta=0$}
\end{subfigure}
\begin{subfigure}{0.45\textwidth}{\begin{tikzpicture}[baseline=1ex, scale=1.5]
\draw[->] (-1.8,0) -- (1.8,0) node[right] {\small $x$};
\draw[->] (0,-0.2) -- (0,1.4);
\draw[thick, blue, domain=-1.5:1.5, samples=80] plot (\x, {0.125*((\x)^2-1)^2*8+0.3*(\x)});
\fill[red] (-1,0) circle (1pt);
\fill[red] (1,0) circle (1pt);
\fill[orange] (0,1) circle (1pt);
\node[above right, scale=0.7] at (0,1) {$\frac{1}{8}\mu^2\omega^2$};
\node[below, scale=0.7] at (-1,0) {$-\mu$};
\node[below, scale=0.7] at (1,0) {$\mu$};
\end{tikzpicture}}
	\caption{$\vartheta=0.3$}
\end{subfigure}
\caption{Graphs of the one-dimensional \textbf{(a)} symmetric and \textbf{(b)} tilted double-well potentials.}
\label{fig:double_well}
\end{figure}

\begin{figure}[tb]
	\centering
    \begin{subfigure}{0.32\textwidth}
	\includegraphics[width=\textwidth]{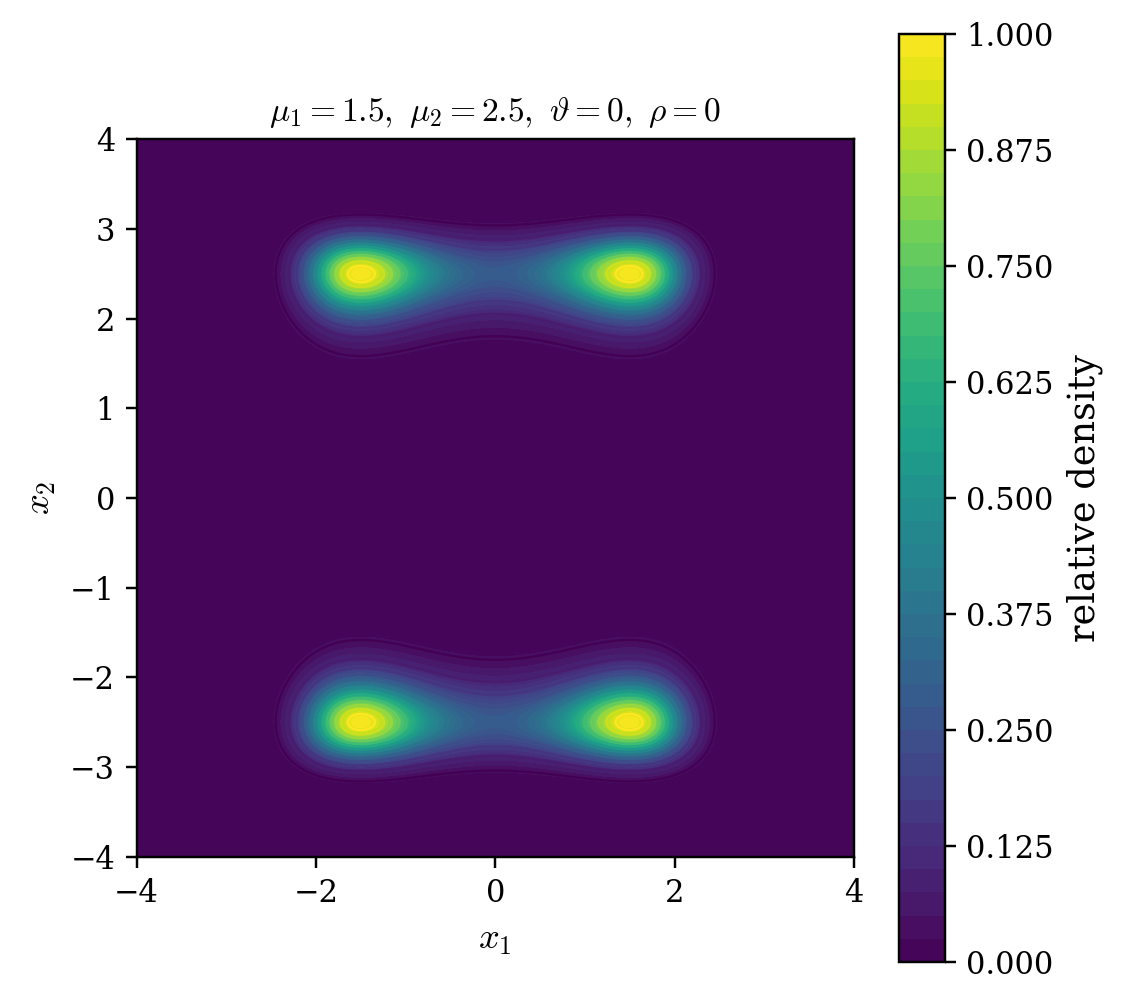}
	\caption{}
	\vspace{2mm}
	\label{fig:contour_1}
	\end{subfigure}
    \begin{subfigure}{0.32\textwidth}
	\includegraphics[width=\textwidth]{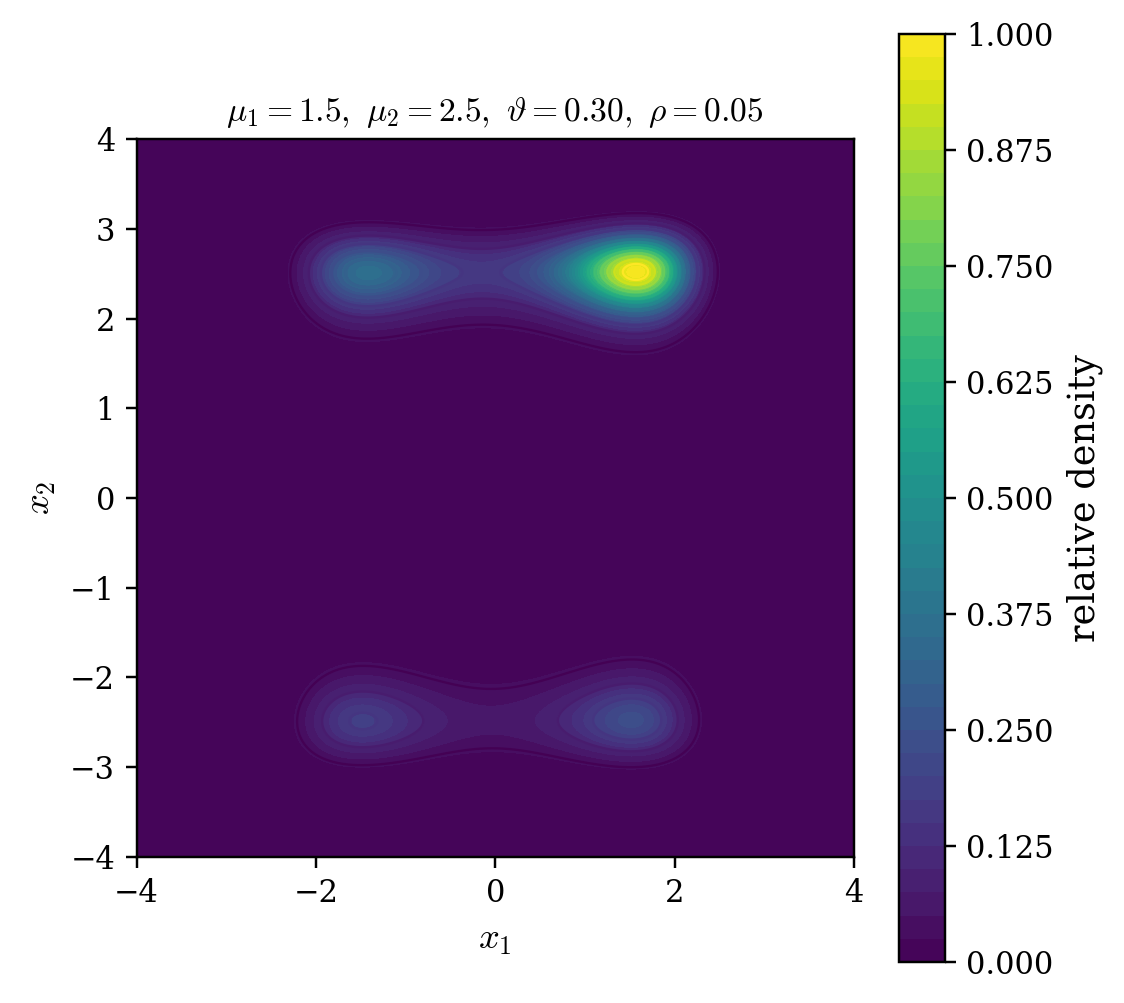}
	\caption{}
	\vspace{2mm}
	\label{fig:contour_2}
	\end{subfigure}
	 \begin{subfigure}{0.32\textwidth}
	\includegraphics[width=\textwidth]{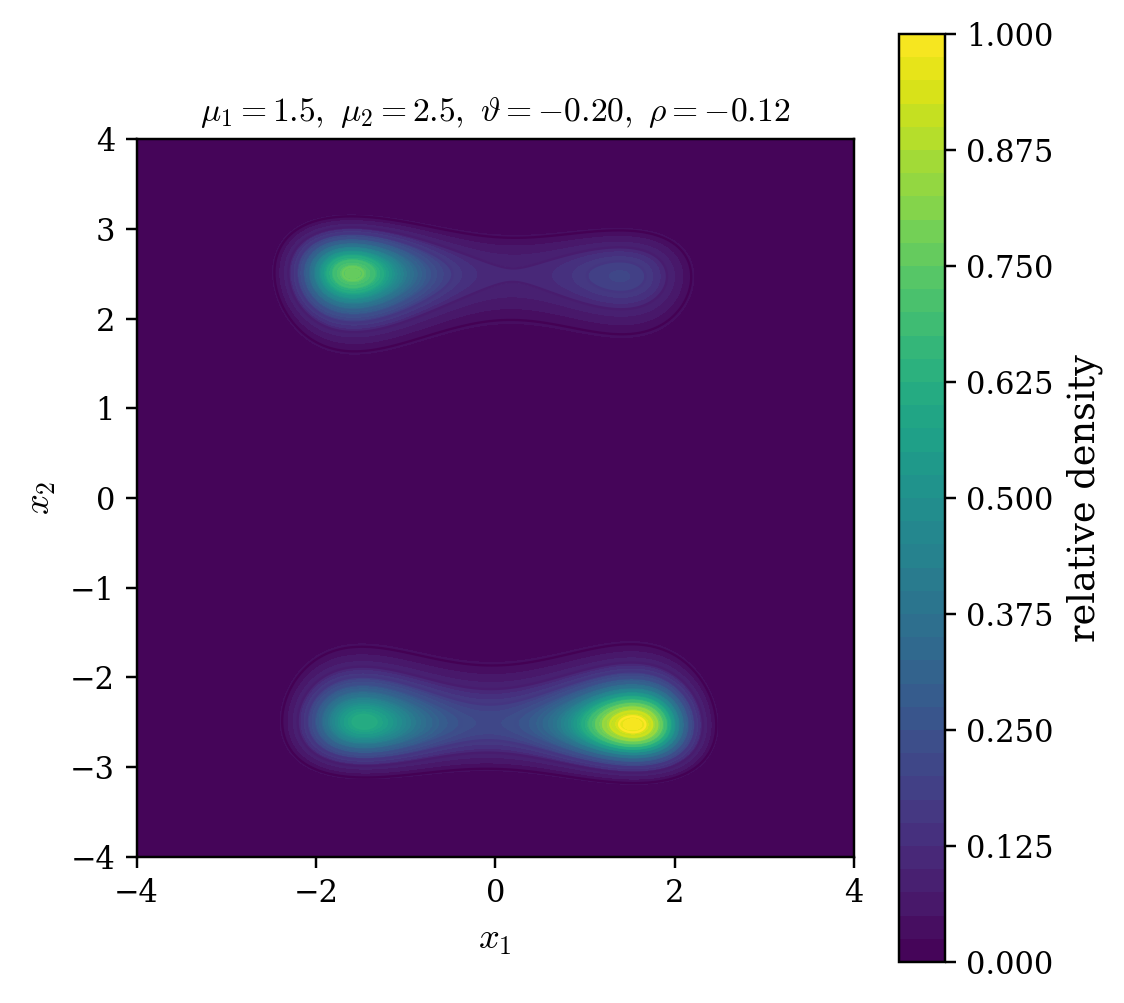}
	\caption{}
	\vspace{2mm}
	\label{fig:contour_3}
	\end{subfigure}
	\caption{All three configurations have in common that $\mu_1=1.5$, $\mu_2=2.5$ and $v=(1/\sqrt{2})[1,1]^\top$. They differ as follows. \textbf{(a)} Symmetric case: $\vartheta=0$, $\rho=0$. \textbf{(b)} Positive tilt and same-sign interaction: $\vartheta=0.3$, $\rho=0.05$. \textbf{(c)} Negative tilt and opposite-sign interaction: $\vartheta=-0.2$, $\rho=-0.12$.}
	\label{fig:contourplots}
\end{figure}

For $x=(x_1,x_2) \in \mathbb{R}^2$, define the density
\begin{equation*}
f_{\theta}(x)=\frac{1}{\mathtt{C}(\theta)}\exp\left\{-U_{\theta}(x)\right\},
\end{equation*}
where $\theta=(\mu_1,\mu_2,\vartheta,\rho)\in\Theta\subseteq(0,\infty)^2\times\R\times\R^2$ is the parameter vector, $\mathtt{C}(\theta)$ is the normalizing constant, and the energy potential is given by
\begin{equation*}
U_{\theta}(x)=\frac{1}{4}\sum_{i=1}^2 \left(x_i^2-\mu_i^2\right)^2-\vartheta v^\top x-\rho x_1x_2\,.
\end{equation*}
Here, $\mu_1,\mu_2>0$ control the locations of the wells (modes) in the two coordinate
directions, $\rho\in\R$ controls the interaction between the coordinates, and $v\in\mathbb{R}^2$ is a fixed direction along which the distribution is tilted.

\begin{table}[htb]
\centering
\begin{tabular}{ll}
\toprule
Quantity & Value \\
\midrule
Mode-location parameters&$\mu_1 = 1.5,\quad \mu_2 \in \{1.5,2.5,3.5,4.5\}$ \\
Tilt parameters&$\vartheta \in \{-0.5,0,0.5\}$ \\
Interaction parameter&$\rho = 0.05$ \\
Tilt direction&$v=(1,1)/\sqrt{2}$ \\
Replications per setting $(\mu_2,\vartheta)$&$N=20$ \\
Sample size per dataset&$n=100$ \\
Performance measures&Bias, variance, MSE, computation time \\
DDSME terminal time&$T=\log(\mu_2^2)\vee 1$ \\
DDSME Monte Carlo setup&40 time points, 8 Gaussian samples \\
MLE approximation&Two-dimensional grid quadrature \\
\bottomrule
\end{tabular}
\caption{Simulation design for the interacting anisotropic double-well model.}
\label{tab:simulation-design-double-well}
\end{table}

The parameter of interest is the scalar tilt parameter $\vartheta\in\R$. The model is multimodal: for small $\vartheta$ and $\rho$, the density has four modes located approximately near $(\pm \mu_1,\pm \mu_2)$. Choosing $\mu_1\neq \mu_2$ makes the distances between the modes unequal, while nonzero values of $\vartheta$ and $\rho$ introduce asymmetry in the relative heights of the modes. This provides a simple but nontrivial unnormalized target distribution for comparing MLE, SME, and DDSME.

\begin{figure}[tb]
	\centering
	\includegraphics[width=\textwidth]{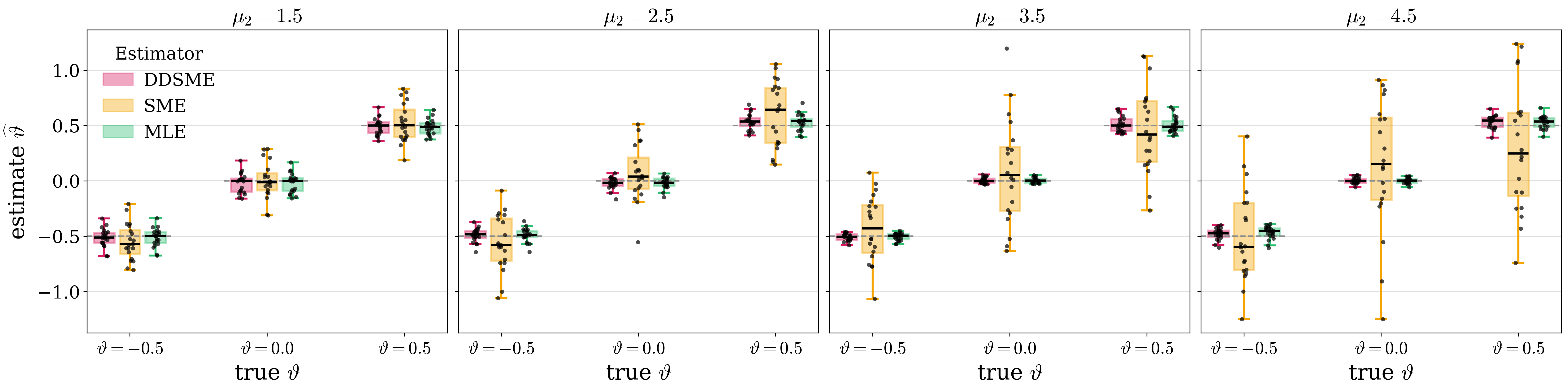}
	\caption{Boxplots of the estimates for the tilted double-well model. From left to right, the nuisance parameter $\mu_2$ increases from 1.5 to 4.5. Within each plot, the tilt parameter $\vartheta$ varies from -0.5 to 0.5 from left to right. The jittered points show the estimates across the $N=20$ replications for each setting. Clearly, the SME deteriorates in performance as $\mu_2$ increases and as $\theta\rightarrow0.5$, while the DDSME and MLE remain stable.}
	\label{fig:simulation_doublewell}
\end{figure}

Figure~\ref{fig:contourplots} shows contour plots of the density for three different parameter configurations. In all three cases, we set $\mu_1=1.5$, $\mu_2=2.5$ and $v=(1/\sqrt{2})[1,1]^\top$. In the first configuration, we then set $\vartheta=0$ and $\rho=0$, which results in a symmetric distribution with four modes of equal height. In the second configuration, we set $\vartheta=0.3$ and $\rho=0.05$, which results in an asymmetric distribution with four modes of unequal height. In the third configuration, we set $\vartheta=-0.2$ and $\rho=-0.12$, which results in an asymmetric distribution with four modes of unequal height, but with a different pattern of asymmetry than in the second configuration.

\begin{table}[htb]
\centering
\scriptsize
\setlength{\tabcolsep}{3.0pt}
\begin{tabular}{cc rrrr rrrr rrrr}
\toprule
& &
\multicolumn{4}{c}{DDSME} &
\multicolumn{4}{c}{SME} &
\multicolumn{4}{c}{MLE} \\
\cmidrule(lr){3-6}
\cmidrule(lr){7-10}
\cmidrule(lr){11-14}
$\vartheta$ & $\mu_2$
& Bias & Var. & MSE & Time
& Bias & Var. & MSE & Time
& Bias & Var. & MSE & Time \\
\midrule
$-0.5$ & $1.5$ & $-0.0153$ & $0.0072$ & $0.0074$ & \cellcolor{red!18}$0.208$ & \cellcolor{red!18}$-0.0507$ & \cellcolor{red!18}$0.0271$ & \cellcolor{red!18}$0.0297$ & \cellcolor{green!18}$0.000$ & \cellcolor{green!18}$-0.0126$ & \cellcolor{green!18}$0.0069$ & \cellcolor{green!18}$0.0071$ & $0.013$ \\
$-0.5$ & $2.5$ & \cellcolor{green!18}$0.0097$ & \cellcolor{green!18}$0.0039$ & \cellcolor{green!18}$0.0040$ & \cellcolor{red!18}$0.208$ & \cellcolor{red!18}$-0.0539$ & \cellcolor{red!18}$0.0629$ & \cellcolor{red!18}$0.0658$ & \cellcolor{green!18}$0.000$ & $0.0099$ & $0.0039$ & $0.0040$ & $0.009$ \\
$-0.5$ & $3.5$ & $-0.0086$ & \cellcolor{green!18}$0.0011$ & \cellcolor{green!18}$0.0012$ & \cellcolor{red!18}$0.205$ & \cellcolor{red!18}$0.0683$ & \cellcolor{red!18}$0.0916$ & \cellcolor{red!18}$0.0963$ & \cellcolor{green!18}$0.000$ & \cellcolor{green!18}$-0.0062$ & $0.0012$ & $0.0012$ & $0.009$ \\
$-0.5$ & $4.5$ & $0.0187$ & \cellcolor{green!18}$0.0028$ & \cellcolor{green!18}$0.0032$ & \cellcolor{red!18}$0.204$ & \cellcolor{green!18}$-0.0011$ & \cellcolor{red!18}$0.1741$ & \cellcolor{red!18}$0.1741$ & \cellcolor{green!18}$0.000$ & \cellcolor{red!18}$0.0305$ & $0.0034$ & $0.0043$ & $0.012$ \\
$0.0$ & $1.5$ & $-0.0198$ & $0.0080$ & $0.0084$ & \cellcolor{red!18}$0.205$ & \cellcolor{green!18}$0.0011$ & \cellcolor{red!18}$0.0278$ & \cellcolor{red!18}$0.0278$ & \cellcolor{green!18}$0.000$ & \cellcolor{red!18}$-0.0205$ & \cellcolor{green!18}$0.0078$ & \cellcolor{green!18}$0.0082$ & $0.013$ \\
$0.0$ & $2.5$ & $-0.0202$ & $0.0029$ & $0.0033$ & \cellcolor{red!18}$0.207$ & \cellcolor{red!18}$0.0680$ & \cellcolor{red!18}$0.0634$ & \cellcolor{red!18}$0.0680$ & \cellcolor{green!18}$0.000$ & \cellcolor{green!18}$-0.0195$ & \cellcolor{green!18}$0.0026$ & \cellcolor{green!18}$0.0030$ & $0.008$ \\
$0.0$ & $3.5$ & \cellcolor{green!18}$0.0025$ & $0.0008$ & $0.0008$ & \cellcolor{red!18}$0.203$ & \cellcolor{red!18}$0.0835$ & \cellcolor{red!18}$0.2245$ & \cellcolor{red!18}$0.2314$ & \cellcolor{green!18}$0.000$ & $0.0036$ & \cellcolor{green!18}$0.0006$ & \cellcolor{green!18}$0.0007$ & $0.008$ \\
$0.0$ & $4.5$ & $0.0027$ & $0.0008$ & $0.0009$ & \cellcolor{red!18}$0.205$ & \cellcolor{red!18}$0.1414$ & \cellcolor{red!18}$0.3479$ & \cellcolor{red!18}$0.3679$ & \cellcolor{green!18}$0.000$ & \cellcolor{green!18}$0.0005$ & \cellcolor{green!18}$0.0007$ & \cellcolor{green!18}$0.0007$ & $0.010$ \\
$0.5$ & $1.5$ & \cellcolor{green!18}$-0.0092$ & $0.0059$ & $0.0059$ & \cellcolor{red!18}$0.206$ & \cellcolor{red!18}$0.0252$ & \cellcolor{red!18}$0.0308$ & \cellcolor{red!18}$0.0315$ & \cellcolor{green!18}$0.000$ & $-0.0138$ & \cellcolor{green!18}$0.0051$ & \cellcolor{green!18}$0.0053$ & $0.014$ \\
$0.5$ & $2.5$ & $0.0329$ & \cellcolor{green!18}$0.0055$ & $0.0066$ & \cellcolor{red!18}$0.206$ & \cellcolor{red!18}$0.0977$ & \cellcolor{red!18}$0.0891$ & \cellcolor{red!18}$0.0986$ & \cellcolor{green!18}$0.000$ & \cellcolor{green!18}$0.0297$ & $0.0055$ & \cellcolor{green!18}$0.0064$ & $0.008$ \\
$0.5$ & $3.5$ & $0.0103$ & \cellcolor{green!18}$0.0051$ & $0.0052$ & \cellcolor{red!18}$0.201$ & \cellcolor{red!18}$-0.0397$ & \cellcolor{red!18}$0.1518$ & \cellcolor{red!18}$0.1534$ & \cellcolor{green!18}$0.000$ & \cellcolor{green!18}$0.0056$ & $0.0051$ & \cellcolor{green!18}$0.0051$ & $0.009$ \\
$0.5$ & $4.5$ & $0.0334$ & $0.0036$ & $0.0047$ & \cellcolor{red!18}$0.205$ & \cellcolor{red!18}$-0.2052$ & \cellcolor{red!18}$0.3319$ & \cellcolor{red!18}$0.3740$ & \cellcolor{green!18}$0.000$ & \cellcolor{green!18}$0.0292$ & \cellcolor{green!18}$0.0033$ & \cellcolor{green!18}$0.0042$ & $0.012$ \\
\bottomrule
\end{tabular}
\caption{Empirical bias, variance, MSE and computation time across the Monte Carlo replications for the DDSME, SME, and MLE for each parameter configuration $(\vartheta,\mu_2)$. Green cells mark the best value among the three estimators in a given row and metric; red cells mark the worst value. For bias, the comparison is based on absolute bias. The numerical results confirm the visual impression from Figure~\ref{fig:simulation_doublewell} that the SME deteriorates in performance as $\mu_2$ increases, while the DDSME and MLE remain stable. In terms of computation time, the DDSME is slower than the MLE, because numerical integration does not sufficiently suffer from the curse of dimensionality in $\R^2$.}
\label{tab:double-well-performance-time-highlighted}
\end{table}

The results of the simulation specified in Table~\ref{tab:simulation-design-double-well} are shown in Figure~\ref{fig:simulation_doublewell} and Table~\ref{tab:performance-all-cases-highlighted}. The boxplots in Figure~\ref{fig:simulation_doublewell} and the performance metrics in Table~\ref{tab:double-well-performance-time-highlighted} basically confirm the observations from the Gaussian mixture model. The SME performs worse than both the DDSME and MLE in all instances. The performance of the DDSME and MLE is comparable across all parameter configurations. Note that the implementation of the DDSME is more prone to numerical instability, which can negatively impact its statistical performance. A more sophisticated implementation of the DDSME, for example using a more efficient sampling scheme and a more accurate approximation of the DDSM score, is expected to improve the performance of the DDSME. Lastly, there is a computational gap between the DDSME and the MLE. However, it is reasonable to expect, that the computation of the MLE in higher dimensions becomes prohibitively expensive, while the Monte Carlo approximations for the DDSME are still feasible.

\section{Discussion}\label{sec:discussion}

In this work, we have studied the statistical properties of the SME and DDSME in a setting where the data distribution is a mixture of two isotropic Gaussians in $\R^d$. First, we showed that the poor performance of the SME can be phrased in terms of the isoperimetric constant, which becomes small for well-separated modes. More precisely, by lower bounding the asymptotic variance of the SME with the reciprocal of the isoperimetric constant, we show an exploding variance and by upper bounding the SM loss with the isoperimetric constant, we prove its vanishing curvature. This extends previous research which only showed that the asymptotic variance of the SME for EFs can be lower bounded by the reciprocal of the isoperimetric constant. Second, we have derived explicit error bounds for both estimators that give raise to a theoretical explanation for the flaws of vanilla SM. Specifically, the SME has an error bound that increases proportionally with the exponential of the squared separation between the modes. This drawback can be overcome with the DDSME, where suitable hyperparameter tuning ensures that the error bound does not lose precision as the separation between the modes increases. This confirms our initial hypothesis that diffusion-based denoising score matching can be more effective than standard score matching when the data distribution has well-separated modes. Lastly, we have conducted a simulation study that confirms our theoretical findings in the Gaussian setting, but also in a non-Gaussian scenario. Our empirical results show that the DDSME can outperform the SME in practice. We also verified that the DDSMEs performance is on par with the MLE.


\subsection*{Acknowledgements and Disclosure of Funding}

This research was supported by the Deutsche Forschungsgemeinschaft (DFG, German Research Foundation) through the Emmy Noether grant KL3037/1-1 (second author), the TRR391 (projects A02 and A07), grant number 520388526 (all authors) and the Bundesministerium für Bildung und Forschung (BMBF, Federal Ministry of Education and Research) through GINN 01IS24065B (second authors). The authors thank Claudius Lütke Schwienhorst and Francesco Iafrate for fruitful discussions.

\bibliography{references}

\appendix

\newpage
\section{Isoperimetric constant}
\label{app:isoperimetric_constant}

\begin{definition}[Minkowski content and isoperimetric constant]
	\label{definition:minkowski_content_isoperimetric_constant}
	Let $(\Omega,\mathcal{A},\Prob)$ a probability space. The Minkowski content or \textit{boundary measure} of $\mathbb{P}$ is defined as
	\begin{equation*}
		\mathbb{P}^+(A):=\lim_{\varepsilon\downarrow0}\varepsilon^{-1}\{\mathbb{P}(A_\varepsilon)-\mathbb{P}(A)\}
	\end{equation*}
	for any measurable $A\in\mathcal{A}$ with $A_\varepsilon:=\{x\in\samplespace:\inf_{a\in A}d(x,a)\leqslant\varepsilon\}$ being its $\varepsilon$-neighborhood for $\varepsilon>0$. The isoperimetric constant $\CIP(\Prob)$ of $\Prob$ is defined as
	\begin{equation}
		\label{eq:isoperimetric_constant_def}
		\CIP(\Prob):=\inf_{A\in\mathcal{A}}\frac{\Prob^+(A)}{\Prob(A)\wedge\Prob(A^\complement)}.
	\end{equation}
	For a family of probability measures $\mathcal{F}$ we set $\CIP(\mathcal{F}):=\inf_{\Prob\in\mathcal{F}}\CIP(\Prob)$.
\end{definition}

\section{M-estimation}
\label{app:m_estimation_theory}

We gather relevant results that we need in order to derive our main results in Section~\ref{sec:main_results}. For a comprehensive treatment of M-estimators we refer to~\cite{Vaa1998} or~\cite{VaaWel2023}.

\subsection{Preliminaries and notation}\label{appendix:prelim_notation}

Following common practice in empirical process theory, we use the following operator notation throughout Appendix~\ref{app:m_estimation_theory}: for a function $f:\samplespace\to\R$ and a probability measure $P$ we write $P f$ to denote the integral of $f$ wrt $P$, that is, $\int f(x)\dif P(x)$. In addition, we write $\mathbb{G}_n f$ to denote the empirical process evaluated at $f$, that is, $\mathbb{G}_n f = \sqrt{n}(\Prob_n f - P f)$, where $\Prob_n:=(1/n)\sum_{i=1}^n\delta_{X_i}$ is the empirical measure of the i.i.d sample $X_1,\ldots,X_n$ with $X_i\sim P$. Lastly, we use the shorthand $m_\theta$ to denote the measurable function $x\mapsto m(\theta,x)$.

If we understand the stochastic process $(\Prob_nf)_{f\in\mathcal{F}}$ to be taking values in the metric space $(\ell^{\infty}(\mathcal{F}),d_\infty)$, where
\begin{align*}
	\ell^{\infty}(\mathcal{F})&:=\{g:\mathcal{F}\rightarrow\R\text{ with }\|g\|_\infty<\mathtt{C}\},\\
	d_\infty(g,h)&:=\sup_{f\in\mathcal{F}}|g(f)-h(f)|,
\end{align*}
then the mapping
\begin{align*}
	(\samplespace^n,\mathcal{B}(\samplespace)^{\otimes n})&\rightarrow(\ell^{\infty}(\mathcal{F}),\mathcal{B}(\ell^{\infty}(\mathcal{F})))\\
	(x_1,\ldots,x_n)&\mapsto \left(f\mapsto\frac{1}{n}\sum_{i=1}^nf(x_i)\right)
\end{align*}
need not be measurable. As a counter example one may consider $\mathcal{F}=\{x\mapsto\bm{1}_{[\vartheta,1]}(x)\mid\vartheta\in[0,1]\}$ and $X_i\sim\text{Uniform}([0,1])$. To see how measurability breaks down in this case we refer to~\citet{VaaWel2023},~\citet{Bil1999} or~\citet{Chi1965}. The problem can be avoided by considering outer expectations and outer probabilities. More precisely, we write $\E^{*}[\sup_{f\in\mathcal{F}}|f(X)|]$ to denote the outer expectation of the supremum of $|f(X)|$ over $f\in\mathcal{F}$, that is, $\E^{*}[\sup_{f\in\mathcal{F}}|f(X)|]=\inf\{\E Z:Z\geqslant\sup_{f\in\mathcal{F}}|f(X)|,Z\text{ is a random variable}\}$. Analogously, we write $\Probouter(A)$ for the outer probability of the event $A$.

In addition, we will make use of stochastic big O and little o notation: For a sequence of random variables $(X_n)_{n\in\N}$ and a sequence of strictly positive real random variables $(r_n)_{n\in\N}$, we write $X_n=O_{\Prob}(r_n)$ if for every $\varepsilon>0$ there exist constants $\mathtt{C}_\varepsilon>0$ such that
\begin{equation*}
	\sup_{n\geqslant 1}\Prob\left(\left|\frac{X_n}{r_n}\right| > \mathtt{C}_\varepsilon\right)<\varepsilon.
\end{equation*}
Then, we call $(X_nr_n^{-1})_{n\in\N}$ \textit{bounded in probability} or \textit{uniformly tight}. Similarly, we write $X_n=o_{\Prob}(r_n)$ if $X_nr_n^{-1}\xrightarrow{\Prob} 0$, that is, $(X_nr_n^{-1})_{n\in\N}$ converges to zero in probability.

\begin{lemma}[Equivalence of stochastic boundedness and convergence in probability]
	\label{lemma:equiv_stoch_bound_conv}
	Let $(X_n)_{n\in\N}$ a sequence of real-valued random variables. Then, the following are equivalent:
	\begin{enumerate}
		\item $X_n = O_{\Prob}(r_n)$ for some sequence $(r_n)_{n\in\N}\subseteq(0,\infty)$ with $r_n\rightarrow0$ for $n\rightarrow\infty$.
		\item $X_n = o_{\Prob}(1)$.
	\end{enumerate}
\end{lemma}
\begin{proof}
	$(1.\Rightarrow 2.)$ Let $\varepsilon>0$ and $\delta>0$. Then, since $X_n = O_{\Prob}(r_n)$, there exists $\mathtt{C}_\varepsilon>0$ such that	
	\begin{equation*}
	\sup_{n\geqslant 1}\Prob\left(\left|\frac{X_n}{r_n}\right| > \mathtt{C}_\varepsilon\right)<\varepsilon.
\end{equation*}
And because $r_n\rightarrow0$ for $n\rightarrow\infty$, there exists $N_\varepsilon\in\mathbb{N}$ such that for all $n\geqslant N_\varepsilon$ it holds that $\mathtt{C}_\varepsilon r_n\leqslant\delta$. Therefore, the inclusion
\begin{equation*}
	\{|X_n|>\delta\}\subseteq\{|X_n|> \mathtt{C}_\varepsilon r_n\}
\end{equation*}
holds for all $n\geqslant N_\varepsilon$. With this, we conclude
\begin{equation*}
	\limsup_{n\rightarrow\infty}\Prob\left(\left|X_n\right| > \delta\right)\leqslant \Prob\left(\left|\frac{X_n}{r_n}\right| > \mathtt{C}_\varepsilon\right)<\varepsilon
\end{equation*}
and since both $\delta>0$ and $\varepsilon>0$ were arbitrary, we get $\Prob(|X_n| > \delta)\rightarrow 0$ for any $\delta>0$.

$(2.\Rightarrow 1.)$ From $X_n = o_{\Prob}(1)$ we have for any $k\in\N$ that
\begin{equation*}
	\lim_{n\rightarrow\infty}\Prob(|X_n|>1/k)=0.
\end{equation*}
Thus, we can choose $N_k\in\N$ such that for all $n\geqslant N_k$ it holds that
\begin{equation*}
	\Prob(|X_n|>1/k)<1/k.
\end{equation*}
For the $N_1<N_2<\ldots$ we therefore define the decaying sequence $(r_n)_{n\in\N}$ via $r_n:=(1/k)\bm{1}(N_k\leqslant n <N_{k+1})$. Then, for any $\varepsilon>0$ we can find $k\in\mathbb{N}$ and its corresponding $N_k$ such that $1/k<\varepsilon$ and for $N_k\leqslant n<N_{k+1}$ it holds that
\begin{equation*}
	\Prob\left(\left|\frac{X_n}{r_n}\right| > 1\right)=\Prob(|X_n| > 1/k	)<1/k<\varepsilon.
\end{equation*}
Since $r_n\rightarrow 0$ we can conclude that
\begin{equation*}
	\sup_{n\geqslant N_k}\Prob\left(\left|\frac{X_n}{r_n}\right| > 1\right)<\varepsilon.
\end{equation*}
For all $j\in[N_k-1]$ there exists $M_j>0$ such that
\begin{equation*}
	\Prob\left(\left|\frac{X_j}{r_j}\right| > M_j\right)\leqslant\varepsilon,
\end{equation*}
because the $X_j$ are real-valued. With $\mathtt{M}:=\max\{1,M_1,\ldots,M_{N_k-1}\}$ it then follows that
\begin{equation*}
	\sup_{n\geqslant 1}\Prob\left(\left|\frac{X_n}{r_n}\right| > \mathtt{M}\right)\leqslant\varepsilon,
\end{equation*}
which concludes the proof.
\end{proof}

\subsection{Consistency}\label{appendix:consistency}

We obtain non-asymptotic error bounds for the SME and DDSME by applying Proposition~\ref{proposition:m_estimator_error_bound}, which is itself based on Theorem~\ref{theorem:rate_theorem}. One of the assumptions that needs to be verified is the consistency of the M-estimator, which can be done with the following classical result, which we state without proof.

\begin{theorem}[Consistency, Theorem 5.7 in \citet{Vaa1998}]
\label{theorem:consistency}
Let $M_n$ be random functions and let $M$ be a fixed function of $\theta$ such that
for every $\varepsilon > 0$,
\begin{align*}
	\sup_{\theta \in \Theta}& | M_n(\theta) - M(\theta) |
\xrightarrow{\Probouter} 0,\\	
\sup_{\theta : d(\theta,\theta_0) \geqslant \varepsilon} &M(\theta) < M(\theta_0).
\end{align*}
Then any sequence of estimators $\hat{\theta}_n$ with $M_n(\hat{\theta}_n) \geqslant M_n(\theta_0) - o^*_{P}(1)$
converges in outer probability to $\theta_0$.
\end{theorem}

\subsection{Non-asymptotic error bounds for M-estimators}\label{app:error_bounds_for_m_estimators}

Next, we provide a restatement of Theorem 5.52 by~\citet{Vaa1998} with the additional explicit dependence on the constants $\mathtt{C}_1$ and $\mathtt{C}_2$. Since it is central to our analysis, we include it here for completeness and provide a full proof with additional details that are left out in the proof given by~\citet{Vaa1998}. It is formulated under the assumption that $\hat{\theta}_n:=\argmax_{\theta\in\Theta}(1/n)\sum_{i=1}^n m(\theta,X_i)$ and therefore we have to consider $-m_{\SM}$ and $-m_{\DDSM}$ in order to apply it to $\SM$ and $\DDSM$.

\begin{theorem}[Non-asymptotic error bound for M-estimators]\label{theorem:rate_theorem}
Assume that for fixed constants $\mathtt{C}_1,\mathtt{C}_2$ and $\alpha > \beta$, for every $n\in\N$ and for every sufficiently small $\delta > 0$,
\begin{align}
	\label{eq:rate_theorem_curvature}
  \sup_{\frac{\delta}{2}\leqslant d(\theta,\theta_0) < \delta} P^*(m_\theta - m_{\theta_0})
  &\leqslant - \mathtt{C}_1 \delta^{\alpha},\tag{$\mathtt{curvature}$}\\
  \label{eq:rate_theorem_maximal_inequality}
  \E^{*}\left[ \sup_{d(\theta,\theta_0) < \delta}
  \left| \mathbb{G}_n(m_\theta - m_{\theta_0}) \right| \right]
  &\leqslant \mathtt{C}_2 \delta^{\beta}.\tag{$\mathtt{maxinequal}$}
\end{align}
If the sequence $\hat\theta_n$ satisfies $\Prob_n m_{\hat\theta_n}
  \geqslant \Prob_n m_{\theta_0} - O^*_P(n^{\alpha/(2\beta-2\alpha)})$
and converges in outer probability to $\theta_0$, then $n^{1/(2\alpha-2\beta)} d(\hat\theta_n,\theta_0) = O_P^{*}(1)$. In particular, for every sufficiently small $\varepsilon>0$ and any $K>0$ it holds that
\begin{equation}
	\label{eq:rate_guarantee}
	  \Probouter\left(d(\hat\theta_n,\theta_0) > \frac{2^M}{n^{1/(2\alpha-2\beta)}}\right)
  \leqslant \frac{\mathtt{C}_2}{\mathtt{C}_1}\frac{2^{M(\beta-\alpha)+\alpha+1}}{1-2^{\beta-\alpha}}+\Probouter\left(2d(\hat{\theta}_n,\theta_0)\geqslant\varepsilon\right)+\Probouter\left(r_n^\alpha R_n\geqslant K\right),
\end{equation}
for any $M \in \mathbb{N}$ that is large enough to fulfill $\mathtt{C}_12^{(M-1)\alpha}/2\geqslant K$.
\end{theorem}

\begin{remark}
	Before we begin the proof, we want to make some comments.
	\begin{itemize}
		\item The proof follows essentially the one of Theorem 5.52 in~\citet{Vaa1998}. However, we fill in some details that were left out and most importantly, we provide the concentration inequality in~\eqref{eq:rate_guarantee}, which is then used to derive Proposition~\ref{proposition:m_estimator_error_bound}.
		\item At the beginning of the proof,~\cite{Vaa1998} defines the rate $r_n=n^{1/(2\alpha-2\beta)}$ and assumes that $\hat\theta_n$ maximizes the map $\theta\mapsto\Prob_nm_\theta$ up to a variable $R_n=O_P(r_n^{-\alpha})$. We understand this to mean
	\begin{equation}
		\label{eq:max_up_to_error}
		\Prob_n m_{\hat\theta_n}\geqslant \sup_{\theta\in\Theta}\Prob_n m_{\theta} - R_n,
	\end{equation}
	which is the characterization of \textit{nearly maximizing} given on page 45 of~\citet{Vaa1998}, but with a specified rate. Note that $R_n = O_P(a_n)$ with $a_n\rightarrow0$ is equivalent to $R_n = o_P(1)$ according to Lemma~\ref{lemma:equiv_stoch_bound_conv}. But due to the trivial inequality $\sup_{\theta\in\Theta}\Prob_n m_{\theta}\geqslant\Prob_n m_{\theta_0}$ the characterization in~\eqref{eq:max_up_to_error} is stronger than the assumption
	\begin{equation}
		\label{eq:max_up_to_error_true}
		\Prob_n m_{\hat\theta_n}  \geqslant \Prob_n m_{\theta_0} - R_n,
	\end{equation}
	given in the statement of the theorem. And more surprisingly,~\eqref{eq:max_up_to_error} is not used in the theorem, while~\eqref{eq:max_up_to_error_true} is used to obtain~\eqref{eq:error_on_shell}.
	\end{itemize}
\end{remark}

\begin{proof}[Proof of Theorem~\ref{theorem:rate_theorem}]
	Let $R_n=O_P(r_n^{-\alpha})$ with $r_n=n^{1/(2\alpha-2\beta)}$ and define the \textit{shells}
	\begin{equation*}
		S_{j,n} := \left\{ \theta\in\Theta : 2^{j-1} < r_nd(\theta,\theta_0) \leqslant 2^j \right\}
	\end{equation*}
	for $j \in\mathbb{Z}$ such that
	\begin{equation*}
		\Theta \backslash \{\theta_0\} = \bigcup_{j\in\mathbb{Z}} S_{j,n}
	\end{equation*}
	for every $n\in\mathbb{N}$.	If $r_nd(\hat\theta_n,\theta_0)>2^M$, then there exists $j\geqslant M$ such that $\hat\theta_n\in S_{j,n}$ and on this particular shell $S_{j,n}$ we have
	\begin{equation}
		\label{eq:error_on_shell}
		\sup_{\theta\in S_{j,n}} \left\{\Prob_n m_\theta - \Prob_n m_{\theta_0}\right\} \geqslant \Prob_n m_{\hat\theta_n} - \Prob_n m_{\theta_0}\geqslant - R_n,
	\end{equation}
	according to the assumption on $\hat\theta_n$, which we discussed above. With the definitions
	\begin{align*}
		A&:=\left\{r_nd(\hat\theta_n,\theta_0)>2^M\right\},\\
		B&:=\left\{2d(\hat{\theta}_n,\theta_0)<\varepsilon\right\},\\
		C&:=\left\{r_n^\alpha R_n<K\right\},
	\end{align*}
	and the union bound
	\begin{align*}
		\Prob(A)&\leqslant \Prob\left(A\cap B\cap C\right)+\Prob(A\cap B^c)+\Prob(A\cap C^c)\\
		&\leqslant \Prob\left(A\cap B\cap C\right)+\Prob(B^c)+\Prob(C^c),
	\end{align*}
	we can then conclude that
	\begin{equation*}
		\Probouter\left(r_nd(\hat\theta_n,\theta_0)>2^M\right)\leqslant\Probouter(A\cap B\cap C)+\Probouter\left(2d(\hat{\theta}_n,\theta_0)\geqslant\varepsilon\right)+\Probouter\left(r_n^\alpha R_n\geqslant K\right)
	\end{equation*}
	for every $\varepsilon,K>0$. In order to bound the first probability on the right-hand side, we observe that on the set $A\cap B$ it holds that
	\begin{equation*}
		2^M < r_nd(\hat\theta_n,\theta_0)<\frac{r_n\varepsilon}{2}\leqslant r_n\varepsilon,
	\end{equation*}
	thus there exists $j\geqslant M$ such that $\hat\theta_n\in S_{j,n}$ and $2^j\leqslant r_n\varepsilon$. Therefore, we have
	\begin{align*}
		\Probouter(A\cap B\cap C)&= \Probouter\left(\bigcup_{\substack{j\geqslant M\\2^j\leqslant r_n\varepsilon}}\left\{\hat{\theta}_n\in S_{j,n}\right\}\cap C\right)\\
		&\leqslant\sum_{\substack{j\geqslant M\\2^j\leqslant r_n\varepsilon}}\Probouter\left(\left\{\sup_{\theta\in S_{j,n}} \left\{\Prob_n m_\theta - \Prob_n m_{\theta_0}\right\} \geqslant - R_n\right\}\cap C\right)\\
		&=\sum_{\substack{j\geqslant M\\2^j\leqslant r_n\varepsilon}}\Probouter\left(\sup_{\theta\in S_{j,n}} \left\{\Prob_n m_\theta - \Prob_n m_{\theta_0}\right\} \geqslant - \frac{K}{r_n^\alpha}\right),
	\end{align*}
	where the inequality is based on a union bound together with~\eqref{eq:error_on_shell}. In total, we have obtained
	\begin{align}
		\label{eq:upper_bound}
		\begin{split}
		\Probouter\left(r_nd(\hat\theta_n,\theta_0)>2^M\right)&\leqslant \sum_{\substack{j\geqslant M\\2^j\leqslant r_n\varepsilon}}\Probouter\left(\sup_{\theta\in S_{j,n}} \left\{\Prob_n m_\theta - \Prob_n m_{\theta_0}\right\} \geqslant - \frac{K}{r_n^\alpha}\right)\\
		&\qquad+\Probouter\left(2d(\hat{\theta}_n,\theta_0)\geqslant\varepsilon\right)+\Probouter\left(r_n^\alpha R_n\geqslant K\right).
		\end{split}
	\end{align}
	The consistency of $\hat\theta_n$ implies that the second probability on the right-hand side converges to zero for every $\varepsilon>0$. The third probability can be made arbitrarily small, since $R_n=O_P(r_n^{-\alpha})$ by assumption, that is for every $\varepsilon>0$ there exists $K\geqslant 0$ such that
	\begin{equation}
		\sup_n\Prob\left(r_n^\alpha R_n\geqslant K\right)<\varepsilon.
	\end{equation}

	It remains to bound the first term on the right-hand side of~\eqref{eq:upper_bound}. To that end, choose $\varepsilon>0$ small enough, such that the assumptions of the theorem, that is, the curvature condition in~\eqref{eq:rate_theorem_curvature} and the maximal inequality in~\eqref{eq:rate_theorem_maximal_inequality} hold for every $\delta\leqslant\varepsilon$. Then, for every $j$ such that $2^j/r_n\leqslant \varepsilon$ we have
	\begin{equation*}
		\sup_{\theta\in S_{j,n}} P^*(m_\theta - m_{\theta_0}) \leqslant - \mathtt{C}_1 \frac{2^{(j-1)\alpha}}{r_n^\alpha},
	\end{equation*}
	according to~\eqref{eq:rate_theorem_curvature}. Using this and the equality
	\begin{align*}
		\sqrt{n}\left(\Prob_n m_\theta - \Prob_n m_{\theta_0}\right)&=\sqrt{n}\left(\Prob_n m_\theta-Pm_\theta - (\Prob_n m_{\theta_0}-Pm_{\theta_0})+Pm_\theta-Pm_{\theta_0}\right)\\
		&=\sqrt{n}\left(\mathbb{G}_n m_\theta- \mathbb{G}_n m_{\theta_0}+Pm_\theta-Pm_{\theta_0}\right)\\
		&=\sqrt{n}\left(\mathbb{G}_n(m_\theta - m_{\theta_0})+\sqrt{n}P(m_\theta - m_{\theta_0})\right),
	\end{align*}
	we can continue to bound the first term on the right-hand side of~\eqref{eq:upper_bound} as follows:
	\begin{align*}
		&\Probouter\left(\sup_{\theta\in S_{j,n}} \left\{\Prob_n m_\theta - \Prob_n m_{\theta_0}\right\} \geqslant - \frac{K}{r_n^\alpha}\right)\\
		&\quad=\Probouter\left(\sup_{\theta\in S_{j,n}} \left\{\mathbb{G}_n(m_\theta - m_{\theta_0})+\sqrt{n}P(m_\theta - m_{\theta_0})\right\} \geqslant - \sqrt{n}\frac{K}{r_n^\alpha}\right)\\
		&\quad\leqslant\Probouter\left(\sup_{\theta\in S_{j,n}} \mathbb{G}_n(m_\theta - m_{\theta_0})+ \sup_{\theta\in S_{j,n}} \sqrt{n}P(m_\theta - m_{\theta_0}) \geqslant - \sqrt{n}\frac{K}{r_n^\alpha}\right)\\
		&\quad\leqslant\Probouter\left(\sup_{\theta\in S_{j,n}} \mathbb{G}_n(m_\theta - m_{\theta_0}) \geqslant \sqrt{n} \frac{\mathtt{C}_12^{(j-1)\alpha}-K}{r_n^\alpha}\right)
	\end{align*}
	If we choose $M$ sufficiently large such that $\mathtt{C}_12^{(M-1)\alpha}/2\geqslant K$, then it holds that $\mathtt{C}_12^{(j-1)\alpha}-K\geqslant \mathtt{C}_12^{(j-1)\alpha}/2$ for every $j\geqslant M$ and we can continue to bound the last probability as follows:
	\begin{equation*}
		\Probouter\left(\sup_{\theta\in S_{j,n}} \mathbb{G}_n(m_\theta - m_{\theta_0}) \geqslant \sqrt{n} \frac{\mathtt{C}_12^{(j-1)\alpha}-K}{r_n^\alpha}\right)
		\leqslant\Probouter\left(\sup_{\theta\in S_{j,n}} |\mathbb{G}_n(m_\theta - m_{\theta_0})| \geqslant \sqrt{n} \frac{\mathtt{C}_12^{(j-1)\alpha}}{2r_n^\alpha}\right)
	\end{equation*}
	Using Markov's inequality, we can further bound this probability by
	\begin{align*}
		\Probouter\left(\sup_{\theta\in S_{j,n}} |\mathbb{G}_n(m_\theta - m_{\theta_0})| \geqslant \sqrt{n} \frac{\mathtt{C}_12^{(j-1)\alpha}}{2r_n^\alpha}\right)
		&\leqslant \frac{2r_n^\alpha}{\sqrt{n}\mathtt{C}_12^{(j-1)\alpha}}\E^*\left[\sup_{\theta\in S_{j,n}} \left|\mathbb{G}_n(m_\theta - m_{\theta_0})\right|\right].
	\end{align*}
	By the maximal inequality in~\eqref{eq:rate_theorem_maximal_inequality}, we have
	\begin{equation*}
		\E^*\left[\sup_{d(\theta,\theta_0) < \delta}
		\left| \mathbb{G}_n(m_\theta - m_{\theta_0}) \right| \right]
		\leqslant \mathtt{C}_2 \delta^{\beta}
	\end{equation*}
	for every $\delta\leqslant\varepsilon$. Since for every $\theta\in S_{j,n}$ it holds that $d(\theta,\theta_0)\leqslant 2^j/r_n\leqslant \varepsilon$, we can apply this inequality with $\delta=2^j/r_n$ to obtain	
	\begin{align*}
		\Probouter\left(\sup_{\theta\in S_{j,n}} |\mathbb{G}_n(m_\theta - m_{\theta_0})| \geqslant \sqrt{n} \frac{\mathtt{C}_12^{(j-1)\alpha}}{2r_n^\alpha}\right)
		\leqslant \frac{2r_n^\alpha(2^j/r_n)^\beta}{\sqrt{n}2^{(j-1)\alpha}}\frac{\mathtt{C}_2}{\mathtt{C}_1}= \frac{\mathtt{C}_2}{\mathtt{C}_1}\frac{2r_n^{\alpha-\beta}}{\sqrt{n}}2^{j(\beta-\alpha)+\alpha}.
	\end{align*}
	With $r_n=n^{1/(2\alpha-2\beta)}$ we have $r_n^{\alpha-\beta}/\sqrt{n}=1$ and therefore
	\begin{equation*}
		\sum_{\substack{j\geqslant M\\2^j\leqslant r_n\varepsilon}}\Probouter\left(\sup_{\theta\in S_{j,n}} \left\{\Prob_n m_\theta - \Prob_n m_{\theta_0}\right\} \geqslant - \frac{K}{r_n^\alpha}\right)
		\leqslant \frac{\mathtt{C}_2}{\mathtt{C}_1}2^{\alpha+1}\sum_{j\geqslant M}2^{j(\beta-\alpha)}.
	\end{equation*}
	Since the sum is a geometric series, we can compute
	\begin{equation*}
		\sum_{j\geqslant M}2^{j(\beta-\alpha)}=2^{M(\beta-\alpha)}\sum_{k=0}^\infty 2^{k(\beta-\alpha)}=\frac{2^{M(\beta-\alpha)}}{1-2^{\beta-\alpha}},
	\end{equation*}
	and this gives us the final bound
	\begin{equation}
		\label{eq:final_bound}
		\Probouter\left(d(\hat\theta_n,\theta_0) > \frac{2^M}{n^{1/(2\alpha-2\beta)}}\right)
  \leqslant \frac{\mathtt{C}_2}{\mathtt{C}_1}\frac{2^{M(\beta-\alpha)+\alpha+1}}{1-2^{\beta-\alpha}}+\Probouter\left(2d(\hat{\theta}_n,\theta_0)\geqslant\varepsilon\right)+\Probouter\left(r_n^\alpha R_n\geqslant K\right),
	\end{equation}
	which holds for $\varepsilon>0$ sufficiently small, any $K\geqslant0$ and $M\in\N$ sufficiently large such that $\mathtt{C}_12^{(M-1)\alpha}/2\geqslant K$. This concludes the proof.
\end{proof}

Next, we want to refine Theorem~\ref{theorem:rate_theorem} in order to obtain Proposition~\ref{proposition:m_estimator_error_bound}, on which the error bounds for the SME and DDSME in Theorems~\ref{theorem:error_bound_sme} and~\ref{theorem:error_bound_ddsme} are based. To that end, we need two intermediate results. Lemma~\ref{lemma:entropy_parametric_class} essentially coincides with Example 19.7 from~\citet{Vaa1998} and gives an upper bound for the bracketing number of a function class containing functions that are Lipschitz in its parameter and have a bounded parameter space. We give a short proof, since our version slightly deviates from Example 19.7 in~\citet{Vaa1998}. Lemma~\ref{lemma:entropy_integral} provides a bracketing maximal inequality based on the bracketing integral, which is defined as
\begin{equation*}
	J_{[\,]}(\delta,\mathcal{F},\|\cdot\|):=\int_0^\delta\sqrt{\log N_{[\,]}(\varepsilon\|F\|,\mathcal{F},\|\cdot\|)}\dif\varepsilon,
\end{equation*}
where $F$ is a measurable envelope of $\mathcal{F}$. In the above, $N_{[\,]}(\varepsilon,\mathcal{F},\|\cdot\|)$ denotes the $\varepsilon$-bracketing number of $\mathcal{F}$ wrt the norm $\|\cdot\|$, that is, the minimal number of brackets of size $\varepsilon$ needed to cover $\mathcal{F}$. A bracket $[l,u]$ is defined as the set of functions $f$ such that $l(x)\leqslant f(x)\leqslant u(x)$ for every $x\in\samplespace$ and the size of a bracket is defined as $\|u-l\|$. The bracketing maximal inequality from Lemma~\ref{lemma:entropy_integral} is a bracketing analogue of Dudley's entropy integral inequality~\cite[cf.~Theorem 8.1.3 in][]{Ver2026}.

\begin{lemma}[Entropy of parametric class]
	\label{lemma:entropy_parametric_class}
	Consider the function class $\mathcal{F}=\{f_\theta:\samplespace\rightarrow\R\mid\theta\in\Theta\}$ where $\Theta\subseteq\R^k$ is bounded, that is, it holds that	 $\|\theta\|_2\leqslant R$ for all $\theta\in\Theta$. Furthermore, assume that there exists $\mathtt{L}$ with $\|\mathtt{L}\|_{L^2(\Prob_0)}<\infty$ such that
	\begin{equation*}
		|f_{\theta_1}(x) - f_{\theta_2}(x)|\leqslant \mathtt{L}(x)\|\theta_1-\theta_2\|
	\end{equation*}
	for all $x\in\samplespace$ and $\theta_1,\theta_2\in\Theta$. Then, for every $\varepsilon>0$ it holds that
	\begin{equation*}
		N_{[\,]}(\varepsilon\|\mathtt{L}\|_{L^2(\Prob_0)},\mathcal{F}_\delta,L^2(\Prob_0))\leqslant \left(1+\frac{4R}{\varepsilon}\right)^k.
	\end{equation*}
\end{lemma}

\begin{proof}
	From Theorem 2.7.17 in~\citet{VaaWel2023} it follows that
	\begin{equation*}
		N_{[\,]}(\varepsilon\|\mathtt{L}\|_{L^2(\Prob_0)},\mathcal{F}_\delta,L^2(\Prob_0))\leqslant N(\varepsilon/2,\Theta,\|\cdot\|_2),
	\end{equation*}
	where $N(\varepsilon/2,\Theta,\|\cdot\|_2)$ denotes the $\varepsilon/2$-covering number of $\Theta$ wrt the Euclidean metric on $\R^k$. The $\varepsilon$-covering number of a set in a metric space is the minimal number of balls of radius $\varepsilon$ needed to cover the set. A straightforward extension of Corollary 4.2.11 in~\citet{Ver2026} to balls of radius $R$ gives
	\begin{equation*}
		N(\varepsilon,B(0,R),\|\cdot\|_2)\leqslant \left(1+\frac{2R}{\varepsilon}\right)^k,
	\end{equation*}
	where $B(0,R):=\{x\in\R^k:\|x\|_2\leqslant R\}$ denotes the ball of radius $R$ centered at the origin. Since $\Theta\subseteq B(0,R)$, we have $N(\varepsilon/2,\Theta,\|\cdot\|_2)\leqslant N(\varepsilon/2,B(0,R),\|\cdot\|_2)$, and combining the two inequalities gives the desired result.
\end{proof}

\begin{lemma}[Maximal inequality, Corollary 19.35 in \cite{Vaa1998}]
	\label{lemma:entropy_integral}
	For any class $\mathcal{F}$ of measurable functions with envelope function $F$,
	\begin{equation*}
		\E^{*}\left[ \sup_{f\in\mathcal{F}}
  \left| \mathbb{G}_n f \right| \right]\lesssim J_{[\,]}(\|F\|_{L^2(\Prob_0)},\mathcal{F},L^2(\Prob_0)).
	\end{equation*}	
\end{lemma}

We are now ready to prove Proposition~\ref{proposition:m_estimator_error_bound}. Note, that the measurability problems discussed in Appendix~\ref{appendix:prelim_notation}, which are accounted for in Theorem~\ref{theorem:rate_theorem}, do not arise in the setting of Proposition~\ref{proposition:m_estimator_error_bound}, due to the Lipschitz continuity of the contrast and the separability of the parameter space.\\

\begin{proof}[Proof of Proposition~\ref{proposition:m_estimator_error_bound}]
We will show that the assumptions of Proposition~\ref{proposition:m_estimator_error_bound} imply those of Theorem~\ref{theorem:rate_theorem}. First of all, the curvature condition in Proposition~\ref{proposition:m_estimator_error_bound} is equivalent to~\eqref{eq:rate_theorem_curvature} in Theorem~\ref{theorem:rate_theorem}, taking into account that Proposition~\ref{proposition:m_estimator_error_bound} is formulated for minimizers, while Theorem~\ref{theorem:rate_theorem} is formulated for maximizers. Second, $\Prob_n m_{\hat\theta_n}\leqslant \Prob_n m_{\theta_0}$ holds without any error due to the definition of $\hat{\theta}_n$ as an exact minimizer.

Lastly, we need to show that the maximal inequality~\eqref{eq:rate_theorem_maximal_inequality} in Theorem~\ref{theorem:rate_theorem} holds. In fact, we will show that it can be refined when the contrast $m$ is Lipschitz-continuous in the parameter $\theta$ for every $x\in\samplespace$. Then, the maximal inequality is fulfilled with $\beta=1$ and
\begin{equation*}
	\mathtt{C}_2=\mathtt{C}\|\mathtt{L}	\|_{L^2(\Prob_0)}\sqrt{k},
\end{equation*}
where $\mathtt{C}$ is an absolute constant.

To see this, consider the family of functions
	\begin{equation*}
		\mathcal{F}_\delta:=\{m_\theta - m_{\theta_0}:\theta\in\Theta,||\theta-\theta_0||\leqslant\delta\}.
	\end{equation*}
From the Lipschitz property of $\theta\mapsto m_\theta(x)$ with Lipschitz constant $\mathtt{L}(x)$ it follows that
\begin{equation*}
	|m_{\theta_1}(x) - m_{\theta_0}(x)-(m_{\theta_2}(x) - m_{\theta_0}(x))|=|m_{\theta_1}(x) - m_{\theta_2}(x)|\leqslant \mathtt{L}(x)\|\theta_1-\theta_2\|,
\end{equation*}
that is, the maps in $\mathcal{F}_\delta$ are Lipschitz with Lipschitz constant $\mathtt{L}(x)$ as well. Furthermore, $\mathcal{F}_\delta$ has the envelope function $\delta\mathtt{L}$, since for all $f_\theta\in\mathcal{F}_\delta$ it holds that
\begin{equation*}
	|f_\theta(x)|=|m_{\theta}(x) - m_{\theta_0}(x)|\leqslant\mathtt{L}(x)\|\theta-\theta_0\|\leqslant\delta\mathtt{L}(x).
\end{equation*}
Therefore, the assumptions of Lemma~\ref{lemma:entropy_parametric_class} are satisfied and we can apply it to obtain
	\begin{equation*}
		N_{[\,]}(\varepsilon\|\mathtt{L}\|_{L^2(\Prob_0)},\mathcal{F}_\delta,L^2(\Prob_0))\leqslant \left(1+\frac{4\delta}{\varepsilon}\right)^k,
	\end{equation*}
for every $\varepsilon>0$. The same is true for Lemma~\ref{lemma:entropy_integral}, which gives
\begin{align*}
	\E\left[ \sup_{f\in\mathcal{F}_\delta}
  \left| \mathbb{G}_n f \right| \right]&\lesssim J_{[\,]}(\|\delta\mathtt{L}\|_{L^2(\Prob_0)},\mathcal{F}_\delta,L^2(\Prob_0))\\
  &=\int_0^{\|\delta\mathtt{L}\|_{L^2(\Prob_0)}}\sqrt{\log N_{[\,]}(\varepsilon,\mathcal{F}_\delta,L^2(\Prob_0))}\dif\varepsilon\\
  &=\|\delta\mathtt{L}\|_{L^2(\Prob_0)}\int_0^1\sqrt{\log N_{[\,]}(\varepsilon\|\delta\mathtt{L}\|_{L^2(\Prob_0)},\mathcal{F}_\delta,L^2(\Prob_0))}\dif\varepsilon\\
  &\leqslant \delta\|\mathtt{L}\|_{L^2(\Prob_0)}\int_0^1\sqrt{k\log\left(1+\frac{4}{\varepsilon}\right)}\dif\varepsilon\\
  &\lesssim \sqrt{k}\delta\|\mathtt{L}\|_{L^2(\Prob_0)}.
\end{align*}
With this we have shown that
\begin{equation*}
	\E\left[ \sup_{d(\theta,\theta_0) < \delta}
  | \mathbb{G}_n(m_\theta - m_{\theta_0}) | \right]
  \lesssim\sqrt{k}\delta\|\mathtt{L}\|_{L^2(\Prob_0)},
\end{equation*}
that is, the assumption in~\eqref{eq:rate_theorem_maximal_inequality} of Theorem~\ref{theorem:rate_theorem} is satisfied with $\beta=1$ and $\mathtt{C}_2=\mathtt{C}\|\mathtt{L}\|_{L^2(\Prob_0)}\sqrt{k}$, where $\mathtt{C}>0$ is a constant.

It remains to be shown how the explicit bound on the probability in the statement of Proposition~\ref{proposition:m_estimator_error_bound} can be obtained. To that end, let $\delta>0$ and recall that Theorem~\ref{theorem:rate_theorem} gave us
\begin{equation}
	\label{eq:rate_theorem_result}
		\Probtrue\left(d(\hat\theta_n,\theta_0) > \frac{2^M}{n^{1/(2\alpha-2\beta)}}\right)
  \leqslant \frac{\mathtt{C}_2}{\mathtt{C}_1}\frac{2^{M(\beta-\alpha)+\alpha+1}}{1-2^{\beta-\alpha}}+\Probtrue\left(2d(\hat{\theta}_n,\theta_0)\geqslant\varepsilon\right)+\Probtrue\left(r_n^\alpha R_n\geqslant K\right),
\end{equation}
for any $\varepsilon>0$ sufficiently small, any $K\geqslant0$ and $M\in\N$ sufficiently large such that $\mathtt{C}_12^{(M-1)\alpha}/2\geqslant K$. Since $\hat {\theta}_n$ converges in probability to $\theta_0$ by assumption, there exists $n_\delta\in\N$ such that
\begin{equation*}
	\Probtrue\left(2d(\hat{\theta}_n,\theta_0)\right)\leqslant\delta
\end{equation*}
holds for all $n\geqslant n_\delta$. And from $\hat {\theta}_n=\mathop{\textnormal{argmin}}_{\theta\in \Theta}\frac{1}{n}\sum_{i=1}^n m(\theta,X_i)$ follows $R_n=0$ such that
\begin{equation*}
	\Probtrue\left(r_n^\alpha R_n\geqslant K\right)=0.
\end{equation*}
Lastly, since $\alpha>\beta=1$, we can choose $M$ suffiently large such that
\begin{equation}
	\label{eq:delta_bound}
	\frac{\mathtt{C}_2}{\mathtt{C}_1}\frac{2^{M(1-\alpha)+\alpha+1}}{1-2^{1-\alpha}}=2^{M(1-\alpha)}\frac{K_\alpha\mathtt{C}_2}{\mathtt{C}_1}\leqslant\delta,
\end{equation}
where $K_\alpha=2^{2\alpha}/(2^{\alpha-1}-1)$. With~\eqref{eq:rate_theorem_result} it follows that
\begin{equation*}
	\Prob_0\left(r_nd(\hat\theta_n,\theta_0)>2^M\right)\leqslant2\delta,
\end{equation*}
or alternatively with $r_n=n^{1/(2\alpha-2\beta)}=n^{1/(2\alpha-2)}$ we have
\begin{equation}
	\label{eq:error_bound}
	\Prob_0\left(d(\hat\theta_n,\theta_0)\leqslant\frac{2^M}{n^{1/(2\alpha-2)}}\right)\geqslant 1-2\delta.
\end{equation}
Rearranging~\eqref{eq:delta_bound} gives
\begin{equation*}
	2^M\leqslant\left(\frac{K_\alpha\mathtt{C}_2}{\delta\mathtt{C}_1}\right)^{\frac{1}{\alpha-1}},
\end{equation*}
and inserting $\mathtt{C}_2=\mathtt{C}\|\mathtt{L}\|_{L^2(\Prob_0)}\sqrt{k}$ and then inserting everything into~\eqref{eq:error_bound} gives us the desired result:
\begin{equation*}
	\Probtrue\left(d(\hat\theta_n,\theta_0)\leqslant \left(\frac{\mathtt{C}K_\alpha\|\mathtt{L}\|_{L^2(\Prob_0)}\sqrt{k}}{\delta\mathtt{C}_1\sqrt{n}}\right)^{\frac{1}{\alpha-1}}\right)\geqslant 1-2\delta.
\end{equation*}
Plugging in the Euclidean metric $d(v,w):=\|v-w\|_2$ concludes the proof.
\end{proof}

\section{Diffusion models and the DDPM objective}
\label{app:diffusion_models_ddpm_objective}

\subsection{The forward and backward process}

Diffusion models are based on a forward process which turns data into noise and a reverse process which turns noise back into data. The forward process $(X_t)_{t\in[0,T]}$ for some $T>0$ is an OU process, that is it is a strong solution to the following SDE:
\begin{equation}
	\label{eq:forward_process}
	\dif X_t=-X_t\dif t+\sqrt{2}\dif B_t,\quad X_0 \sim\Probtrue.
\end{equation}
Here, $(B_t)_{t\geqslant0}$ is a $d$-dimensional standard Brownian motion and $\Probtrue$ is the data distribution from which one wishes to generate samples. It\^{o}'s formula~\citep[cf.~Theorem 5.10 in][]{LeG2016} applied to $F(t,X_t):=\exp(t)X_t$ gives
\begin{equation*}
	\exp(t)X_t=F(0,X_0)+\int_0^t\frac{\partial}{\partial s}F(s,X_s)\dif s+\int_0^t\frac{\partial}{\partial x}F(s,X_s)\dif X_s=X_0+\int_0^t\sqrt{2}\exp(s)\dif B_s,
\end{equation*}
and then multiplying both sides by $\exp(-t)$ yields
\begin{equation}
	\label{eq:OU}
	X_t=\exp(-t)X_0+\int_0^t\sqrt{2}\exp(-(t-s))\dif B_s.
\end{equation}
This can be further simplified, since the stochastic integral on the right-hand side of~\eqref{eq:OU} is a Gaussian random variable with mean zero and variance $1-\exp(-2t)$. Indeed, for a step function $\xi$ one easily verifies that
\begin{align*}
	\int_0^t\xi(s)\dif B_s\sim\mathcal{N}\left(0,\int_0^t\xi^2(s)\dif s\right),
\end{align*}
by simply using the definition of the stochastic integral for step functions and the properties of the Brownian motion. Then, since the space of step functions is dense in $L^2([0,t])$ and the map $\xi\mapsto\int_0^t\xi(s)\dif B_s$ is an isometry, hence continuous wrt the $L^2$-norm, for any $f\in L^2([0,t])$ there exists a sequence of step functions $(\xi_n)_{n\in\N}$ such that $\xi_n\rightarrow f$ in $L^2([0,t])$ and therefore $\int_0^t\xi_n(s)\dif B_s\rightarrow\int_0^tf(s)\dif B_s$ in $L^2(\Omega)$, which implies convergence in distribution. In total, this gives us that for any deterministic function $f\in L^2([0,t])$ it holds that
\begin{align*}
	\int_0^tf(s)\dif B_s\sim\mathcal{N}\left(0,\int_0^tf^2(s)\dif s\right).
\end{align*} In particular, since the function $s\mapsto\sqrt{2}\exp(-(t-s))$ is in $L^2([0,t])$, we have
\begin{align*}
	\int_0^t\sqrt{2}\exp(-(t-s))\dif B_s\sim\mathcal{N}\left(0,\int_0^t2\exp(-2(t-s))\dif s\right)=\mathcal{N}\left(0,1-\exp(-2t)\right).
\end{align*}
All in all, we have shown that $X_t=\exp(-t)X_0+\sqrt{1-\exp(-2t)}Z$ for $Z\sim\mathcal{N}(0,\mathbb{I}_d)$ independent of $X_0$, as was mentioned in Section~\ref{subsec:estimators}. We denote by $\Prob_t$ the distribution of $X_t$ and by $f_t$ its density wrt the Lebesgue measure. Note that $\Prob_T\approx\mathcal{N}(0,\mathbb{I}_d)$ for sufficiently large $T$ or more precisely $\lim_{t\rightarrow\infty}\Prob_t=\mathcal{N}(0,\mathbb{I}_d)$.

In order generate samples from $\Probtrue$, the forward diffusion given in~\eqref{eq:forward_process} is reversed in time. This time reversal is possible under some regularity assumptions~\citep{And1982} and the resulting backward process $(\widetilde{X}_t)_{t\in[0,T]}$ with $\widetilde{X}_t=X_{T-t}$ is the solution of the following SDE:
\begin{equation}
	\label{eq:reverse_process}
	\dif \widetilde{X}_t=\{2\nabla\log f_{T-t}(\widetilde{X}_t)+\widetilde{X}_t\}\dif t+\sqrt{2}\dif B_t,\quad \widetilde{X}_0\sim\Prob_T.
\end{equation}
Then, roughly speaking, simulating the backward process starting from $\widetilde{X}_0\sim\mathcal{N}(0,\mathbb{I}_d)$ outputs a sample $\widetilde{X}_T\sim\Probtrue$.

\begin{figure}[ht]
\centering
\begin{tikzpicture}
\begin{axis}[
    width=13cm,
    height=8cm,
    xmin=-4.2, xmax=4.5,
    ymin=-0.55, ymax=4.55,
    axis lines=left,
    xlabel={$x$},
    ylabel={time},
    xtick={-4,-2,0,2,4},
    ytick=\empty,
    clip=true,
]

\addplot[very thick,domain=-4:4,samples=220]
({x},{0.0 + 1.35*(1/sqrt(2*3.1416))*exp(-x^2/2)});
\node[anchor=west, font=\scriptsize] at (axis cs:3.15,0.18) {$t=0$};

\addplot[very thick,domain=-4:4,samples=220]
({x},
 {0.95
  + 1.35*(
    0.78*(1/sqrt(2*3.1416*0.96^2))*exp(-x^2/(2*0.96^2))
    + 0.22*(
        0.62*(1/sqrt(2*3.1416*0.72^2))*exp(-(x+0.95)^2/(2*0.72^2))
      + 0.38*(1/sqrt(2*3.1416*0.58^2))*exp(-(x-1.10)^2/(2*0.58^2))
      )
  )});
\node[anchor=west, font=\scriptsize] at (axis cs:3.15,1.10) {$t=0.25$};

\addplot[very thick,domain=-4:4,samples=220]
({x},
 {1.90
  + 1.35*(
    0.52*(1/sqrt(2*3.1416*0.90^2))*exp(-x^2/(2*0.90^2))
    + 0.48*(
        0.64*(1/sqrt(2*3.1416*0.64^2))*exp(-(x+1.35)^2/(2*0.64^2))
      + 0.36*(1/sqrt(2*3.1416*0.48^2))*exp(-(x-1.55)^2/(2*0.48^2))
      )
  )});
\node[anchor=west, font=\scriptsize] at (axis cs:3.15,2.05) {$t=0.5$};

\addplot[very thick,domain=-4:4,samples=220]
({x},
 {2.85
  + 1.35*(
    0.24*(1/sqrt(2*3.1416*0.84^2))*exp(-x^2/(2*0.84^2))
    + 0.76*(
        0.66*(1/sqrt(2*3.1416*0.56^2))*exp(-(x+1.75)^2/(2*0.56^2))
      + 0.34*(1/sqrt(2*3.1416*0.40^2))*exp(-(x-1.95)^2/(2*0.40^2))
      )
  )});
\node[anchor=west, font=\scriptsize] at (axis cs:3.15,3.00) {$t=0.75$};

\addplot[very thick,domain=-4:4,samples=220]
({x},
 {3.80
  + 1.35*(
      0.68*(1/sqrt(2*3.1416*0.58^2))*exp(-(x+1.85)^2/(2*0.58^2))
    + 0.32*(1/sqrt(2*3.1416*0.34^2))*exp(-(x-2.15)^2/(2*0.34^2))
  )});
\node[anchor=west, font=\scriptsize] at (axis cs:3.15,3.95) {$t=1$};

\pgfplotsextra{

\tikzset{
  scatdot/.style={
    fill=red!95!magenta,
    fill opacity=0.68,
    draw=red!95!magenta,
    draw opacity=0.68
  }
}


\pgfmathsetseed{1001}
\foreach \k in {1,...,90}{
  \pgfmathsetmacro{\uone}{rnd}
  \pgfmathsetmacro{\utwo}{rnd}
  \pgfmathsetmacro{\z}{sqrt(-2*ln(max(\uone,0.0001)))*cos(360*\utwo)}
  \pgfmathsetmacro{\xx}{0 + 1.00*\z}
  \pgfmathsetmacro{\yy}{0.00 + 0.11*(2*rnd-1)}
  \fill[scatdot] (axis cs:\xx,\yy) circle[radius=0.75pt];
}

\pgfmathsetseed{1002}
\foreach \k in {1,...,90}{
  \pgfmathsetmacro{\mix}{rnd}
  \pgfmathsetmacro{\uone}{rnd}
  \pgfmathsetmacro{\utwo}{rnd}
  \pgfmathsetmacro{\z}{sqrt(-2*ln(max(\uone,0.0001)))*cos(360*\utwo)}
  \pgfmathparse{ifthenelse(\mix < 0.78, 0, ifthenelse(\mix < 0.78 + 0.22*0.62, -0.95, 1.10))}
  \let\mu\pgfmathresult
  \pgfmathparse{ifthenelse(\mix < 0.78, 0.96, ifthenelse(\mix < 0.78 + 0.22*0.62, 0.72, 0.58))}
  \let\sigma\pgfmathresult
  \pgfmathsetmacro{\xx}{\mu + \sigma*\z}
  \pgfmathsetmacro{\yy}{0.95 + 0.10*(2*rnd-1)}
  \fill[scatdot] (axis cs:\xx,\yy) circle[radius=0.75pt];
}

\pgfmathsetseed{1003}
\foreach \k in {1,...,88}{
  \pgfmathsetmacro{\mix}{rnd}
  \pgfmathsetmacro{\uone}{rnd}
  \pgfmathsetmacro{\utwo}{rnd}
  \pgfmathsetmacro{\z}{sqrt(-2*ln(max(\uone,0.0001)))*cos(360*\utwo)}
  \pgfmathparse{ifthenelse(\mix < 0.52, 0, ifthenelse(\mix < 0.52 + 0.48*0.64, -1.35, 1.55))}
  \let\mu\pgfmathresult
  \pgfmathparse{ifthenelse(\mix < 0.52, 0.90, ifthenelse(\mix < 0.52 + 0.48*0.64, 0.64, 0.48))}
  \let\sigma\pgfmathresult
  \pgfmathsetmacro{\xx}{\mu + \sigma*\z}
  \pgfmathsetmacro{\yy}{1.90 + 0.09*(2*rnd-1)}
  \fill[scatdot] (axis cs:\xx,\yy) circle[radius=0.75pt];
}

\pgfmathsetseed{1004}
\foreach \k in {1,...,86}{
  \pgfmathsetmacro{\mix}{rnd}
  \pgfmathsetmacro{\uone}{rnd}
  \pgfmathsetmacro{\utwo}{rnd}
  \pgfmathsetmacro{\z}{sqrt(-2*ln(max(\uone,0.0001)))*cos(360*\utwo)}
  \pgfmathparse{ifthenelse(\mix < 0.24, 0, ifthenelse(\mix < 0.24 + 0.76*0.66, -1.75, 1.95))}
  \let\mu\pgfmathresult
  \pgfmathparse{ifthenelse(\mix < 0.24, 0.84, ifthenelse(\mix < 0.24 + 0.76*0.66, 0.56, 0.40))}
  \let\sigma\pgfmathresult
  \pgfmathsetmacro{\xx}{\mu + \sigma*\z}
  \pgfmathsetmacro{\yy}{2.85 + 0.085*(2*rnd-1)}
  \fill[scatdot] (axis cs:\xx,\yy) circle[radius=0.75pt];
}

\pgfmathsetseed{1005}
\foreach \k in {1,...,90}{
  \pgfmathsetmacro{\mix}{rnd}
  \pgfmathsetmacro{\uone}{rnd}
  \pgfmathsetmacro{\utwo}{rnd}
  \pgfmathsetmacro{\z}{sqrt(-2*ln(max(\uone,0.0001)))*cos(360*\utwo)}
  \pgfmathparse{ifthenelse(\mix < 0.68, -1.85, 2.15)}
  \let\mu\pgfmathresult
  \pgfmathparse{ifthenelse(\mix < 0.68, 0.58, 0.34)}
  \let\sigma\pgfmathresult
  \pgfmathsetmacro{\xx}{\mu + \sigma*\z}
  \pgfmathsetmacro{\yy}{3.80 + 0.08*(2*rnd-1)}
  \fill[scatdot] (axis cs:\xx,\yy) circle[radius=0.75pt];
}

} 

\end{axis}
\end{tikzpicture}

\caption{A qualitative depiction of the evolution induced by the reverse SDE in terms of the distribution $\Prob_{\theta,t}$ and data $\{X^{(i)}_t\}_{i=1}^n$ at time $t\in\{0,0.25,0.5,0.75,1\}$. In contrast to classical LMC, in the reverse SDE~\eqref{eq:reverse_process} both the data and the score function evolve over time, such that the necessity of traversing low-density regions does not pose a problem. This is a crucial difference that is responsible for the superior performance of diffusion models in contrast to classical LMC.}
\label{fig:diffusion-bimodal-particles}
\end{figure}

\subsection{The DDPM objective}

The score function $\nabla\log f_t$ which appears in the reverse process~\eqref{eq:reverse_process} is unknown and needs to be learned for every $t\in[0,T]$. This is done by minimizing the following objective:
\begin{equation}
	\label{eq:ddpm_objective}
	\ell\left((s_\theta(t))_{t\in[0,T]}\right)=\int_0^T\E_{X_t}\left[\|\nabla\log f_t(X_t)-s_\theta(X_t,t)\|_2^2\right]\dif t
\end{equation}
At first glance, the above objective is simply the Fisher divergence as in~\eqref{eq:fisher_distance} integrated over time. It would be intuitive to perform the same integration by parts in order to obtain a tractable objective. But the fact that $X_t$ is a noisy version of $X_0$ shows that the given setting is exactly the one of DSM. In other words, DSM is not a surrogate method to estimate DDPMs, but instead vanilla SM based on $X_t$ is equivalent to DSM based on $X_0$ and the noising mechanism induced by the forward process. In~\eqref{eq:ddpm_objective}, $s_\theta$ is generally a neural network with parameters $\theta$. But note that
\begin{equation*}
		\label{eq:convolved_density_forward_process}
		f_t(x)\propto\int_\samplespace f_0(\exp(t)(x-\tau))\varphi\left((1-\exp(-2t))^{-1/2}\tau\right)\dif \tau,
	\end{equation*}
such that the score function $\nabla\log f_t$ depends on the true but unknown data distribution $\Probtrue$ through its density $f_0$. The same holds for the expectation inside the integral in~\eqref{eq:ddpm_objective}, since $X_t$ is distributed according to $\Prob_t$ which also depends on $\Probtrue$. As a consequence, the DDPM objective as given in~\eqref{eq:ddpm_objective} is intractable, since we have access to $\Probtrue$ only through samples $\{X_i\}_{i=1}^n$. However, DSM as suggested by~\citet{Vin2011} solves exactly this problem.

The expectation inside the integral in~\eqref{eq:ddpm_objective} can be expanded in the same way as the Fisher divergence in~\eqref{eq:fisher_distance_2} in Section~\ref{subsec:estimators}:
\begin{equation}
	\label{eq:ddpm_expansion}
	\E_{X_t}\left[\|\nabla\log f_t(X_t)-s_\theta(X_t,t)\|^2\right]=\E_{X_t}\left[\|s_\theta(X_t,t)\|^2\right]-2\E_{X_t}\left[\langle\nabla\log f_t(X_t), s_\theta(X_t,t)\rangle\right]+\mathtt{C}(\Prob_t).
\end{equation}
Then, the scalar product term can be rewritten as follows:
\begin{align*}
	\E_{X_t}\left[\langle\nabla\log f_t(X_t), s_\theta(X_t,t)\rangle\right]&=\E_{X_t}\left[\frac{\langle\nabla\E_{X_0}[f_t(X_t\mid X_0)], s_\theta(X_t,t)\rangle}{f_t(X_t)}\right]\\
	&=\E_{X_0}\left[\E_{X_t}\left[\frac{\langle\nabla f_t(X_t\mid X_0), s_\theta(X_t,t)\rangle}{f_t(X_t\mid X_0)f_t(X_t)}f_t(X_t\mid X_0)\right]\right]\\
	&=\E_{X_0}\left[\E_{X_t\mid X_0}\left[\langle\nabla\log f_t(X_t\mid X_0), s_\theta(X_t,t)\rangle\right]\right].
\end{align*}
Here, we denote $f_t(\cdot\mid X_0)$ the density of $X_t$ conditional on $X_0$. The first equality is based on the definition of the score function $\nabla\log f_t$ and the identity $f_t(X_t)=\E_{X_0}[f_t(X_t\mid X_0)]$ which is based on the law of total expectation. The second equality follows with Fubini and the third amounts to cancelling densities and thus changing the measure wrt which one integrates.
Next, we insert this back into~\eqref{eq:ddpm_expansion}, discart the constant $\mathtt{C}(\Prob_t)$ independent of $\theta$ and use the fact that
\begin{equation*}
	\nabla\log f_t(X_t\mid X_0)=-\frac{X_t-\exp(-t)X_0}{1-\exp(-2t)}=-\frac{Z_t}{\sqrt{1-\exp(-2t)}},
\end{equation*}
since $X_t\mid X_0\sim\mathcal{N}(\exp(-t)X_0,(1-\exp(-2t))I)$. This gives us
\begin{equation}
	\label{eq:ddpm_objective_2}
	\ell\left((s_\theta(t))_{t\in[0,T]}\right)=\int_0^T\E_{X_0}\left[\E_{X_t\mid X_0}\left[\|s_\theta(X_t,t)\|^2-\frac{2}{\sqrt{1-\exp(-2t)}}\langle s_\theta(X_t,t),Z_t\rangle\right]\right]\dif t,
\end{equation}
which coincides with~\eqref{eq:ddpm_loss}, except for the fact that $s_{\theta,t}$ in~\eqref{eq:ddpm_loss} is based on the density $f_\theta$ of some probability law $\Prob_\theta$ and $s_\theta$ is a neural network that is trained to approximate the score function $\nabla\log f_t$.

\section{Omitted proofs}

\subsection{Proof of Lemma~\ref{lemma:embedded_one_dim_sym_mixture}}\label{app:lemma_embedded_one_dim_sym_mixture}

	1. For arbitrary but fixed $\mu_1,\mu_2\in\R^d$ with $\mu_1\not=\mu_2$, let $\mathtt{m}:=(\mu_1+\mu_2)/2$ denote the midpoint and $\Delta:=\|\mu_1-\mu_2\|_2/2$ half the distance between the two means. Then, there exists an orthonormal matrix $Q\in\R^{d\times d}$ such that $Q(\mu_1-\mathtt{m})=\Delta e_1$. Indeed, the vector
	\begin{equation*}
		u:=\frac{\mu_1-\mu_2}{\|\mu_1-\mu_2\|_2}
	\end{equation*}
	can be extended to an orthonormal basis $\{u,v_2,\ldots,v_d\}$ of $\R^d$ via the Gram-Schmidt procedure, such that the matrix $Q:=[u,v_2,\ldots,v_d]^{\top}$ is orthonormal and satisfies $Qu=e_1$. Define the isometry $T:\R^d\to\R^d$ by $T(x):=Q(x-\mathtt{m})$. Then, one can check that
	\begin{align*}
		T(\mu_1) &= Q(\mu_1-\mathtt{m}) = \Delta e_1,\\
		T(\mu_2) &= Q(\mu_2-\mathtt{m}) = -\Delta e_1.
	\end{align*}
	With this, for $X\sim\Prob_\theta$ we have $X\stackrel{d}{=}\theta Z+(1-\theta)Y$ with $Z\sim\mathcal{N}(\mu_1,\mathbb{I}_d)$ and $Y\sim\mathcal{N}(\mu_2,\mathbb{I}_d)$. Then, by definition of the pushforward measure we have that $T(X)\sim\mathbb{Q}_\theta=\Prob_\theta\circ T^{-1}$ and with the linearity of $T$ we can write
	\begin{equation*}
		T(X)\stackrel{d}{=}\theta T(Z)+(1-\theta)T(Y),
	\end{equation*}
	where $T(Z)\sim\mathcal{N}(T(\mu_1),\mathbb{I}_d)$ and $T(Y)\sim\mathcal{N}(T(\mu_2),\mathbb{I}_d)$, using the orthonormality of $Q$. Hence, the pushforward $\mathbb{Q}_\theta$ is a mixture of two isotropic Gaussians with means $T(\mu_1)$ and $T(\mu_2)$ and mixture weight $\theta$. To see that $T$ is an isometry, note that it preserves distances: $\|T(x)-T(y)\|_2=\|Q(x-\mathtt{m})-Q(y-\mathtt{m})\|_2=(x-y)Q^{\top}Q(x-y)=\|(x-y)\|_2$ for all $x,y\in\R^d$. This also implies that the distance between the means of the mixture components is preserved under $T$, that is $\|T(\mu_1)-T(\mu_2)\|_2=\|\mu_1-\mu_2\|_2=2\Delta$. This proves the first statement.\\

	\noindent
	2. With $T(\mu_1)=\Delta e_1$, we can write
	\begin{align*}
		\phi(x-T(\mu_1))&=(2\pi)^{-d/2}\exp\left(-\frac{1}{2}\|x-\Delta e_1\|_2^2\right)\\
		&=(2\pi)^{-d/2}\exp\left(-\frac{1}{2}(x_1-\Delta)^2-\frac{1}{2}\|x_{2:d}\|_2^2\right)\\
		&=\varphi(x_1-\Delta)\phi(x_{2:d}).
	\end{align*}
	Analogously, we have $\phi(x-T(\mu_2))=\varphi(x_1+\Delta)\phi(x_{2:d})$. Hence, we can write	
	\begin{align*}
		q_\theta(x)&=\theta \phi(x-T(\mu_1))+(1-\theta)\phi(x-T(\mu_2))\\
		&=\theta \varphi(x_1-\Delta)\phi(x_{2:d})+(1-\theta)\varphi(x_1+\Delta)\phi(x_{2:d})\\
		&=\left(\theta \varphi(x_1-\Delta)+(1-\theta)\varphi(x_1+\Delta)\right)\phi(x_{2:d}),
	\end{align*}
	which is the desired factorization of $q_\theta$. This clearly implies that $\mathbb{Q}_\theta=\mathbb{Q}_{\theta,1}\otimes\mathcal{N}(0,\mathbb{I}_{d-1})$ with $\mathbb{Q}_{\theta,1}$ being the marginal distribution of $\mathbb{Q}_\theta$ in the first coordinate.\qed

\subsection{Proof of Proposition~\ref{proposition:gaussian_isoperimetric_constant}}\label{app:proposition_gaussian_isoperimetric_constant}

Since translations are isometries, it follows from 1. of Lemma~\ref{lemma:properties_isoperimetric_constants} that the isoperimetric constant $\CIP(\Prob)$ of the Gaussian distribution $\Prob=\mathcal{N}(\mu,\Sigma)$ does not depend on the mean $\mu$. Therefore, we can assume without loss of generality that $\Prob=\mathcal{N}(0,\Sigma)$. Then, since half-spaces $H_{u,t}:=\{x\in\R^d\mid\langle u,x\rangle\leqslant t\}$ with $u\in\mathbb{S}^{d-1}:=\{x\in\R^d:\|x\|_2=1\}$ and $t\in\R$ are extremal sets~\citep{Bor1975,SudTsi1978}, we can compute
\begin{align*}
	\CIP(\Prob)&=\inf_{\substack{u\in\mathbb{S}^{d-1}\\t\in\R}}\frac{\Prob^+(H_{u,t})}{\Prob(H_{u,t})\wedge\Prob(H_{u,t}^\complement)}=\inf_{u\in\mathbb{S}^{d-1}}\inf_{t\in\R}\frac{\Prob^+(H_{u,t})}{\Prob(H_{u,t})\wedge\Prob(H_{u,t}^\complement)}.
\end{align*}
The second equality follows from the elementary identity 
\begin{equation*}
    \inf_{(x,y)\in \mathcal{X}\times\mathcal{Y}} f(x,y)=\inf_{x\in \mathcal{X}}\inf_{y\in \mathcal{Y}} f(x,y),
\end{equation*}
valid for any function $f:\mathcal{X}\times\mathcal{Y}\rightarrow\overline{\R}$. This follows immediately from the definition of the infimum as the greatest lower bound. With $\langle u, X\rangle\sim\mathcal{N}(0,\langle u,\Sigma u\rangle)$ the boundary measure of $H_{u,t}$ can be computed as
\begin{align*}
	\Prob^+(H_{u,t}):=&\lim_{\varepsilon\downarrow0}\frac{\Prob(H_{u,t}^\varepsilon)-\Prob(H_{u,t})}{\varepsilon}\\
	=&\lim_{\varepsilon\downarrow0}\frac{\Phi\left(\langle u,\Sigma u\rangle^{-1/2}(t+\varepsilon)\right)-\Phi\left(\langle u,\Sigma u\rangle^{-1/2}t\right)}{\varepsilon}\\
	=&\langle u,\Sigma u\rangle^{-1/2}\varphi\left(t\langle u,\Sigma u\rangle^{-1/2}\right).
\end{align*}
The last equality follows from the definition of the derivative and the fact that $\Phi$ is continuously differentiable. With this we arrive at
\begin{equation*}
	\CIP(\Prob)=\inf_{u\in\mathbb{S}^{d-1}}\frac{1}{\sigma_u}\left\{\inf_{t\in\R}\frac{\varphi\left(t/\sigma_u\right)}{\Phi\left(t/\sigma_u\right)\wedge\left\{1-\Phi\left(t/\sigma_u\right)\right\}}\right\},
\end{equation*}
where $\sigma_u^2:=\langle u,\Sigma u\rangle$.

To compute the inner infimum we can fix $u\in\mathbb{S}^{d-1}$ and treat $\sigma_u^2$ as a constant. For ease of notation we therefore perform the change of variables $x=t/\sigma_u$. Then, since $\varphi$ is symmetric and unimodal, we need to compute $\inf_{x\leqslant 0}\lambda(x)$, where $\lambda$ is the inverse Mills ratio
	\begin{align*}
		\lambda:\R\rightarrow\R_+,\quad
		x\mapsto\frac{\varphi(x)}{\Phi(x)}.
	\end{align*}
	One can show that $\lambda'<0$ (or alternatively $\lambda'>0$ when $\lambda(x):=\varphi(x)/(1-\Phi(x))$ as done by~\citet{Sam1953}) and together with the fact that $\lambda$ is continuous, it follows that the minimum is attained at $x=0$.

Indeed, the derivative of $\lambda$ is 
	\begin{align*}
		\lambda'(x)&=\frac{\varphi'(x)\Phi(x)-\varphi(x)\Phi'(x)}{\Phi^2(x)}
		=\frac{-x\varphi(x)\Phi(x)-\varphi(x)\varphi(x)}{\Phi^2(x)}
		=-\lambda(x)(x+\lambda(x)).
	\end{align*}
	Since $\lambda>0$ it remains to be shown that $x+\lambda(x)>0$ for $x\leqslant0$. To that end, note that
	\begin{align*}
		\E[X \mid X<x]=\frac{1}{\Phi(x)}\int_{-\infty}^xt\varphi(t)\dif t
		=\frac{1}{\Phi(x)\sqrt{2\pi}}\int_{-\infty}^xt\exp(-t^2/2)\dif t,
	\end{align*}
	for $X\sim\Prob$. The substitution $u(t)=t^2/2$ with $\dif u=t\dif t$ together with the identity
	\begin{equation*}
		\int_a^b f'(t)\dif t=f(b)-f(a)=-(f(a)-f(b))=-\int_b^a f'(t)\dif t
	\end{equation*}
	then yields
	\begin{equation*}
		\E[X \mid X<x]=-\frac{1}{\Phi(x)\sqrt{2\pi}}\int_{x^2/2}^\infty\exp(-u)\dif u
		=-\frac{\varphi(x)}{\Phi(x)},
	\end{equation*}
	In total we obtain $\E[X | X<x]=-\lambda(x)$ and therefore
	\begin{equation*}
		x+\lambda(x)=x-\E[X | X<x]>0,
	\end{equation*}
	because trivially it holds that $\E[X | X<x]<x$. From $\lambda'<0$ we can conclude that $\lambda$ is monotonically decreasing and therefore attains its infimum at the right boundary of its domain, that is, $x=0$.
	
Lastly, it is well known that the Rayleigh quotient of a matrix $M$ is lower and upper bounded by its smallest and largest eigenvalues respectively:
\begin{equation*}
	\lambda_{\min}(M)\leqslant\frac{\langle x,Mx\rangle}{\langle x,x\rangle}\leqslant \lambda_{\max}(M).
\end{equation*}
Combining this with the previous part gives us
\begin{equation*}
	\CIP(\Prob)=\frac{\lambda(0)}{\sqrt{\lambda_{\max}(\Sigma)}}=\frac{2\varphi(0)}{\sqrt{\lambda_{\max}(\Sigma)}},
\end{equation*}
which concludes the proof.\qed

\subsection{Proof of Lemma~\ref{lemma:properties_isoperimetric_constants}}\label{app:lemma_properties_isoperimetric_constants}

1. Let $B\in\mathcal{B}(\R^d)$ be a Borel set and $d(x,y):=\|x-y\|_2$ denote the Euclidean metric. Since $T$ is an isometry it is continuous and thus measurable, such that $A:=T^{-1}(B)$ is well-defined and a Borel set as well. By definition of $\mathbb{Q}$ it holds that $\Prob(A)=\mathbb{Q}(B)$. And using the fact that $d(T(x),T(y))=d(x,y)$ for all $x,y\in\R^d$, we have for any $\varepsilon>0$ that
	\begin{align*}
		A_\varepsilon&:=\left\{x\in\R^d:\inf_{y\in A}d(x,y)\leqslant\varepsilon\right\}
		=\left\{x\in\R^d:\inf_{y\in A}d(T(x),T(y))\leqslant\varepsilon\right\}\\
		&=\left\{x\in\R^d:\inf_{z\in B}d(T(x),z)\leqslant\varepsilon\right\}
		=\left\{x\in\R^d:T(x)\in B_\varepsilon\right\}\\
		&=T^{-1}(B_\varepsilon).
	\end{align*}
	With Definition~\ref{definition:minkowski_content_isoperimetric_constant} it follows that
	\begin{equation*}
		\Prob^+(A)=\lim_{\varepsilon\downarrow0}\frac{\Prob(A_\varepsilon)-\Prob(A)}{\varepsilon}=\lim_{\varepsilon\downarrow0}\frac{\mathbb{Q}(B_\varepsilon)-\mathbb{Q}(B)}{\varepsilon}=\mathbb{Q}^+(B),
	\end{equation*}
	such that for every Borel set $B\in\mathcal{B}(\R^d)$ we obtain
	\begin{equation*}
		\frac{\Prob^+(A)}{\Prob(A)\wedge \Prob(A^\complement)}=\frac{\mathbb{Q}^+(B)}{\mathbb{Q}(B)\wedge \mathbb{Q}(B^\complement)}.
	\end{equation*}
	Since $T$ is bijective, taking the infimum over all Borel sets $B\in\mathcal{B}(\R^d)$ is equivalent to taking the infimum over all Borel sets $A=T^{-1}(B)\in\mathcal{B}(\R^d)$. This shows that $\CIP(\Prob)= \CIP(\mathbb{Q})$.\\
	
	\noindent
	2. We will show $\CIP(\Prob_\theta)=\CIP(\mathbb{Q}_\theta)=\CIP(\mathbb{Q}_{\theta,1})$. The first equality follows immediately from 1. together with Lemma~\ref{lemma:embedded_one_dim_sym_mixture}. For the second equality, we require Theorem 1.1 from~\citet{Lie2026}. Among other things, it states that extremal sets of any measure $\Prob_\theta$ from the family~\ref{eq:gm} are half-spaces perpendicular to the vector
	\begin{equation*}
		\nu:=\frac{\mu_1-\mu_2}{\|\mu_1-\mu_2\|_2}.
	\end{equation*}
	In the case of the distribution $\mathbb{Q}_\theta$, we have $\nu=e_1$ and we can restrict the infimum in the definition of $\CIP(\mathbb{Q}_\theta)$ to sets of the form $H_t:=(-\infty,t]\times\R^{d-1}$ for $t\in\R$. This gives
	\begin{align*}
		\CIP(\mathbb{Q}_\theta)
		=
		\inf_{t\in\R}
		\frac{\mathbb{Q}_\theta^+(H_t)}{\mathbb{Q}_\theta(H_t)\wedge \mathbb{Q}_\theta(H_t^\complement)}
		=
		\inf_{t\in\R}
		\frac{\mathbb{Q}_{\theta,1}^+((-\infty,t])}{\mathbb{Q}_{\theta,1}((-\infty,t])\wedge \left(1-\mathbb{Q}_{\theta,1}((-\infty,t])\right)}
		=
		\CIP(\mathbb{Q}_{\theta,1}),
	\end{align*}
	which concludes the proof.\qed

\subsection{Proof of Proposition~\ref{proposition:isoperimetric_constant}}\label{appendix:proof_isoperimetric_constant}

Let $(\Prob_\theta)_{\theta\in\Theta}$ be the Gaussian mixture model~\ref{eq:gm} with parameter space $\Theta=[0,1]$ and nuisance parameters $\mu_1,\mu_2\in\R^d$. We split the proof into multiple steps.\\

\noindent
\textbf{Step 1: Reduction to the symmetric one-dimensional case.} We will first show the claim for the case $d=1$ and $\mu_1=-\mu_2=\mu$ for some $\mu>0$. In the last step we will extend the result to the general case on $\R^d$. Therefore, for now we fix
\begin{equation}
	\label{eq:one_dim_embedding}
	\Prob_\theta:=\theta\mathcal{N}(\mu,1)+(1-\theta)\mathcal{N}(-\mu,1),
\end{equation}
for $\mu>0$ and we denote its density by $f_\theta$ and its distribution function by $F_\theta$.\\

\noindent
\textbf{Step 2: Simplification via symmetry and $\CIP(\Prob_{1/2})=2\varphi(\mu)$.} Since $f_\theta$ is continuous and strictly positive on $\R$, the essential infimum in~\eqref{eq:real_line_isoperimetric_constant} reduces to the usual infimum. Then, it is sufficient to show that
\begin{equation}
	\label{eq:lambda_bound}
	\lambda(\theta,x):=\frac{f_\theta(x)}{F_\theta(x)}\geqslant 2\varphi(\mu),
\end{equation}
for all $(\theta,x)\in\Omega\subseteq\Theta\times\R$, where $\Omega:=\{(\theta,x)\in\Theta\times\R:F_\theta(x)\leqslant 1/2\}$. Indeed, the symmetric mixture $\Prob_{1/2}$ has the isoperimetric constant
\begin{equation}
	\label{eq:liehr_result}
	 \CIP(\Prob_{1/2})=2\varphi(\mu),
\end{equation}
which follows from Theorem 1.2 in recent work by~\citet{Lie2026}. Therefore, the claim follows if we can show that
\begin{equation}
	\label{eq:lambda_bound_2}
	\frac{f_\theta(x)}{F_\theta(x)\wedge\{1-F_\theta(x)\}}\geqslant 2\varphi(\mu)=\lambda(1/2,0),
\end{equation}
for all $(\theta,x)\in\Theta\times\R$. Restricting to the subset $\Omega\subseteq\Omega\times\R$, in order to avoid a case-by-case analysis wrt the denominator on the left-hand side of~\eqref{eq:lambda_bound_2}, is possible due to the following symmetry: the lower tail probability $\Prob_\theta(X\leqslant x)$ is equal to the upper tail probability $\Prob_{1-\theta}(X\geqslant -x)$. More precisely, $\varphi(x-\mu)=\varphi(-x+\mu)$ implies $f_\theta(x)=f_{1-\theta}(-x)$ and $F_\theta(x)=1-F_{1-\theta}(-x)$, such that $\lambda(\theta,x)=\lambda(1-\theta,-x)$. Then, for every tuple $(\theta,x)\in\Omega^\complement$ we know that $1-F_\theta(x)\leqslant 1/2$ and with the preceeding identity it follows that $F_{1-\theta}(-x)\leqslant 1/2$, thus $(1-\theta,-x)\in\Omega$ and $\lambda(\theta,x)=\lambda(1-\theta,-x)$. In other words, for every $(\theta,x)\in\Omega^\complement$ there exists a tuple $(1-\theta,-x)\in\Omega$ with $\lambda(\theta,x)=\lambda(1-\theta,-x)$, which implies that the infimum of $\lambda$ over $\Theta\times\R$ is attained on $\Omega$.\\

\noindent
\textbf{Step 3: Fix $x\in\R$ and determine boundary solution $\theta_x$.} Let us proceed to prove~\eqref{eq:lambda_bound}. To that end, let $x\in\R$ be arbitrary but fixed and define the terms $A=\varphi(x-\mu)$, $B=\varphi(x+\mu)$, $C=\Phi(x-\mu)=F_1(x)$ and $D=\Phi(x+\mu)=F_0(x)$, independent of $\theta$. Then, we are interested in finding a minimizer $\theta_x$ of the map
	\begin{equation*}
		\theta\mapsto \lambda_x(\theta):=\lambda(\theta,x)=\frac{B+\theta(A-B)}{D+\theta(C-D)},
	\end{equation*}
with domain $\Omega_x:=\{\theta\in\Theta:F_\theta(x)\leqslant 1/2\}$. Due to the monotonicity of $\Phi$, which implies that $C<D$, the set $\Omega_x$ must have one of the following shapes:
\begin{align}
	\label{eq:omega_empty}
	\Omega_x&=\emptyset \text{ if } x>\mu \text{ such that } 1/2<C=F_1(x)\leqslant F_\theta(x) \text{ for all } \theta\in\Theta,\tag{$\emptyset$}\\
	\label{eq:omega_interval}
	\Omega_x&=[\mathtt{c},1] \text{ for some } \mathtt{c}\in[0,1] \text{ if } |x|\leqslant\mu \text{ such that } F_1(x)=C\leqslant 1/2\leqslant D=F_0(x),\tag{$\mathtt{int}$}\\
	\label{eq:omega_full}
	\Omega_x&=[0,1] \text{ if } x<-\mu \text{ such that } F_1(x)=C<D=F_0(x)<1/2.\tag{$\Theta$}
\end{align}
In order to find the minimizer $\theta_x$, we compute the derivative
	\begin{equation}
		\label{eq:lambda_x_derivative}
		\theta\mapsto \partial_\theta\lambda_x(\theta)=\frac{AD-BC}{(D+\theta(C-D))^2}.
	\end{equation}
Since $0<(D+\theta(C-D))^2\leqslant 1/4$ by definition of $\Omega_x$, as well as $D>0$ and $C-D<0$ due to the monotonicity of $\Phi$, the denominator in~\eqref{eq:lambda_x_derivative} is monotonically decreasing and the sign of the numerator $AD-BC$ determines whether $\partial_\theta\lambda_x$ is monotonically increasing or decreasing. Then, for $\theta_2>\theta_1$ the mean value theorem implies the existence of $\widetilde{\theta}\in(\theta_1,\theta_2)$ with
	\begin{equation*}
		\lambda_x(\theta_2)-\lambda_x(\theta_1)=\partial_\theta\lambda_x(\widetilde{\theta})(\theta_2-\theta_1),
	\end{equation*}
	which shows that $\lambda_x$ is monotone as well. And since we know from Proposition~\ref{proposition:gaussian_isoperimetric_constant} that the inverse Mill's ratio $x\mapsto\varphi(x)/\Phi(x)$ is monotonically decreasing, from $x-\mu< x+\mu$ it follows that
	\begin{equation*}
		\lambda_x(0)=\frac{B}{D}=\frac{\varphi(x+\mu)}{\Phi(x+\mu)}\leqslant \frac{\varphi(x-\mu)}{\Phi(x-\mu)}=\frac{A}{C}=\lambda_x(1),
	\end{equation*}
	which implies that $\lambda_x$ is monotonically increasing and therefore can only attain its infimum on the left boundary point of $\Omega_x$.

	In the case~\eqref{eq:omega_empty}, there is nothing to show. Disregarding the monotonicity of $\lambda_x$ for a moment, the boundary points $\{0,1\}$ appearing in both cases~\eqref{eq:omega_interval} and~\eqref{eq:omega_full} are not candidates, since for any $\mu>0$ it holds that
	\begin{equation*}
		\CIP(\Prob_0)=\CIP(\Prob_1)=2\varphi(0)>2\varphi(\mu)=\CIP(\Prob_{1/2}),
	\end{equation*}
	where the equalities on the left-hand side follow from Proposition~\ref{proposition:gaussian_isoperimetric_constant} and the identity on the right-hand side from~\eqref{eq:liehr_result}, which was taken from~\citet{Lie2026}. Therefore, the only candidate for the minimizer is given by $\theta_x=\mathtt{c}$ in case~\eqref{eq:omega_interval} holds and because $\theta_x$ lies on the boundary of $\Omega_x$, it follows that $F_{\theta_x}(x)=1/2$.\\

	\noindent
	\textbf{Step 4: Find minimizer of $x\mapsto\lambda_x(\theta_x)$.} In the previous step we have shown two things: First, for any $|x|>\mu$ there does not exist a $\theta\in\Theta$ such that $(\theta,x)$ minimizes $\lambda$. By contraposition, this tells us that for all $x\in\R$ if there exists a $\theta\in\Theta$ such that $(\theta,x)$ minimizes $\lambda$, then $|x|\leqslant\mu$ most hold. This coincides with the intuition that the minimizer in $x$ has to lie in the interval $[0,\mu]$ when $\theta\in[1/2,1]$, since $\CIP(\Prob_{1/2})=2f_{1/2}(0)=2\varphi(\mu)$ and $\CIP(\Prob_1)=2f_1(\mu)=2\varphi(0)$. Second, when $|x|\leqslant\mu$, it holds that
	\begin{equation*}
		\lambda(\theta,x)\geqslant \lambda_x(\theta_x)=\frac{f_{\theta_x}(x)}{1/2}=2f_{\theta_x}(x),
	\end{equation*}
	for all $\theta\in\Theta$ and the infimum of $\lambda$ over $\Omega$ is attained at $(\theta_x,x)$ with $F_{\theta_x}(x)=1/2$.
	
	To finish the proof, we will show that $x\mapsto\lambda_x(\theta_x)=2f_{\theta_x}(x)$ is minimized at $x=0$ for all $x\in\mathbb{R}$ with $|x|\leqslant\mu$ and that $\lambda_0(\theta_0)=2f_{\theta_0}(0)=2\varphi(\mu)$ with $\theta_0=1/2$. To that end, first note that $F_{\theta_x}(x)=1/2$ gives
	\begin{equation}
		\label{eq:theta_x}
		\theta_x=\frac{1/2-D}{C-D}=\frac{\Phi(x+\mu)-1/2}{\Phi(x+\mu)-\Phi(x-\mu)}=\frac{\tau(x+\mu)}{\tau(x+\mu)+\tau(\mu-x)},
	\end{equation}
	where $\tau(t):=\Phi(t)-1/2$. Inserting $x=0$ into the above, one can check that $\theta_0=1/2$. Next, plugging~\eqref{eq:theta_x} into $f_{\theta_x}(x)$ yields
	\begin{equation*}
		f_{\theta_x}(x)=\frac{\tau(x+\mu)\varphi(x-\mu)+\tau(\mu-x)\varphi(x+\mu)}{\tau(x+\mu)+\tau(\mu-x)}=:\frac{n(x)}{d(x)},
	\end{equation*}
	where we use $n(x)$ and $d(x)$ to denote the numerator and denominator, respectively. After checking that $f_{\theta_0}(0)=\varphi(\mu)$, it remains to be shown that $x=0$ is the minimizer.
	
	To do so, we will show that $f'_{\theta_0}(0)=0$ and $f'_{\theta_x}(x)>0$ for all $x\in(0,\mu)$. Note, that is sufficient to only consider $x\in(0,\mu)$ due to the symmetry of $f_{\theta_x}$, that is, $f_{\theta_x}(x)=f_{\theta_x}(-x)$, together with the previously mentioned fact that $\CIP(\Prob_1)=2f_1(\mu)$ and the observation that $f'_{\theta_0}(0)=0$. Indeed, we can compute
	\begin{align*}
		n'(x)&=(\mu-x)\tau(x+\mu)\varphi(x-\mu)-(x+\mu)\tau(\mu-x)\varphi(x+\mu),\\
		d'(x)&=\varphi(x+\mu)-\varphi(x-\mu).
	\end{align*}
	With the quotient rule we obtain
	\begin{align}
		\nonumber
		f'_{\theta_x}(x)&=\frac{n'(x)d(x)-n(x)d'(x)}{d^2(x)}\\
		\nonumber
		&=\frac{\{(\mu-x)\tau(x+\mu)\varphi(x-\mu)-(x+\mu)\tau(\mu-x)\varphi(x+\mu)\}(\tau(x+\mu)+\tau(\mu-x))}{(\tau(x+\mu)+\tau(\mu-x))^2}\\
		\label{eq:f_derivative}
		&\quad+\frac{\{(\tau(x+\mu)\varphi(x-\mu)+\tau(\mu-x)\varphi(x+\mu)\}(\varphi(x-\mu))-\varphi(x+\mu)}{(\tau(x+\mu)+\tau(\mu-x))^2}.
	\end{align}
	From $n'(0)=d'(0)=0$ we see that $f'_{\theta_0}(0)=0$. Since $0<\mu-x<\mu+x$ we have $\varphi(x-\mu)>\varphi(x+\mu)$, and thus the second summand in~\eqref{eq:f_derivative} is positive. For the first summand, note that $\tau(x+\mu)+\tau(\mu-x)>0$ for $x\in(0,\mu)$ and thus we have to show that
	\begin{equation*}
		(\mu-x)\tau(x+\mu)\varphi(x-\mu)-(x+\mu)\tau(\mu-x)\varphi(x+\mu)>0,
	\end{equation*}
	which is equivalent to
	\begin{equation*}
		\frac{\tau(x+\mu)}{(x+\mu)\varphi(x+\mu)}>\frac{\tau(\mu-x)}{(\mu-x)\varphi(\mu-x)}.
	\end{equation*}
	Given that $x+\mu>\mu-x$, showing that the map $t\mapsto\tau(t)/(t\varphi(t))$ is strictly increasing for $t>0$ is sufficient to conclude that the above inequality holds. To see this, we compute
	\begin{equation*}
		t\mapsto\frac{\tau(t)}{t\varphi(t)}=\frac{1}{t}\int_0^t\exp\left(\frac{t^2-s^2}{2}\right)\dif s=\int_0^1\exp\left(\frac{t^2(1-r^2)}{2}\right)\dif r,
	\end{equation*}
	by using the definition of $\tau$ and substituting $s=rt$. The integral on the right-hand side is clearly strictly increasing for $t>0$ and thus the map $t\mapsto\tau(t)/(t\varphi(t))$ is strictly increasing as well. This proves $f'_{\theta_x}(x)>0$ for all $x\in(0,\mu)$ and since $x$ was chosen arbitrarily at the beginning, this concludes the proof.\\

	\noindent
	\textbf{Step 5: Putting everything together.} Let $\Prob_\theta$ be a distribution from~\ref{eq:gm}, that is,
	\begin{equation*}
		\Prob_\theta=\theta\mathcal{N}(\mu_1,\mathbb{I}_d)+(1-\theta)\mathcal{N}(\mu_2,\mathbb{I}_d),
	\end{equation*}
	for $\theta\in(0,1)$ and nuisance parameters $\mu_1,\mu_2\in\R^d$ with $\mu_1\not=\mu_2$. From 2. of Lemma~\ref{lemma:properties_isoperimetric_constants} we know that $\CIP(\Prob_\theta)=\CIP(\mathbb{Q}_{\theta,1})$ for all $\theta\in\Theta$, where $\mathbb{Q}_{\theta,1}$ is the marginal distribution of $\mathbb{Q}_\theta$ in its first component and $\mathbb{Q}_\theta$ is the pushforward of $\Prob_\theta$ constructed in Lemma~\ref{lemma:embedded_one_dim_sym_mixture}. Combining this with an application of the results from Steps 2 to 4 to $\mathbb{Q}_{\theta,1}$ we obtain
	\begin{equation*}
		\CIP(\Prob_\theta)=\CIP(\mathbb{Q}_{\theta,1})=\inf_{x\in\R}\lambda(\theta,x)\geqslant \inf_{x\in\R}\lambda(\theta_x,x)=\lambda(\theta_0,0)=2f_{\theta_0}(0)=2\varphi(\Delta).
	\end{equation*}
	And since $\CIP(\Prob_{1/2})=2\varphi(\Delta)$ was shown in Theorem 1.2 of~\citet{Lie2026}, the lower bound is attained for $\theta=1/2$. The limiting behavior is an obvious consequence of the exponentially decaying tails of the Gaussian distribution. This concludes the proof.
	\qed

\subsection{Proof of Lemma~\ref{lemma:gm_properties}}\label{appendix:proof_lemma_gm_properties}
    
1. With the straightforward computation
\begin{equation*}
	\nabla f_\theta(x)=-\theta(x-\mu_1)\phi(x-\mu_1)-(1-\theta)(x-\mu_2)\phi(x-\mu_2),
\end{equation*}
we obtain
\begin{align*}
	\nabla\log f_\theta(x)&=\frac{\nabla f_\theta(x)}{f_\theta(x)}=w_\theta(x)(\mu_1-x)+\left(1-w_\theta(x)\right)(\mu_2-x)\\
	&=w_\theta(x)(\mu_1-x-\mu_2+x)-x+\mu_2\\
	&=(2w_\theta(x)-1)\frac{\mu_1-\mu_2}{2}-\frac{\mu_1+\mu_2}{2}-x\\
	&=(2w_\theta(x)-1)\mathtt{d}+\mathtt{m}-x,
\end{align*}
where $w_\theta(x):=\theta\phi(x-\mu_1)/f_\theta(x)$, $\mathtt{m}:=(\mu_1+\mu_2)/2$ and $\mathtt{d}:=(\mu_1-\mu_2)/2$. The score function $\partial_{x_1}\log q_{\theta,1}$ is simply a one-dimensional version of $\nabla\log f_\theta$. It can be computed by replacing $\phi$ with $\varphi$ and using the fact that $\mathtt{m}=0$ when $\mu_1=-\mu_2$. This gives
\begin{equation*}
	\partial_{x_1}\log q_{\theta,1}(x_1)=(2w_{\theta,1}(x_1)-1)(\Delta)-x_1,
\end{equation*}
with $w_{\theta,1}(x_1):=\theta\varphi(x_1-\Delta)/q_{\theta,1}(x_1)$.\\

\noindent
2. Plugging in the definition of $\phi$ into $w_\theta$ gives
\begin{align*}
	w_\theta(x)&=\frac{\theta\phi(x-\mu_1)}{\theta\phi(x-\mu_1)+(1-\theta)\phi(x-\mu_2)}\\
&=\frac{1}{1+\frac{1-\theta}{\theta}\exp\left(-\frac{\|x-\mu_2\|_2^2-\|x-\mu_1\|_2^2}{2}\right)}\\
&=\sigma\left(\langle x,\mu_1-\mu_2\rangle+\frac{1}{2}\left(\|\mu_2\|_2^2-\|\mu_1\|_2^2\right)+\log\left(\frac{\theta}{1-\theta}\right)\right),
\end{align*}
where $\sigma(z)=1/(1+\exp(-z))$ is the logistic sigmoid function. With analogous computations one obtains the corresponding representation for $w_{\theta,1}$.\\

\noindent
3a. Let $\theta,\theta^*\in(0,1)$ and $x\in\R^d$ arbitrary but fixed. To show (a) we use 1. and compute
\begin{equation*}
	\|\nabla\log f_\theta(x)-\nabla\log f_{\theta^*}(x)\|_2=\|2\mathtt{d}w_\theta(x)-2\mathtt{d}w_{\theta^*}(x)\|_2=2\|\mathtt{d}\|_2|w_\theta(x)-w_{\theta^*}(x)|.
\end{equation*}
Next, we get
\begin{align*}
	|w_\theta(x)-w_{\theta^*}(x)|&=\left|\frac{\theta\phi(x-\mu_1)}{f_\theta(x)}-\frac{\theta^*\phi(x-\mu_1)}{f_{\theta^*}(x)}\right|
	=\left|\frac{\theta\phi(x-\mu_1)f_{\theta^*}(x)-\theta^*\phi(x-\mu_1)f_\theta(x)}{f_\theta(x)f_{\theta^*}(x)}\right|\\
	&=\left|\frac{\phi(x-\mu_1)\{\theta f_{\theta^*}(x)-\theta^* f_\theta(x)\}}{f_\theta(x)f_{\theta^*}(x)}\right|
	=|\theta-\theta^*|\frac{\phi(x-\mu_1)\phi(x-\mu_2)}{f_\theta(x)f_{\theta^*}(x)},
\end{align*}
after introducing a common denominator and simplifying. Combining the two previous displays and using $\Delta=\|\mathtt{d}\|_2$ gives the desired result.\\

\noindent
3b. In order to prove (b), we insert the factorization $q_\theta(x)=q_{\theta,1}(x_1)\phi(x_{2:d})$ from 2. of Lemma~\ref{lemma:embedded_one_dim_sym_mixture} into the right-hand side of (a). This gives
\begin{equation*}
	\|\nabla\log q_\theta(x)-\nabla\log q_{\theta^*}(x)\|_2=2\Delta|\theta-\theta^*|\frac{\varphi(x_1-\Delta)\varphi(x_1+\Delta)}{q_{\theta,1}(x_1)q_{\theta^*,1}(x_1)},
\end{equation*}
after canceling the common factors $\phi(x_{2:d})$ in the numerator and denominator. With analogous computations as in (a) one can check that
\begin{equation*}
	|\partial_{x_1}\log q_{\theta,1}(x_1)-\partial_{x_1}\log q_{\theta^*,1}(x_1)|= 2\Delta|\theta-\theta^*|\frac{\varphi(x_1-\Delta)\varphi(x_1+\Delta)}{q_{\theta,1}(x_1)q_{\theta^*,1}(x_1)},
\end{equation*}
which concludes the proof of (b).\\

\noindent
3c. To prove (c), we will show $\FI(\Prob_\theta,\Prob_{\theta^*})=\FI(\mathbb{Q}_\theta,\mathbb{Q}_{\theta^*})=\FI(\mathbb{Q}_{\theta,1},\mathbb{Q}_{\theta^*,1})$. To that end, let $T(x)=Q(x-\mathtt{m})$ be the isometry characterized in Lemma~\ref{lemma:embedded_one_dim_sym_mixture}. With the transformation theorem for densities~\citep[Theorem 1.101 in][]{Kle2020} we obtain $q_\theta=f_\theta\circ T^{-1}$ and $q_{\theta^*}=f_{\theta^*}\circ T^{-1}$ since $|\det DT^{-1}(x)|=|\det Q^\top|=1$ due to the orthonormality of $Q$, with $q_\theta$ and $q_{\theta^*}$ being the densities of $\mathbb{Q}_{\theta}$ and $\mathbb{Q}_{\theta^*}$ respectively. Using $T^{-1}(x)=Q^\top x+\mathtt{m}$, as well as $\nabla\log q_\theta(x)=Q^\top\nabla\log f_\theta(T^{-1}(x))$ we can compute
	\begin{align*}
		\FI(\mathbb{Q}_\theta,\mathbb{Q}_{\theta^*})&=\int_{\R^d}\|\nabla\log q_\theta(x)-\nabla\log q_{\theta^*}(x)\|_2^2q_\theta(x)\dif x\\
		&=\int_{\R^d}\|Q^\top(\nabla\log f_\theta(T^{-1}(x))-\nabla\log f_{\theta^*}(T^{-1}(x))) \|_2^2f_\theta(T^{-1}(x))\dif x\\
		&=\int_{\R^d}\|\nabla\log f_\theta(z)-\nabla\log f_{\theta^*}(z) \|_2^2f_\theta(z)\dif z\\
		&=\FI(\Prob_\theta,\Prob_{\theta^*}),
	\end{align*}
	where we used the orthogonality of $Q$ and the change of variables $z=T^{-1}(x)$ in the penultimate equality. This proves the first equality. For the second equality, we use (b) and the factorization of $q_\theta$ to get
	\begin{align*}
		\FI(\mathbb{Q}_\theta,\mathbb{Q}_{\theta^*})&=\int_{\R^d}\|\nabla\log q_\theta(x)-\nabla\log q_{\theta^*}(x)\|_2^2q_\theta(x)\dif x\\
		&=\int_{\R^{d-1}}\int_\R|\partial_{x_1}\log q_{\theta,1}(x_1)-\partial_{x_1}\log q_{\theta^*,1}(x_1)|^2q_{\theta,1}(x_1)\phi(x_{2:d})\dif x_1\dif x_{2:d}\\
		&=\FI(\mathbb{Q}_{\theta,1},\mathbb{Q}_{\theta^*,1})\int_{\R^{d-1}}\phi(x_{2:d})\dif x_{2:d}\\
		&=\FI(\mathbb{Q}_{\theta,1},\mathbb{Q}_{\theta^*,1}),
	\end{align*}
	where we used Fubini's theorem in the second equality and the fact that $\int_{\R^{d-1}}\phi(x)\dif x=1$ in the last equality. This proves the second equality and concludes the proof of (c).\\

\noindent
4. For ease of notation we give the proof for
\begin{equation*}
	f_\theta(x)=\theta\varphi(x-\mu)+(1-\theta)\varphi(x+\mu),
\end{equation*}
for some $\mu>0$ and $\theta\in(0,1)$, such that $q_{\theta,1}(x_1)=f_\theta(x)$ for $x_1=x$ and $\mu=\Delta$. Then, it holds that
	\begin{align*}
		\frac{\varphi(x-\mu)\varphi(x+\mu)}{f_\theta(x)f_{\theta^*}(x)}&=\frac{\varphi(x-\mu)\varphi(x+\mu)}{(\theta\varphi(x-\mu)+(1-\theta)\varphi(x+\mu))(\theta^*\varphi(x-\mu)+(1-\theta^*)\varphi(x+\mu))}\\
		&=\frac{1}{\theta\theta^*\exp(2\mu x)+\{\theta(1-\theta^*)+\theta^*(1-\theta)\}+(1-\theta)(1-\theta^*)\exp(-2\mu x)}\\
		&\leqslant \frac{1}{\{\theta\theta^*\wedge(1-\theta)(1-\theta^*)\}\{\exp(2\mu|x|)\bm{1}(x\geqslant 0)+\exp(-2\mu|x|)\bm{1}(x<0)\}}\\
		&=\frac{\exp(-2\mu|x|)}{\theta\theta^*\wedge(1-\theta)(1-\theta^*)},
	\end{align*}
which gives the right inequality with $m:=\theta\theta^*\wedge(1-\theta)(1-\theta^*)$.

For the lower bound, note that for any $\theta\in(0,1)$ and any arbitrary but fixed $x\in\R$ it holds that
\begin{equation*}
    \varphi(x-\mu)\wedge\varphi(x+\mu)\leqslant f_\theta(x)\leqslant \varphi(x-\mu)\vee\varphi(x+\mu),
\end{equation*}
because $f_\theta(x)$ is a convex combination of $\varphi(x-\mu)$ and $\varphi(x+\mu)$. This yields
\begin{equation*}
    \frac{\varphi(x-\mu)\varphi(x+\mu)}{f_\theta(x)f_{\theta^*}(x)}\geqslant\frac{\varphi(x-\mu)\varphi(x+\mu)}{\{\varphi(x-\mu)\vee \varphi(x+\mu)\}^2}=\frac{\varphi(x-\mu)\wedge\varphi(x+\mu)}{\varphi(x-\mu)\vee \varphi(x+\mu)}\leqslant1.
\end{equation*}
From the two possible cases
\begin{equation*}
    \frac{\varphi(x-\mu)}{\varphi(x+\mu)}=\exp(2\mu x)\geqslant\exp(-2\mu |x|)\quad\text{ and } \quad\frac{\varphi(x+\mu)}{\varphi(x-\mu)}=\exp(-2\mu x)\geqslant\exp(-2\mu |x|)
\end{equation*}
we arrive at
\begin{equation*}
    \frac{\varphi(x-\mu)\varphi(x+\mu)}{f_\theta(x)f_{\theta^*}(x)}\geqslant\exp(-2\mu |x|),
\end{equation*}
which concludes the proof.\\

\noindent
5. Again, for ease of notation, we set $\Prob_\theta:=\theta\mathcal{N}(\mu,1)+(1-\theta)\mathcal{N}(-\mu,1)$. Given an even map $f:\R\to\R$, that is, $f(-x)=f(x)$ for all $x\in\R$, and $\Prob_\mu=\mathcal{N}(\mu,1)$ it holds that $\E_{\Prob_\mu}[f(X)]=\E_{\Prob_{-\mu}}[f(X)]$. Indeed, the substitution $x=-y$ gives
\begin{align*}
    \E_{\Prob_\mu}[f(X)]&=\int_{-\infty}^{\infty} f(x)\varphi(x-\mu)\dif x=\int_{\infty}^{-\infty} -f(-y)\varphi(-y-\mu)\dif y.
\end{align*}
Then, the fundamental theorem of calculus gives us
\begin{equation*}
	\int_a^b f(x)\dif x=-\int_b^a f(x)\dif x,
\end{equation*}
and togethwer with the eveness of both $f$ and $\varphi$ we obtain
\begin{equation*}
	\int_{\infty}^{-\infty} -f(-y)\varphi(-y-\mu)\dif y=\int_{-\infty}^{\infty} f(-y)\varphi(y+\mu)\dif y=\int_{-\infty}^{\infty} f(y)\varphi(y+\mu)\dif y=\E_{\Prob_{-\mu}}[f(X)].
\end{equation*}
From this follows
\begin{align*}
	\E_{\Prob_\theta}[f(X)]&=\theta\E_{\Prob_\mu}[f(X)]+(1-\theta)\E_{\Prob_{-\mu}}[f(X)]\\
	&=\theta\E_{\Prob_\mu}[f(X)]+(1-\theta)\E_{\Prob_{\mu}}[f(X)]\\
	&=\E_{\Prob_\mu}[f(X)].
\end{align*}
With $f(x)=\exp(-t|x|)$, the fact that $t\geqslant \mu$ by assumption and Lemma~\ref{lemma:Gaussian_exp_bound} we obtain
\begin{equation*}
	\E_{\Prob_\theta}[\exp(-t|X|)]=\E_{\Prob_\mu}[\exp(-t|X|)]\leqslant \left\{\sqrt{\frac{\pi}{2}}\wedge\left(\frac{1}{t-\mu}+\frac{1}{t+\mu}\right)\right\}\varphi(\mu),
\end{equation*}
	which shows the claim after plugging in $\mu=\Delta$.\\
	
	\noindent
	6. With the factorization $q_\theta(x)=q_{\theta,1}(x_1)\phi(x_{2:d})$ from 2. of Lemma~\ref{lemma:embedded_one_dim_sym_mixture}, we can compute
	\begin{align*}
		\nabla\log q_\theta(x)=\nabla_x\log q_{\theta,1}(x_1)+\nabla_x\log \phi(x_{2:d})
		=\begin{bmatrix}
\partial_{x_1}\log q_{\theta,1}(x_1)\\
0\\
\vdots\\
0
		\end{bmatrix}
		+\begin{bmatrix}
0\\
-x_2\\
\vdots\\
-x_d
		\end{bmatrix}
	\end{align*}
	and $\divergence\nabla\log q_\theta(x)=\partial_{x_1x_1}\log q_{\theta,1}(x_1)-d$. In total, this gives us
	\begin{align*}
		m_{\SM}^{q_\theta}(x)&=\|\nabla\log q_\theta(x)\|_2^2+2\divergence\nabla\log q_\theta(x)\\
		&=\left(\partial_{x_1}\log q_{\theta,1}(x_1)\right)^2+\|x_{2:d}\|_2^2+2(\partial_{x_1x_1}\log q_{\theta,1}(x_1)-d)\\
		&=m_{\SM}^{q_{\theta,1}}(x_1)+\mathtt{r}(x_{2:d}),
	\end{align*}
	with $\mathtt{r}(x_{2:d})=\|x_{2:d}\|_2^2-2d$.\qed

	\subsection{Proof of Lemma~\ref{lemma:ddpm_density_and_score}}\label{appendix:proof_lemma_ddpm_density}

	From $X_0\sim\Prob_\theta$ and $Z_t\sim\mathcal{N}(0,\mathbb{I}_d)$ we can conclude that
	\begin{align*}
		\exp(-t)X_0&\sim\theta\mathcal{N}(\exp(-t)\mu_1, \exp(-2t)\mathbb{I}_d)+(1-\theta)\mathcal{N}(-\exp(-t)\mu_2, \exp(-2t)\mathbb{I}_d),\\
		\sqrt{1-\exp(-2t)}Z_t&\sim\mathcal{N}(0,1-\exp(-2t)\mathbb{I}_d).
	\end{align*}
	It is well known that Gaussians are closed under convolutions, that is, we have
	\begin{equation*}
		\mathcal{N}(\mu_1,\Sigma_1)*\mathcal{N}(\mu_2,\Sigma_2)=\mathcal{N}(\mu_1+\mu_2,\Sigma_1+\Sigma_2),
	\end{equation*}
	for any $\mu_1,\mu_2\in\R^d$ and $\Sigma_1,\Sigma_2\in\R^{d\times d}$ positive definite. As a consequence, for
	\begin{equation*}
		X_t=\exp(-t)X_0+\sqrt{1-\exp(-2t)}Z_t,
	\end{equation*}
	we get
	\begin{equation*}
		X_t\sim\theta\mathcal{N}(\exp(-t)\mu_1, \mathbb{I}_d)+(1-\theta)\mathcal{N}(\exp(-t)\mu_2, \mathbb{I}_d).
	\end{equation*}
	The density of $X_t$ is therefore given by
	\begin{equation*}
		f_{\theta,t}(x)=\theta\phi(x-\mu_{1,t})+(1-\theta)\phi(x-\mu_{2,t}),
	\end{equation*}
	with $\mu_{1,t}=\exp(-t)\mu_1$ and $\mu_{2,t}=\exp(-t)\mu_2$, as desired.\qed

	\subsection{Proof of Lemma~\ref{lemma:reduction_constants}}\label{app:proof_lemma_reduction_constants}

	1. After noting that $\mathtt{Lip}(m_{\SM}^\Prob)(x)=\mathtt{Lip}(m_{\SM}^{\mathbb{Q}})(T(x))$ for every $x\in\R^d$, we can compute
	\begin{align*}
		\|\mathtt{Lip}(m_{\SM}^\Prob)\|_{L^2(\Prob_{\theta_0})}^2
		&=\int_{\R^d}|\mathtt{Lip}(m_{\SM}^{\mathbb{Q}})(T(x))|^2f_{\theta_0}(x)\dif x
		=\int_{\R^d}|\mathtt{Lip}(m_{\SM}^{\mathbb{Q}})(y)|^2f_{\theta_0}(T^{-1}(y))\dif y\\
		&=\int_{\R^d}|\mathtt{Lip}(m_{\SM}^{\mathbb{Q}})(y)|^2q_{\theta_0}(y)\dif y
		=\|\mathtt{Lip}(m_{\SM}^{\mathbb{Q}})\|_{L^2(\mathbb{Q}_{\theta_0})}^2,
	\end{align*}
	by using the change of variables $x=T^{-1}(y)$ and the fact that $q_{\theta_0}=f_{\theta_0}\circ T^{-1}$. Then, from 6. of Lemma~\ref{lemma:gm_properties} it follows that $\partial_\theta m_{\SM}^{\mathbb{Q}}(\theta,x)=\partial_\theta m_{\SM}^{\mathbb{Q}_1}(\theta,x_1)$ for every $x=(x_1,x_{2:d})\in\R^d$, and together with $q_\theta(x)=q_{\theta,1}(x_1)\phi(x_{2:d})$ from Lemma~\ref{lemma:embedded_one_dim_sym_mixture} we can compute
	\begin{equation*}
		\|\mathtt{Lip}(m_{\SM}^{\mathbb{Q}})\|_{L^2(\mathbb{Q}_{\theta_0})}^2=\int_{\R^d}|\mathtt{Lip}(m_{\SM}^{\mathbb{Q}_1})(y_1)|^2q_{\theta_0,1}(y_1)\phi(x_{2:d})\dif y_1\dif x_{2:d}=\|\mathtt{Lip}(m_{\SM}^{\mathbb{Q}_1})\|_{L^2(\mathbb{Q}_{\theta_0,1})}^2.
	\end{equation*}
	Combining the two equalities yields the claim for SM. The above proof extends mutatis mutandis to the DDSM case, after replacing $T$ with the isometry $T_t$ given by $T_t(x):=Q(x-\exp(-t)\mathtt{m})$.\\

	\noindent
	2. From~\eqref{eq:curvature_constant} we know that
	\begin{equation}
		\label{eq:curvature_proper_scoring}
		\mathtt{Curv}(m_{\SM}^\Prob)=\inf_{|\theta-\theta_0| \geqslant \delta}\left\{4\Delta^2\int_\R\left(\frac{\phi(x-\mu_1)\phi(x+\mu_2)}{f_\theta(x)f_{\theta_0}(x)}\right)^2f_{\theta_0}(x)\dif x\right\}.
	\end{equation}
	With 3c. of Lemma~\ref{lemma:gm_properties} we have
	\begin{align*}
		4\Delta^2\int_\R\left(\frac{\phi(x-\mu_1)\phi(x+\mu_2)}{f_\theta(x)f_{\theta_0}(x)}\right)^2f_{\theta_0}(x)\dif x &=\frac{\FI(\Prob_\theta,\Prob_{\theta_0})}{|\theta-\theta_0|^2}\\
		&=\frac{\FI(\mathbb{Q}_{\theta,1},\mathbb{Q}_{\theta_0,1})}{|\theta-\theta_0|^2}\\
		&=4\Delta^2\int_\R\left(\frac{\varphi(x_1-\Delta)\phi(x_1+\Delta)}{q_{\theta,1}(x_1)f_{\theta_0,1}(x_1)}\right)^2f_{\theta_0}(x_1)\dif x_1,
	\end{align*}
	which gives the claim after inserting the above equality into~\eqref{eq:curvature_proper_scoring}. Once again, the proof extends mutatis mutandis to the DDSM case.\qed
	
    \subsection{Proof of Lemma~\ref{lemma:lipschitz_curvature_bounds_sm}}\label{app:proof_lemma_lipschitz_curvature_bounds_sm}

    1. A function is Lipschitz continuous with Lipschitz constant $\mathtt{L}$ if and only if its derivative is bounded by $\mathtt{L}$. Therefore, we have
	\begin{equation}
		\label{eq:lipschitz_constant_definition}
		\mathtt{L}_{\SM}(x)=\sup_{\theta\in\Theta}\left|\partial_\theta m_{\SM}(\theta,x)\right|.
	\end{equation}
	In view of 1. of Lemma~\ref{lemma:reduction_constants} it suffices to prove a lower bound for $\|\mathtt{L}_{\SM}(X)\|_{L^2(\Prob_{\theta_0})}$ when
	\begin{equation}
		\label{eq:symmetric_one_d_family}
		\Prob_\theta=\theta\mathcal{N}(\mu,1)+(1-\theta)\mathcal{N}(-\mu,1).
	\end{equation}
	Replacing $\mu$ in the obtained lower bound by $\Delta$ then gives the desired result for the $d$-dimensional family in~\eqref{eq:gm}.
	
	With $s_\theta(x):=\partial_x\log f_\theta(x)$ for $x\in\mathbb{R}$, the derivative on the right-hand side of~\eqref{eq:lipschitz_constant_definition} can be computed as follows:
	\begin{equation}
		\label{eq:derivative_m_SM}
		\partial_\theta m_{\SM}(\theta,x)=\partial_\theta\{s_\theta^2(x)+\partial_x s_\theta(x)\}=2s_\theta(x)\partial_\theta s_\theta(x)+2\partial_\theta\partial_x s_\theta(x)
	\end{equation}
	In order to simplify the right-hand side of~\eqref{eq:derivative_m_SM}, we use the identity $s_\theta(x)=(2w_\theta(x)-1)\mu-x$ given in 1. of Lemma~\ref{lemma:gm_properties} and the derivatives
	\begin{align*}
		\partial_\theta w_\theta(x)&=w_\theta(x)(1-w_\theta(x))/(\theta(1-\theta)),\\
		\partial_x w_\theta(x)&=2\mu w_\theta(x)(1-w_\theta(x)),\\
		\partial_\theta\partial_x w_\theta(x)&=2\mu \partial_\theta w_\theta(x)(1-2w_\theta(x)),
	\end{align*}
	where in particular $\partial_xw_\theta$ is based on the fact that $w_\theta$ is a logistic function, as shown in 2. of Lemma~\ref{lemma:gm_properties}. With these identities and $\partial_x s_\theta(x)=2\mu \partial_x w_\theta(x)-1$, we can compute
	\begin{align*}
		\partial_\theta s_\theta(x)&=2\mu \partial_\theta w_\theta(x),\\
		\partial_\theta\partial_x s_\theta(x)&=2\mu\partial_\theta\partial_x w_\theta(x)=4\mu^2 \partial_\theta w_\theta(x)(1-2w_\theta(x)).
	\end{align*}
	Plugging this into the right-hand side of~\eqref{eq:derivative_m_SM} yields
	\begin{equation*}
		\partial_\theta m_{\SM}(\theta,x)=4\mu\partial_\theta w_\theta(x)\{s_\theta(x)+2\mu(1-2w_\theta(x))\}=\partial_\theta w_\theta(x)\{4\mu^2-4x\mu-8\mu^2w_\theta(x)\}.
	\end{equation*}
	Having obtained an expression for $\partial_\theta m_{\SM}$ we can plug this into~\eqref{eq:lipschitz_constant_definition}, in order to compute the norm. In doing this, we observe that
	\begin{align}
		\nonumber
		\|\mathtt{L}_{\SM}(X)\|_{L^2(\Probtrue)}^2&=\int_\R\sup_{\theta\in\Theta}\left|\partial_\theta m_{\SM}(\theta,x)\right|^2f_{\theta_0}(x)\dif x\\
		\nonumber
		&=\int_\R\left(\sup_{\theta\in\Theta}\frac{1}{\theta(1-\theta)}w_\theta(x)(1-w_\theta(x))|4\mu^2-4x\mu-8\mu^2w_\theta(x)|\right)^2f_{\theta_0}(x)\dif x\\
		\label{eq:integral_lipschitz_constant}
		&\geqslant\int_\R\left(\frac{1}{\theta_*(1-\theta_*)}w_{\theta_*}(x)(1-w_{\theta_*}(x))|4\mu^2-4x\mu-8\mu^2w_{\theta_*}(x)|\right)^2f_{\theta_0}(x)\dif x
	\end{align}
	holds for every $\theta_*\in\Theta=[\eta,1-\eta]$, since the supremum of the integrand can only be larger. Note that the first equality in the above follows from the replacement of $\Probtrue$ by $\Prob_{\theta_0}$, which is justified by Assumption~\ref{assumption:probtrue}.

	\begin{figure}[!tb]
	\centering

\definecolor{denscol}{RGB}{0,128,128}   
\definecolor{unccol}{RGB}{230,126,34}   
\definecolor{sigcol}{RGB}{142,68,173}   
\definecolor{intcol}{RGB}{41,128,185}   
\definecolor{neonred}{RGB}{255,40,90}

\begin{tikzpicture}
\begin{axis}[
  width=15.2cm,height=7.2cm,
  xmin=-6.5, xmax=6.5,
  ymin=0, ymax=1.25,
  axis lines=left,
  axis line style={-stealth},
  xtick=\empty, ytick=\empty,
  xlabel={$x$},
  xlabel style={at={(axis description cs:1,0)}, anchor=west},
  clip=false,
  legend columns=2,
  legend style={
    draw=none, fill=none,
    font=\small,
    at={(0.02,0.985)}, anchor=north west,
    /tikz/every even column/.append style={column sep=6pt}
  },
]

\def\beps{0.4}
\def\win{0.55}
\def\muplot{3.7}
\def\sigma{0.75}

\addplot[very thick, color=denscol, domain=-6.5:6.5, samples=900]
  {0.62*exp(-((x-\muplot)^2)/(2*\sigma*\sigma))
   +0.38*exp(-((x+\muplot)^2)/(2*\sigma*\sigma))};

\addplot[very thick, color=unccol, domain=-6.5:6.5, samples=1100]
  {0.55*(1/(cosh(3*(x-\beps))^2))};

\addplot[very thick, color=sigcol, domain=-6.5:6.5, samples=1100]
  {0.05 + 0.75*(1/(1+((x-\beps)/0.65)^4))};

\addplot[ultra thick, color=intcol, domain=-6.5:6.5, samples=1400]
  {0.02
   + 0.85*
     (0.62*exp(-((x-\muplot)^2)/(2*\sigma*\sigma))
      +0.38*exp(-((x+\muplot)^2)/(2*\sigma*\sigma)))
     * (0.55*(1/(cosh(3*(x-\beps))^2)))
     * (0.05 + 0.75*(1/(1+((x-\beps)/0.65)^4)))};

\addlegendimage{very thick, color=denscol}
\addlegendentry{$f_{\theta}(x)$}
\addlegendimage{very thick, color=unccol}
\addlegendentry{$w_{\theta}(x)(1-w_{\theta}(x))$}
\addlegendimage{very thick, color=sigcol}
\addlegendentry{$\bigl|4\mu^2-4\mu x-8\mu^2 w_{\theta}(x)\bigr|^2$}
\addlegendimage{ultra thick, color=intcol}
\addlegendentry{integrand (schematic)}

\draw[black] (axis cs:0,0) -- (axis cs:0,-0.02);
\node[font=\scriptsize, anchor=north] at (axis cs:0,-0.02) {$0$};

\draw[black] (axis cs:\beps,0) -- (axis cs:\beps,-0.02);
\node[font=\scriptsize, anchor=north] at (axis cs:\beps,-0.02) {$b_{\theta,\mu}$};

\draw[black] (axis cs:-\muplot,0) -- (axis cs:-\muplot,-0.02);
\node[font=\scriptsize, anchor=north] at (axis cs:-\muplot,-0.02) {$-\mu$};

\draw[black] (axis cs:\muplot,0) -- (axis cs:\muplot,-0.02);
\node[font=\scriptsize, anchor=north] at (axis cs:\muplot,-0.02) {$\mu$};

\draw[densely dashed, gray!70] (axis cs:\beps,0) -- (axis cs:\beps,1.22);

\draw[neonred!90!black, thick, decorate,
      decoration={brace, amplitude=4pt, mirror}]
  (axis cs:-\win, -0.13) -- (axis cs:\win, -0.13);

\node[font=\scriptsize, text=neonred!90!black, anchor=north]
  at (axis cs:0,-0.18) {$|x|\lesssim 1/\mu$};

\end{axis}
\end{tikzpicture}
	\caption{Idea of the proof of 1. of Lemma~\ref{lemma:lipschitz_curvature_bounds_sm}. Integration can be restricted to a small compact set around the origin, where some of the factors of the integrand take on their largest values. This way, one obtains the largest possible polynomial in $\mu$, while the density $f_{\theta}$ always contributes a term with exponential decay in $\mu^2$.}
	\label{fig:proof_idea}
\end{figure}

In order to proceed, we will restrict integration to a small compact set around the origin. The motivation for this is as follows: On an interval of width proportional to $1/\mu$ around the origin, the functions $f_\theta$, $g_\theta$ and $h_\theta$ in the integrand fulfill
	\begin{align*}
		f_\theta(x)&\gtrsim\exp(-\mu^2/2),\\
		g_\theta(x)&:=\left(\frac{w_\theta(x)(1-w_\theta(x))}{\theta(1-\theta)}\right)^2\gtrsim 1,\\
		h_\theta(x)&:=|4\mu^2-4x\mu-8\mu^2w_\theta(x)|^2\gtrsim \mu^4.
	\end{align*}
As a consequence, we obtain
	\begin{equation*}
		\int_\R g_\theta(x)h_\theta(x)f_{\theta_0}(x)\dif x\geqslant \int_\mathcal{A} g_\theta(x)h_\theta(x)f_{\theta_0}(x)\dif x\gtrsim \mu^3 \exp(-\mu^2/2),
	\end{equation*}
	which is the desired lower bound. This underlying idea is depicted schematically in Figure~\ref{fig:proof_idea}. But in order to formalize this, we need to distinguish two cases.\\
	
	\noindent
	\textbf{Case 1: $\mu\geqslant1$.} Consider the interval
	\begin{equation}
		\label{eq:interval_A}
		\mathcal{A}:=\left[-\frac{\alpha}{\mu},\frac{\alpha}{\mu}\right]\quad \text{with}\quad \alpha:=\frac{\eta}{4}\wedge\frac{1-2\eta}{4},
	\end{equation}
	and fix $\theta_*=\eta$. Let us deal with $g_\theta$ first. The map $x\mapsto w_\theta(x)$ is $\mu/2$-Lipschitz, since
	\begin{equation*}
		\partial_x w_\theta(x)=2\mu w_\theta(x)(1-w_\theta(x))\leqslant \frac{2\mu}{4}=\frac{\mu}{2},
	\end{equation*}
	where we used the fact that $w_\theta(x)\in[0,1]$ for all $x\in\R$ and thus $w_\theta(x)(1-w_\theta(x))\in[0,1/4]$. Hence, using $w_\eta(0)=\eta$ and the above Lipschitz property, we get for every $x\in\mathcal{A}$, that
	\begin{equation}
		\label{eq:lipschitz_bound_w}
		|w_\eta(x)-w_\eta(0)|=|w_\eta(x)-\eta|\leqslant \frac{\mu}{2}|x|\leqslant \frac{\mu}{2}\cdot\frac{\alpha}{\mu}=\frac{\alpha}{2}\leqslant\frac{\eta}{8},
	\end{equation}
	or put differently, that
	\begin{equation}
		\label{eq:interval_w}
		w_\eta(x)\in\left[\frac{7\eta}{8},\frac{9\eta}{8}\right]\subseteq(0,1),
	\end{equation}
	for every $x\in\mathcal{A}$. Then, we define
	\begin{equation*}
		\mathtt{C}_w(\eta):=\inf_{x\in\mathcal{A}}\left\{\left(\frac{w_\eta(x)(1-w_\eta(x))}{\eta(1-\eta)}\right)^2\right\},
	\end{equation*}
	and with the continuity of the map $x\mapsto w_\eta(x)(1-w_\eta(x))$, the compactness of $\mathcal{A}$ and~\eqref{eq:interval_w} conclude that $\mathtt{C}_w(\eta)>0$ is finite, such that for every $x\in\mathcal{A}$ it holds that
	\begin{equation*}
		g_\eta(x)\geqslant \mathtt{C}_w(\eta).
	\end{equation*}

	Next, we deal with $h_\theta$. From~\eqref{eq:lipschitz_bound_w} follows that $w_\eta(x)\leqslant \eta+\alpha/2$ for all $x\in\mathcal{A}$, such that
	\begin{equation}
		\label{eq:lower_bound_h}
		4-8w_\eta(x)\geqslant 4-8\left(\eta+\frac{\alpha}{2}\right)=4(1-2\eta)-4\alpha\geqslant 3(1-2\eta),
	\end{equation}
	since $\alpha\leqslant(1-2\eta)/4$ by definition given in~\eqref{eq:interval_A}. Moreover, we have $|4x\mu|\leqslant 4\alpha$ for all $x\in\mathcal{A}$, such that together with~\eqref{eq:lower_bound_h} and the reverse triangle inequality we obtain for every $x\in\mathcal{A}$ that
	\begin{align*}
		|4\mu^2-4x\mu-8\mu^2w_\eta(x)|&=|\mu^2\left(4-8w_\eta(x)\right)-4x\mu|\\
		&\geqslant \mu^2|4-8w_\eta(x)|-|4x\mu|\\
		&\geqslant \mu^2 3(1-2\eta)-4\alpha.
	\end{align*}
	Lastly, using $\alpha\leqslant(1-2\eta)/4$ once more and the assumption that $\mu\geqslant1$ we can conclude that $\mu^2 3(1-2\eta)-4\alpha\geqslant \mu^22(1-2\eta)$, such that in total we have shown
	\begin{equation*}
		h_\eta(x)=|4\mu^2-4x\mu-8\mu^2w_\eta(x)|^2\geqslant \mu^4\mathtt{C}_h(\eta),
	\end{equation*}
	for every $x\in\mathcal{A}$, where $\mathtt{C}_h(\eta):=4(1-2\eta)^2$.
	
	Finally, we need a lower bound for the density $f_{\theta_0}$ on $\mathcal{A}$. It is easy to see that
	\begin{align}
		\nonumber
		f_{\theta_0}(x)&\geqslant \theta_0\varphi(x-\mu)\\
		\nonumber
		&\geqslant\frac{\eta}{\sqrt{2\pi}}\exp\left(-\frac{(-\alpha/\mu-\mu)^2}{2}\right)\\
		\nonumber
		&=\frac{\eta}{\sqrt{2\pi}}\exp\left(-\alpha-\frac{\alpha^2}{2\mu^2}\right)\exp\left(-\frac{\mu^2}{2}\right)\\
		\label{eq:lower_bound_f}
		&\geqslant\mathtt{C}_f(\eta)\exp\left(-\frac{\mu^2}{2}\right),
	\end{align}
	for all $x\in\mathcal{A}$. In the above, the constant $\mathtt{C}_f(\eta)$ is defined as
	\begin{equation*}
		\mathtt{C}_f(\eta):=\frac{\eta}{\sqrt{2\pi}}\exp\left(-\alpha-\frac{\alpha^2}{2}\right),
	\end{equation*}
	and the inequality in~\eqref{eq:lower_bound_f} follows from $\mu\geqslant1$.
	
	Putting everything together, that is, choosing $\theta_*=\eta$ on the right-hand side of~\eqref{eq:integral_lipschitz_constant} and using the lower bounds for $g_{\eta}$, $h_{\eta}$, and $f_{\theta_0}$ from above, we obtain
	\begin{align*}
		\|\mathtt{L}_{\SM}(X)\|_{L^2(\Probtrue)}^2
		&\geqslant \int_\mathcal{A}\left(\frac{1}{\eta(1-\eta)}w_{\eta}(x)(1-w_{\eta}(x))|4\mu^2-4x\mu-8\mu^2w_{\eta}(x)|\right)^2f_{\theta_0}(x)\dif x\\
		&\geqslant \mu^4\exp\left(-\frac{\mu^2}{2}\right)\mathtt{C}_w(\eta)\mathtt{C}_h(\eta)\mathtt{C}_f(\eta)\gamma(\mathcal{A}).
	\end{align*}
	With the simple fact that $\gamma(\mathcal{A})=2\alpha/\mu$ we conclude that
	\begin{equation}
		\label{eq:final_lower_bound}
		\|\mathtt{L}_{\SM}(X)\|_{L^2(\Probtrue)}^2\geqslant \frac{2\alpha}{\mu} \mu^4\exp\left(-\frac{\mu^2}{2}\right)\mathtt{C}_w(\eta)\mathtt{C}_h(\eta)\mathtt{C}_f(\eta)\gtrsim \mu^3\exp\left(-\frac{\mu^2}{2}\right),
	\end{equation}
	which was the desired result.
	\\

	\noindent
	\textbf{Case 2: $\mu<1$.} In order to obtain~\eqref{eq:final_lower_bound} we needed to assume that $\mu\geqslant1$. When $\mu<1$, let us fix $\theta_*=1/2$ and consider the set $\mathcal{A}:=[-2,-1]$. Then, we have
	\begin{align*}
		\frac{g_{1/2}(x)}{(1/2)(1-1/2)}&=4w_{1/2}(x)(1-w_{1/2}(x))\\
		&=\frac{4\varphi(x-\mu)\varphi(x+\mu)}{(\varphi(x-\mu)+\varphi(x+\mu))^2}\\
		&=\frac{4\exp(-2x\mu)}{(1+\exp(-2\mu x))^2}\\
		&\geqslant \frac{4\exp(4)}{(1+\exp(4))^2},
	\end{align*}
	for all $x\in\mathcal{A}$, since $\mu<1$ and $x\in[-2,-1]$ implies that $-2\mu x\in[2,4]$. Moreover, we have $w_{1/2}(x)\leqslant1/2$ for all $x\leqslant0$, thus in particular for all $x\in\mathcal{A}$, and therefore
	\begin{equation*}
		4\mu^2-8\mu^2w_{1/2}(x)=4\mu^2(1-2w_{1/2}(x))\geqslant 0.
	\end{equation*}
	Together with $-4\mu x=4\mu|x|\geqslant4\mu$ for $x\in\mathcal{A}$, this implies that
	\begin{equation*}
		|4\mu^2-4x\mu-8\mu^2w_{1/2}(x)|
		=4\mu^2(1-2w_{1/2}(x))-4\mu x
		\geqslant 4\mu\,
	\end{equation*}
	for all $x\in\mathcal{A}$, thus $h_{1/2}(x)\geqslant 16 \mu^2$. Lastly, we have
	\begin{equation*}
		f_{\theta_0}(x)\geqslant \eta\varphi(x-\mu)\geqslant\eta\varphi(2),
	\end{equation*}
	since $\mu<1$ and $x\in[-2,-1]$ implies that $x-\mu\leqslant -2$. Putting everything together with the same logic as before, we obtain
	\begin{equation*}
		\|\mathtt{L}_{\SM}(X)\|_{L^2(\Probtrue)}^2\gtrsim \mu^2,
	\end{equation*}
	whenever $\mu<1$. But clearly, for $\mu<1$ it holds that
	\begin{equation*}
		\mu^2\geqslant\mu^3\exp\left(-\frac{\mu^2}{2}\right).
	\end{equation*}
	Therefore, the lower bound in~\eqref{eq:final_lower_bound} for $\mu\geqslant1$ also holds for $\mu<1$. Taking the square root and substituting $\mu=\Delta$ yields the lower bound for $\|\mathtt{L}_{\SM}(X)\|_{L^2(\Probtrue)}$ when considering the family in~\eqref{eq:gm}, which concludes the proof.\\

    \noindent
    2. Every proper scoring rule induces a divergence and in the case of vanilla SM this gives the following~\citep{ForLau2015}:
    \begin{equation}
    \label{eq:proper_scoring_rule_fact}
        \E_{\theta_0}\left[m_{\SM}(\theta,X)-m_{\SM}(\theta_0,X)\right]=\FI(\Prob_{\theta_0},\Prob_\theta)
    \end{equation}
	Indeed, replacing $\Prob$ in~\eqref{eq:fisher_distance_3} with $\Prob_{\theta_0}$ and $\mathbb{Q}$ with $\Prob_\theta$ gives
	\begin{equation}
		\label{eq:proper_scoring_rule}
		\FI(\Prob_{\theta_0},\Prob_\theta)=\E_{\theta_0}[m_{\SM}(\theta,X)]+\E_{\theta_0}\left[\|\nabla\log f_{\theta_0}(X)\|_2^2\right],
	\end{equation}
	which holds for every $\theta\in\Theta$. Rearranging and inserting~\eqref{eq:proper_scoring_rule} into the left-hand side of~\eqref{eq:proper_scoring_rule_fact} yields
	\begin{align*}
		&\E_{\theta_0}\left[m_{\SM}(\theta,X)-m_{\SM}(\theta_0,X)\right]\\
		=&\FI(\Prob_{\theta_0},\Prob_\theta)-\E_{\theta_0}\left[\|\nabla\log f_{\theta_0}(X)\|_2^2\right]-\left(\FI(\Prob_{\theta_0},\Prob_{\theta_0})-\E_{\theta_0}\left[\|\nabla\log f_{\theta_0}(X)\|_2^2\right]\right)\\
		=&\FI(\Prob_{\theta_0},\Prob_\theta),
	\end{align*}
	confirming the equality in~\eqref{eq:proper_scoring_rule_fact}. Analogously to 1., it is sufficient to consider families of the form~\eqref{eq:symmetric_one_d_family} in view of 2. of Lemma~\ref{lemma:reduction_constants}. Then, 3. of Lemma~\ref{lemma:gm_properties} implies
    \begin{equation*}
        \E_{\theta_0}\left[m_{\SM}(\theta,X)-m_{\SM}(\theta_0,X)\right]=4\mu^2(\theta-\theta_0)^2\int_\R\left(\frac{\varphi(x-\mu)\varphi(x+\mu)}{f_\theta(x)f_{\theta_0}(x)}\right)^2f_{\theta_0}(x)\dif x,
    \end{equation*}
    such that with
    \begin{equation}
		\label{eq:curvature_constant}
        \mathtt{C}_{\SM}:=\inf_{|\theta-\theta_0| \geqslant \delta}\left\{4\mu^2\int_\R\left(\frac{\varphi(x-\mu)\varphi(x+\mu)}{f_\theta(x)f_{\theta_0}(x)}\right)^2f_{\theta_0}(x)\dif x\right\}
    \end{equation}
    we have
    \begin{equation}
		\label{eq:curvature_condition_sm}
        \inf_{|\theta-\theta_0| \geqslant \delta} \E_{\theta_0}\left[m_{\SM}(\theta,X)-m_{\SM}(\theta_0,X)\right]\geqslant \mathtt{C}_{\SM}\delta^2.
    \end{equation}
	From~\eqref{eq:curvature_condition_sm} together with Assumption~\ref{assumption:probtrue} it follows that the curvature condition stated in Definition~\ref{definition:local_curvature_condition} holds with $\alpha=2$ and constant $\mathtt{C}_{\SM}$. With 5. of Lemma~\ref{lemma:gm_properties} we furthermore obtain
    \begin{equation*}
        \mathtt{C}_{\SM}\leqslant\frac{4\mu^2}{m^2}\E_{\theta_0}\left[\exp(-4\mu|X|)\right]\leqslant\frac{4\mu^2}{m^2\sqrt{2\pi}}\frac{8}{15\mu}\exp\left(-\frac{\mu^2}{2}\right)=\frac{32\mu}{15m^2\sqrt{2\pi}}\exp\left(-\frac{\mu^2}{2}\right),
    \end{equation*}
    which is the desired upper bound for the curvature constant $\mathtt{C}_{\SM}$. Replacing $\mu$ by $\Delta$ yields the desired result in terms of the separation $\Delta$.\qed

	\subsection{Proof of Lemma~\ref{lemma:lipschitz_curvature_bounds_dsm}}\label{app:proof_lemma_lipschitz_curvature_bounds_dsm}

	1. We split the proof into multiple steps.\\

	\noindent
	\textbf{Step 1: Reduction to symmetric one-dimensional setting.}
	As in the proof of 1. of Lemma~\ref{lemma:lipschitz_curvature_bounds_sm}, we can reduce the problem to a one-dimensional setting. In particular, it is sufficient to consider families of the form~\eqref{eq:symmetric_one_d_family} in view of 2. of Lemma~\ref{lemma:reduction_constants}. In the end, we can substitute $\mu$ by $\Delta$ and $\mu_t=\exp(-t)\mu$ by $\Delta_t=\exp(-t)\Delta$ to obtain the desired result for the family~\ref{eq:gm}.\\
	
	\noindent
	\textbf{Step 2: Integration by parts.}
	First, we bring $m_{\DDSM}$ into a form which is more susceptible to our subsequent approach. To that end, note that Stein's lemma~\citep[cf. Lemma 3.6.5 in][]{CasBer2002} states that
	\begin{equation*}
		\E[f(Z)Z]=\E[f'(Z)],
	\end{equation*}
	for $Z\sim\mathcal{N}(0,1)$ and any differentiable function $f:\R\to\R$ such that $\E[f(Z)Z]$ and $\E[f'(Z)]$ exist. It gives us
    \begin{align*}
        \E\left[s_{\theta,t}(X_t)Z_t\bigm\vert X_0=x\right]&=\E\left[g(Z)Z\right]=\E\left[g'(Z)\right]=\sqrt{1-\exp(-2t)}\E\left[s'_{\theta,t}(X_t)\bigm\vert X_0=x\right],
    \end{align*}
	with $Z\sim\mathcal{N}(0,1)$ and $g(z):=s_{\theta,t}(\exp(-t)x+\sqrt{1-\exp(-2t)}z)$. Therefore, we can write
	\begin{equation}
		\label{eq:m_dsm_stein}
		m_{\DDSM}(\theta,x)=\int_0^T\E\left[s^2_{\theta,t}(X_t)+2s'_{\theta,t}(X_t)\Bigm\vert X_0=x\right]\dif t.
	\end{equation}
	Note that Stein's lemma essentially characterizes integration by parts for Gaussian measures, such that the above representation of $m_{\DDSM}$ is based on the same integration by parts argument which was used to obtain the tractable version of the Fisher divergence in~\eqref{eq:fisher_distance_3} in Section~\ref{subsec:estimators}.\\

	\noindent
	\textbf{Step 3: Contraction of the conditional expectation.}
	Based on~\eqref{eq:m_dsm_stein}, we have
	\begin{equation*}
		\partial_\theta m_{\DDSM}(\theta,x)= \int_0^T \E\left[\partial_\theta\Upsilon_{\theta,t}(X_t)\Bigm\vert X_0=x\right]\dif t,
	\end{equation*}
	where $\Upsilon_{\theta,t}(x):=s^2_{\theta,t}(x)+2s'_{\theta,t}(x)$. Differentiation under the integral sign is justified by dominated convergence. Again, as explained at the beginning of the proof of Lemma~\ref{lemma:lipschitz_curvature_bounds_sm} in Appendix~\ref{app:proof_lemma_lipschitz_curvature_bounds_sm}, we define the Lipschitz constant as the global envelope
	\begin{equation*}
		\mathtt{L}_{\DDSM}(x):=\sup_{\theta\in\Theta}\left|\partial_\theta m_{\DDSM}(\theta,x)\right|,
	\end{equation*}
	such that we need to bound
	\begin{equation}
		\label{eq:lipschitz_constant_dsm}
		\|\mathtt{L}_{\DDSM}(X_0)\|_{L^2(\Probtrue)}=\left\|\sup_{\theta\in\Theta}\left|\int_0^T \E\left[\partial_\theta\Upsilon_{\theta,t}(X_t)\Bigm\vert X_0\right]\dif t\right|\right\|_{L^2(\Prob_{\theta_0})}.
	\end{equation}
	Note that on the right-hand side of~\eqref{eq:lipschitz_constant_dsm} we have replaced $\Probtrue$ with $\Prob_{\theta_0}$, which is justified by Assumption~\ref{assumption:probtrue}. The monotonicity of the integral and the conditional expectation respectively allows us to pull the absolute value together with the supremum all the way into the conditional expectation. With the triangle inequality for integrals we may also pull the norm inside the integral wrt $t$, which gives us
	\begin{equation}
		\label{eq:L2_norm_1}
		\left\|\sup_{\theta\in\Theta}\left|\int_0^T \E\left[\partial_\theta\Upsilon_{\theta,t}(X_t)\Bigm\vert X_0\right]\dif t\right|\right\|_{L^2(\Prob_{\theta_0})}\leqslant\int_0^T\left\| \E\left[\sup_{\theta\in\Theta}\left|\partial_\theta\Upsilon_{\theta,t}(X_t)\right|\Bigm\vert X_0\right]\right\|_{L^2(\Prob_{\theta_0})}\dif t.
	\end{equation}
	The conditional expectation is a contraction, that is, it holds that $\|\E[X\,|\,\mathcal{F}_0]\|_{L^2(\Prob)}\leqslant\|X\|_{L^2(\Prob)}$ for any random variable $X$ defined on the probability space $(\Omega,\mathcal{F},\Prob)$ and $\mathcal{F}_0\subseteq\mathcal{F}$ being a sub-$\sigma$-algebra~\citep[see Corollary 8.21 in][]{Kle2020}. Applying this to the right-hand side of~\eqref{eq:L2_norm_1} yields
	\begin{equation}
		\label{eq:L2_norm_2}
		\int_0^T\left\| \E\left[\sup_{\theta\in\Theta}\left|\partial_\theta\Upsilon_{\theta,t}(X_t)\right|\Bigm\vert X_0\right]\right\|_{L^2(\Prob_{\theta_0})}\dif t\leqslant\int_0^T\left\|\sup_{\theta\in\Theta}\left|\partial_\theta\Upsilon_{\theta,t}(X_t)\right|\right\|_{L^2(\Prob_{\theta_0,t})}\dif t.
	\end{equation}
	Note that on the right-hand side of~\eqref{eq:L2_norm_2} we have replaced $\Prob_{\theta_0}$ with $\Prob_{\theta_0,t}$, since $X_t\sim\Prob_{\theta_0,t}$ when $X_0\sim\Prob_{\theta_0}$. In the next step, we will bound the norm inside the integral on the right-hand side of~\eqref{eq:L2_norm_2}.\\

	\noindent
	\textbf{Step 4: Bound the supremum.}
	For the supremum term in~\eqref{eq:L2_norm_2} we have
	\begin{equation}
		\label{eq:supremum_bound}
		\sup_{\theta\in\Theta}\left|\partial_\theta\Upsilon_{\theta,t}(x)\right|\leqslant 2\left\{\sup_{\theta\in\Theta}\left|s_{\theta,t}(x)\right|\right\}\left\{\sup_{\theta\in\Theta}\left|\partial_\theta s_{\theta,t}(x)\right|\right\}+2\sup_{\theta\in\Theta}\left|\partial_\theta s'_{\theta,t}(x)\right|,
	\end{equation}
	where we will use from now on the notation 
	\begin{equation*}
	\alpha_t:=\sup_{\theta\in\Theta}\left| s_{\theta,t}(X_t)\right|,\quad \beta_t:=\sup_{\theta\in\Theta}\left|\partial_\theta s_{\theta,t}(X_t)\right|,\quad \gamma_t:=\sup_{\theta\in\Theta}|\partial_\theta s'_{\theta,t}(X_t)|
	\end{equation*}
	for brevity. Inserting the right-hand side of~\eqref{eq:supremum_bound} into the norm, applying the triangle inequality first and the Cauchy-Schwarz inequality to the product afterwards gives us
	\begin{equation}
		\label{eq:triangle_cauchy_schwarz}
		\left\|\sup_{\theta\in\Theta}\left|\partial_\theta\Upsilon_{\theta,t}(X_t)\right|\right\|_{L^2(\Prob_{\theta_0,t})}\leqslant 2\left\|\alpha_t\right\|_{L^4(\Prob_{\theta_0,t})}\left\|\beta_t\right\|_{L^4(\Prob_{\theta_0,t})}+2\left\|\gamma_t\right\|_{L^2(\Prob_{\theta_0,t})}.
	\end{equation}
	Next, we will bound the three norms on the right-hand side of~\eqref{eq:triangle_cauchy_schwarz} separately.\\

	\noindent
	\textbf{Step 5: Bounding the norms.}
	First, we use the representation for the score function $s_{\theta,t}$ following from 1. of Lemma~\ref{lemma:gm_properties} and the fact that $w_{\theta,t}(x)\in(0,1)$ for all $x\in\R$ to obtain
	\begin{equation*}
		\sup_{\theta\in\Theta}|s_{\theta,t}(x)|=\sup_{\theta\in\Theta}|(2w_{\theta,t}(x)-1)\mu_t-x|\leqslant \mu_t + |x|.
	\end{equation*}
	Inserting this into the norm, applying the triangle inequality and using the identity
	\begin{equation*}
		\|X_t\|_{L^4(\Prob_{\theta_0,t})}=\|X_t\|_{L^4(\mathcal{N}(\mu_t,1))}=\left(3+6\mu_t^2+\mu_t^4\right)^{1/4},
	\end{equation*}
	which is due to the fact that $x\mapsto x^4$ is even, gives us
	\begin{equation}
		\label{eq:score_norm_bound}
		\left\|\alpha_t\right\|_{L^4(\Prob_{\theta_0,t})}=\left\|\sup_{\theta\in\Theta}\left|s_{\theta,t}(X_t)\right|\right\|_{L^4(\Prob_{\theta_0,t})}\leqslant \mu_t + \|X_t\|_{L^4(\Prob_{\theta_0,t})}= \mu_t + \left(3+6\mu_t^2+\mu_t^4\right)^{1/4}.
	\end{equation}
	Next, we turn to the second term. After computing the derivative
	\begin{equation*}
		\partial_\theta w_{\theta,t}(x)=\frac{\varphi(x-\mu_t)\varphi(x+\mu_t)}{\left(\theta\varphi(x-\mu_t)+(1-\theta)\varphi(x+\mu_t)\right)^2},
	\end{equation*}
	we use 3. of Lemma~\ref{lemma:gm_properties} to obtain the pointwise bound
	\begin{equation}
		\label{eq:pointwise_bound}
		\sup_{\theta\in\Theta}\left|\partial_\theta w_{\theta,t}(x)\right|\leqslant \frac{\exp(-2\mu_t|x|)}{\eta^2}.
	\end{equation}
	With $\partial_\theta s_{\theta,t}(x)=2\mu_t\partial_\theta w_{\theta,t}(x)$ we then arrive at
	\begin{equation*}
		\left\|\sup_{\theta\in\Theta}\left|\partial_\theta s_{\theta,t}(X_t)\right|\right\|_{L^4(\Prob_{\theta_0,t})}\leqslant \frac{2\mu_t}{\eta^2}\left\|\exp(-2\mu_t|X_t|)\right\|_{L^4(\Prob_{\theta_0,t})}.
	\end{equation*}
	The norm on the right-hand side can be bounded as follows with 4. of Lemma~\ref{lemma:gm_properties}:
	\begin{equation*}
		\left\|\exp(-2\mu_t|X_t|)\right\|_{L^4(\Prob_{\theta_0,t})}\leqslant \left[\varphi(\mu_t)\left(\frac{1}{7\mu_t}+\frac{1}{9\mu_t}\right)\right]^{1/4}=\left(\frac{16}{63\mu_t\sqrt{2\pi}}\right)^{1/4}\exp\left(-\frac{\mu_t^2}{8}\right)
	\end{equation*}
	In total, the two previous displays give us
	\begin{equation}
		\label{eq:score_derivative_norm_bound}
		\left\|\beta_t\right\|_{L^4(\Prob_{\theta_0,t})}=\left\|\sup_{\theta\in\Theta}\left|\partial_\theta s_{\theta,t}(X_t)\right|\right\|_{L^4(\Prob_{\theta_0,t})}\leqslant\frac{4\mu_t^{3/4}}{\eta^2(2\pi)^{1/8}}\exp\left(-\frac{\mu_t^2}{8}\right).
	\end{equation}
	Finally, we adress the third term on the right-hand side of~\eqref{eq:triangle_cauchy_schwarz}. The representation from 1. of Lemma~\ref{lemma:gm_properties} tells us that $\partial_\theta s'_{\theta,t}(x)=2\mu_t\partial_\theta w'_{\theta,t}(x)$. Differentiating the logistic derivative $w'_{\theta,t}(x)=2\mu_t w_{\theta,t}(x)(1-w_{\theta,t}(x))$ wrt $\theta$ yields
		\begin{align*}
		\partial_\theta w'_{\theta,t}(x)&=2\mu_t\left\{\partial_\theta w_{\theta,t}(x)(1-w_{\theta,t}(x))-w_{\theta,t}(x)\partial_\theta w_{\theta,t}(x)\right\}\\
		&=2\mu_t\partial_\theta w_{\theta,t}(x)(1-2w_{\theta,t}(x)).
	\end{align*}
	Combining this with the pointwise bound in~\eqref{eq:pointwise_bound} and the fact that $1-w_{\theta,t}(x)\in(-1,1)$ for all $x\in\R$ gives
	\begin{equation*}
		\sup_{\theta\in\Theta}\left|\partial_\theta s'_{\theta,t}(x)\right|\leqslant \frac{4\mu_t^2}{\eta^2}\exp(-2\mu_t|x|).
	\end{equation*}
	And lastly, bounding the norm with 4. of Lemma~\ref{lemma:gm_properties} yields
	\begin{align}
		\nonumber
		\left\|\gamma_t\right\|_{L^4(\Prob_{\theta_0,t})}=\left\|\sup_{\theta\in\Theta}\left|\partial_\theta s'_{\theta,t}(X_t)\right|\right\|_{L^2(\Prob_{\theta_0,t})}&\leqslant \frac{4\mu_t^2}{\eta^2}\left\|\exp(-2\mu_t|X_t|)\right\|_{L^2(\Prob_{\theta_0,t})}\\
		\nonumber
		&\leqslant \frac{4\mu_t^2}{\eta^2}\left[\varphi(\mu_t)\left(\frac{1}{3\mu_t}+\frac{1}{5\mu_t}\right)\right]^{1/2}\\
		\label{eq:score_second_derivative_norm_bound}
		&\leqslant\frac{4\mu_t^{3/2}}{\eta^2(2\pi)^{1/4}}\exp\left(-\frac{\mu_t^2}{4}\right).
	\end{align}
	Combining~\eqref{eq:score_norm_bound},~\eqref{eq:score_derivative_norm_bound} and~\eqref{eq:score_second_derivative_norm_bound} and some simplifications give us the following bound for the norm on the right-hand side of~\eqref{eq:triangle_cauchy_schwarz}:
	\begin{align}
		\label{eq:final_norm_bound}
		\left\|\sup_{\theta\in\Theta}\left|\partial_\theta\Upsilon_{\theta,t}(X_t)\right|\right\|_{L^2(\Prob_{\theta_0,t})}
		&\leqslant \frac{8}{\eta^2(2\pi)^{1/8}}\exp\left(-\frac{\mu_t^2}{8}\right)\left(\mu_t^{7/4} + \mu_t^{3/4}\left(3+6\mu_t^2+\mu_t^4\right)^{1/4}+\mu_t^{3/2}\right).
	\end{align}
	Inserting this back into the integral in~\eqref{eq:L2_norm_2} and rewriting it as an expectation wrt a random variable $U\sim\text{Uniform}([0,T])$ gives us
	\begin{align*}
		\int_0^T\left\|\sup_{\theta\in\Theta}\left|\partial_\theta\Upsilon_{\theta,t}(X_t)\right|\right\|_{L^2(\Prob_{\theta_0,t})}\dif t\leqslant \frac{8T}{\eta^2(2\pi)^{1/8}}\E\left[g(\mu_U)\exp\left(-\frac{\mu_U^2}{8}\right)\right],
	\end{align*}
	where $g(x):=x^{7/4} + x^{3/4}\left(3+6x^2+x^4\right)^{1/4}+x^{3/2}$. After substituting as described in Step 1, this is the bound given in the statement of Lemma~\ref{lemma:lipschitz_curvature_bounds_dsm}. In the next step we will distinguish between $\mu_t\geqslant1$ and $\mu_t<1$ in order to show that the integral is independent of $\mu$ when $T$ is chosen sufficiently large.\\

	\noindent
	\textbf{Step 6: Distinguish $\mu_t>1$ and $\mu_t\leqslant1$.} For fixed $\mu>0$ consider the decomposition
	\begin{equation*}
		[0,T]=\{t\in[0,T]:\mu_t>1\}\cup\{t\in[0,T]:\mu_t\leqslant1\}.
	\end{equation*}
	Then, applying $\bm{1}_{[0,T]}=\bm{1}_{\{\mu_t>1\}}+\bm{1}_{\{\mu_t\leqslant1\}}$ to the right-hand side of~\eqref{eq:final_norm_bound} yields
	\begin{align}
		\label{eq:final_norm_bound_2}
		\left\|\sup_{\theta\in\Theta}\left|\partial_\theta\Upsilon_{\theta,t}(X_t)\right|\right\|_{L^2(\Prob_{\theta_0,t})}
		&\leqslant \frac{32}{\eta^2(2\pi)^{1/8}}\left\{\mu_t^2\exp\left(-\frac{\mu_t^2}{8}\right)\bm{1}_{\{\mu_t>1\}}(t)+\mu^{3/4}_t\bm{1}_{\{\mu_t\leqslant1\}}(t)\right\}.
	\end{align}
	Note that $\mu_t\leqslant1$ for all $t\in[0,T]$ if $\mu\leqslant1$. In that case it holds that $\{t\in[0,T]:\mu_t>1\}=\emptyset$ and the first summand in the sum above vanishes. With this, we are ready for the last step.\\

	\noindent
	\textbf{Step 7: Integrate over time.} Let us assume without loss of generality that $\mu>1$ holds true. Otherwise, the computation of the integral over the set $\{t\in[0,T]:\mu_t> 1\}$ becomes obsolete. Given that for fixed $\mu>0$ the function $t\mapsto\mu_t$ is strictly decreasing, the sets $\{t\in[0,T]:\mu_t>1\}$ and $\{t\in[0,T]:\mu_t\leqslant1\}$ are intervals. In fact, with $t^*:=\log\mu$ we have $\{t\in[0,T]:\mu_t>1\}=[0,t^*)$ and $\{t\in[0,T]:\mu_t\leqslant1\}=[t^*,T]$.
	
	Then, using the substitution $x = \mu_t$, we can compute
	\begin{align}
		\nonumber
		\int_0^T\mu_t^2\exp\left(-\frac{\mu_t^2}{8}\right)\bm{1}_{\{\mu_t>1\}}(t)\dif t&=\int_0^{t^*}\mu_t^2\exp\left(-\frac{\mu^2_t}{8}\right)\dif t
		=\int_1^\mu x\exp\left(-\frac{x^2}{8}\right)\dif x\\
		\label{eq:integral_bound_1}
		&=4\left(\exp(-1/8)-\exp\left(-\frac{\mu^2}{8}\right)\right)
		\leqslant4.
	\end{align}
	Similarly, we can compute
	\begin{align}
		\label{eq:integral_bound_2}
		\int_0^T\mu^{3/4}_t\bm{1}_{\{\mu_t\leqslant1\}}(t)\dif t&=\int_{t^*}^T\mu^{3/4}_t\dif t=\int_{\mu_T}^1x^{-1/4}\dif x=\frac{4}{3}\left(1-\mu_T^{3/4}\right)\leqslant\frac{4}{3},
	\end{align}
	where we used the assumption $T> 0\vee\log\mu$ to ensure that $\mu_T\leqslant1$. The combination of~\eqref{eq:final_norm_bound_2}, ~\eqref{eq:integral_bound_1} and ~\eqref{eq:integral_bound_2} results in
	\begin{equation*}
		\int_0^T\left\|\sup_{\theta\in\Theta}\left|\partial_\theta\Upsilon_{\theta,t}(X_t)\right|\right\|_{L^2(\Prob_{\theta_0,t})}\dif t\leqslant \frac{32}{\eta^2(2\pi)^{1/8}}\left(4+\frac{4}{3}\right)\leqslant\frac{192}{\eta^2(2\pi)^{1/8}}.
	\end{equation*}
	Thus, with $\mathtt{C}(\eta):=192/(\eta^2(2\pi)^{1/8})$ we have shown that $\|\mathtt{L}_{\DDSM}(X_0)\|_{L^2(\Prob_{\theta_0})}\leqslant \mathtt{C}(\eta)$ for all $\mu>0$ and $T> 0\vee\log\mu$. With the arguments from Step 1, this bound translates to the family~\ref{eq:gm}. This concludes the proof of the first part of Lemma~\ref{lemma:lipschitz_curvature_bounds_dsm}.\\

	\noindent
    2. Using the same arguments as in the beginning of the proof of 2. of Lemma~\ref{lemma:lipschitz_curvature_bounds_sm} in Appendix~\ref{app:proof_lemma_lipschitz_curvature_bounds_sm} together with identity~\eqref{eq:m_dsm_stein} we obtain
	\begin{equation*}
		\E_{\theta_0}[m_{\DDSM}(\theta,X) - m_{\DDSM}(\theta_0,X)]=\int_0^T\FI(\Prob_{\theta_0,t},\Prob_{\theta,t})\dif t.
	\end{equation*}
	Once more, in view of 2. of Lemma~\ref{lemma:reduction_constants} it is sufficient to consider the one-dimensional symmetric family~\eqref{eq:symmetric_one_d_family}. Then, Proposition~\ref{proposition:score_bound} implies
	\begin{align*}
		&\E_{\theta_0}[m_{\DDSM}(\theta,X) - m_{\DDSM}(\theta_0,X)]\\
		=&(\theta-\theta_0)^2\int_0^T\int_\R4\mu_t^2\left(\frac{\varphi(x-\mu_t)\varphi(x+\mu_t)}{f_{\theta,t}(x)f_{\theta_0,t}(x)}\right)^2f_{\theta_0,t}(x)\dif x\dif t,
	\end{align*}
    such that with
    \begin{equation*}
        \mathtt{C}_{\DDSM}:=\inf_{|\theta-\theta_0|\geqslant\delta}\left\{\int_0^T\int_\R4\mu_t^2\left(\frac{\varphi(x-\mu_t)\varphi(x+\mu_t)}{f_{\theta,t}(x)f_{\theta_0,t}(x)}\right)^2f_{\theta_0,t}(x)\dif x\dif t\right\}
    \end{equation*}
    it holds that
    \begin{equation*}
        \inf_{|\theta-\theta_0|\geqslant\delta}\E_{\theta_0}\left[m_{\DDSM}(\theta,X)-m_{\DDSM}(\theta_0,X)\right]\geqslant\mathtt{C}_{\DDSM}\delta^2,
    \end{equation*}
	that is the curvature condition is fulfilled with $\alpha=2$ and constant $\mathtt{C}_{\DDSM}$. Again, this relies on Assumption~\ref{assumption:probtrue}, as in the case of SM.
	
	In order to obtain the lower bound for $\mathtt{C}_{\DDSM}$, we define the events
	\begin{align*}
		\Omega:=&\left\{x\in\R:|x|\leqslant1\right\},\\
		\mathcal{T}:=&\left\{t\in[0,T]:\mu_t\leqslant1\right\}.
	\end{align*}
    Then, we can compute
\begin{equation}
	\label{eq:lower_bound_subset}
        \mathtt{C}_{\DDSM}
		\geqslant\int_0^T\int_\R\bm{1}_\mathcal{T}(t)\bm{1}_\Omega(x)4\mu_t^2\exp(-4\mu_t|x|)f_{\theta_0,t}(x)\dif x\dif t,
    \end{equation}
	where the inequality follows from 3. of Lemma~\ref{lemma:gm_properties} together with the fact that the integrand is non-negative. Now, note that $\exp(-4\mu_t|x|)\geqslant\exp(-4)$ on the event $\mathcal{T}\times\Omega$ and
	\begin{equation*}
		\Prob(|X_t|\leqslant1)=\Prob(|\mathcal{N}(\mu_t,1)|\leqslant1)\geqslant \Prob(|\mathcal{N}(1,1)|\leqslant1)=:\mathtt{C}_{\mathtt{p}}>0,
	\end{equation*}
	for all $\mu_t\leqslant1$. With this we can further lower bound the right-hand side of~\eqref{eq:lower_bound_subset} as follows:
\begin{align*}
	\int_0^T\int_\R\bm{1}_\mathcal{T}(t)\bm{1}_\Omega(x)4\mu_t^2\exp(-4\mu_t|x|)f_{\theta_0,t}(x)\dif x\dif t&\geqslant 4\exp(-4)\mathtt{C}_{\mathtt{p}}\int_{0}^T\bm{1}_\mathcal{T}(t)\mu_t^2\dif t.
\end{align*}
Setting $t_0:=0\vee\log\mu$, note that $\mu_t\leqslant1$ for all $t\in[t_0,T]$, which gives
\begin{align*}
	\int_{0}^T\bm{1}_\mathcal{T}(t)\mu_t^2\dif t=\int_0^{t_0}\bm{1}_\mathcal{T}(t)\mu_t^2\dif t+\int_{t_0}^T\bm{1}_\mathcal{T}(t)\mu_t^2\dif t=\int_{t_0}^T\mu_t^2\dif t,
\end{align*}
since the first integral over $[0, t_0]\cap\mathcal{T}$ vanishes by definition of $t_0$. Let us now distinguish two cases.\\

\noindent
\textbf{Case 1: $\mu\leqslant 1$.} Then $t_0=0$ and it follows
\begin{align*}
	\int_{t_0}^T\mu_t^2\dif t&=\frac{\mu^2}{2}(1-\exp(-2T)).
\end{align*}

\noindent
\textbf{Case 2: $\mu> 1$.} Then $t_0=\log\mu$ and with $T\geqslant t_0$ it follows
\begin{align*}
	\int_{t_0}^T\mu_t^2\dif t&=\frac{\mu^2}{2}(\exp(-2\log\mu)-\exp(-2T))=\frac{1}{2}(1-\mu^2\exp(-2T)).
\end{align*}

Putting everything together, we arrive at
\begin{align*}
	\mathtt{C}_{\DDSM}\gtrsim
	\begin{cases}
		\mu^2(1-\exp(-2T)),&\text{if }\mu\leqslant1,\\
		 1-\mu^2\exp(-2T),&\text{if }\mu>1.
	\end{cases}
\end{align*}
If $\mu>1$ and $T\geqslant\log(\mu^m)$ for some integer $m\geqslant 2$, then
\begin{equation*}
	\mathtt{C}_{\DDSM}\gtrsim 1-\mu^2\exp(-2\log(\mu^m))=1-\frac{1}{\mu^{2(m-1)}}>0.
\end{equation*}
Setting $\mu=\Delta$ the result translates to~\ref{eq:gm}, which concludes the proof of the second part.\qed

	\section{Auxiliary results}

	\subsection{Gaussian bounds}
	\label{app:gaussian_bounds}

	\begin{lemma}[Gaussian tail bounds]
	\label{lemma:gaussian_tail_bound}
	Let $X\sim\mathcal{N}(\mu,\sigma^2)$ and $t>0$. Then, it holds that
	\begin{equation*}
		\frac{\varphi(t/\sigma)}{t/\sigma+\sigma/t}\leqslant\Prob\left(X-\mu\geqslant t\right)\leqslant \frac{\varphi(t/\sigma)}{t/\sigma}.
	\end{equation*}
	The same bounds hold true for $\Prob(X-\mu\leqslant-t)$.
\end{lemma}

\begin{proof}
	Let $Z\sim\mathcal{N}(0,1)$. Then, we can compute
	\begin{align*}
		\Prob(Z  \geqslant t)=\int_t^\infty\varphi(x)\dif x
		\leqslant \int_t^\infty\frac{x}{t}\varphi(x)\dif x
		=\frac{1}{t}\left[-\varphi(x)\right]_t^\infty
		=\frac{\varphi(t)}{t},
	\end{align*}
	to obtain the upper bound. For the lower bound, we can compute
	\begin{align*}
		\Prob(Z  \geqslant t)=\int_t^\infty\varphi(x)\dif x
		=\int_t^\infty\frac{x}{x}\varphi(x)\dif x=\left[-\frac{\varphi(x)}{x}\right]_t^\infty-\int_t^\infty\frac{\varphi(x)}{x^2}\dif x,
	\end{align*}
	where the last equality follows from integration by parts applied to $u(x)=1/x$ and $v'(x)=x\varphi(x)$. Then, with $[-\varphi(x)/x]_t^\infty=\varphi(t)/t$ and
	\begin{equation}
		\int_t^\infty\frac{\varphi(x)}{x^2}\dif x\leqslant \int_t^\infty\frac{\varphi(x)}{t^2}\dif x=\frac{\Prob(Z  \geqslant t) }{t^2},
	\end{equation}
	we arrive at
	\begin{equation*}
		\Prob(Z  \geqslant t)\geqslant \frac{\varphi(t)}{t}-\frac{\Prob(Z  \geqslant t) }{t^2}.
	\end{equation*}
	Rearranging this gives us
	\begin{equation*}
		\Prob(Z  \geqslant t)\geqslant \frac{\varphi(t)}{t(1+1/t^2)}=\frac{\varphi(t)}{t+1/t},
	\end{equation*}
	which is the desired lower bound. Replacing $Z$ with $X-\mu$ and subsequently normalizing with $\sigma$ yields the upper and lower bound from the statement. The bounds for $\Prob(X-\mu\leqslant-t)$ follow by symmetry.
\end{proof}

The upper bound in Lemma~\ref{lemma:gaussian_tail_bound} is not meaningful for small values of $t$, since for $t\downarrow 0$ it holds that $\varphi(t)/t\rightarrow \infty$. In order to obtain a more sensible bound for small values of $t$, we can use the following Chernoff-type tail bound.	

    \begin{lemma}[Chernoff-type Gaussian tail bound]
	\label{lemma:gaussian_tail_bound1}
	Let $X\sim\mathcal{N}(\mu,\sigma^2)$ and $t\geqslant0$. Then, it holds that
	\begin{equation*}
		\Prob\left(X-\mu\geqslant t\right)\leqslant \frac{1}{2}\exp\left(-\frac{t^2}{2\sigma^2}\right).
	\end{equation*}
	The same bound hold true for $\Prob(X-\mu\leqslant-t)$.
\end{lemma}

\begin{proof}
Consider $Z\sim\mathcal{N}(0,1)$ and the map $h:[0,\infty)\rightarrow\R,\,z\mapsto h(z)$ with
\begin{equation*}
	h(z):=\Prob(Z\geqslant z)-\frac{1}{2}\exp\left(-\frac{z^2}{2}\right).
\end{equation*}
Then, the derivative of $h$ is given by
\begin{align*}
	h'(z)&=\frac{\dif}{\dif x}\left\{1-\Phi(z))-\frac{1}{2}\exp\left(-\frac{z^2}{2}\right)\right\}\\
	&=-\varphi(z)+\frac{z}{2}\exp\left(-\frac{z^2}{2}\right)\\
	&=\frac{1}{2} \exp\left(-\frac{z^2}{2}\right)\left(z- \sqrt{2/\pi}\right)
\end{align*}
and it holds that $h'(z)\leqslant 0$ for $z\in[0,\sqrt{2/\pi}]$ and $h'(z)\geqslant 0$ for $z\geqslant\sqrt{2/\pi}$. This shows that $h$ is decreasing on $[0,\sqrt{2/\pi}]$ and increasing on $[\sqrt{2/\pi},\infty)$. Together with the observation that $h(0)=0$ and $h(z)\rightarrow 0$ as $z\rightarrow\infty$, it implies that $h(z)\leqslant 0$ for all $z\geqslant 0$. Therefore, we have
\begin{equation*}
	\Prob(Z\geqslant z)\leqslant \frac{1}{2}\exp\left(-\frac{z^2}{2}\right),
\end{equation*}
for all $z\geqslant 0$. The claim follows by substituting $z=t/\sigma$.
\end{proof}

Combining the upper bound from Lemma~\ref{lemma:gaussian_tail_bound} with the one from Lemma~\ref{lemma:gaussian_tail_bound1} gives us the following upper bound for the tail probabilities of a Gaussian $X\sim\mathcal{N}(\mu,\sigma^2)$ , which is meaningful for all values of $t\geqslant0$ and in fact tight for $t=0$.

\begin{corollary}
	\label{corollary:gaussian_tail_bound}
	Let $X\sim\mathcal{N}(\mu,\sigma^2)$ and $t\geqslant0$. Then, it holds that
	\begin{equation*}
	\Prob\left(X-\mu\geqslant t\right)\leqslant \left(\sqrt{\frac{\pi}{2}}\wedge\frac{1}{t}\right)\varphi(t/\sigma).
\end{equation*}
	The same bound holds true for $\Prob(X-\mu\leqslant-t)$.
\end{corollary}

The preceeding tail bounds for a Gaussian random variable $X$ translate to the expectation of random variables $\exp(-t|X|)$ for all $t\geqslant\mu$ as shown in the following lemma.

\begin{lemma}[Gaussian exponential bound]
	\label{lemma:Gaussian_exp_bound}
	Let $X\sim\mathcal{N}(\mu,1)$ and $\mu>0$. Then, for all $t\geqslant\mu$ it holds that
		\begin{equation*}
		\E\left[\exp(-t |X|)\right]\leqslant\left\{\sqrt{\frac{\pi}{2}}\wedge\left(\frac{1}{t-\mu}+\frac{1}{t+\mu}\right)\right\}\varphi(\mu).
		\end{equation*}
\end{lemma}

\begin{proof}
	When $X\sim\mathcal{N}(\mu,1)$, then $|X|$ follows a folded Gaussian distribution with parameters $\mu$ and $1$ and the corresponding moment generating function (MGF) has the following closed-form expression for all $t\in\R$:
	\begin{align*}
		M_{|X|}(t)&:=\E\left[\exp(t |X|)\right]\\
		&=\exp\left(\frac{t^2}{2}+\mu t\right)\Phi\left(\mu+ t\right)+\exp\left(\frac{t^2}{2}-\mu t\right)\Phi\left(-\mu+ t\right)
	\end{align*}
	This gives us
	\begin{equation*}
		\E\left[\exp(-t |X|)\right]=\exp\left(\frac{t^2}{2}-\mu t\right)\Phi\left(\mu- t\right)+\exp\left(\frac{t^2}{2}+\mu t\right)\Phi\left(-\mu- t\right).
	\end{equation*}
	Since $\mu-t<0$ and $-\mu-t<0$, we can apply the tail bound from Corollary~\ref{corollary:gaussian_tail_bound} to the two distribution functions $\Phi(\mu-t)=\Prob(Z\leqslant -(t-\mu))$ and $\Phi(-\mu-t)=\Prob(Z\leqslant -(\mu+t))$. Using the identities
	\begin{equation*}
		\varphi(t-\mu)=\varphi(\mu)\exp\left(\mu t-\frac{t^2}{2}\right)
	\end{equation*}
	and
	\begin{equation*}
		\varphi(t+\mu)=\varphi(\mu)\exp\left(-\mu t-\frac{t^2}{2}\right),
	\end{equation*}
	which can be verified by direct computation, we can compute the upper bound as follows:
	\begin{align*}
		\E\left[\exp(-t |X|)\right]&\leqslant\exp\left(\frac{t^2}{2}-\mu t\right)\left(\sqrt{\frac{\pi}{2}}\wedge\frac{1}{t-\mu}\right)\varphi(t-\mu)\\
		&\quad+\exp\left(\frac{t^2}{2}+\mu t\right)\left(\sqrt{\frac{\pi}{2}}\wedge\frac{1}{t+\mu}\right)\varphi(t+\mu)\\
		&=\left\{\sqrt{\frac{\pi}{2}}\wedge\left(\frac{1}{t-\mu}+\frac{1}{t+\mu}\right)\right\}\varphi(\mu).
	\end{align*}
	This concludes the proof. 
\end{proof}

\end{document}